\documentclass[10pt,journal,compsoc]{IEEEtran}
\usepackage{hyperref}
\usepackage[nocompress]{cite}
\usepackage{color,xcolor,bm,mathtools,amsfonts,url,cleveref,amsthm,amssymb}
\usepackage{threeparttable,tabularx,booktabs,multirow}
\usepackage{algorithm,algpseudocode}
\usepackage{cleveref}

\newtheorem{theorem}{Theorem}
\newtheorem{lemma}{Lemma}
\newtheorem{assumption}{Assumption}
\newtheorem{definition}{Definition}
\newtheorem*{remark}{Remark}

\newtheorem{proposition}{Proposition}
\crefname{equation}{}{}
\crefname{theorem}{theorem}{theorems}
\Crefname{theorem}{Theorem}{Theorems}
\crefname{assumption}{assumption}{assumptions}
\Crefname{assumption}{Assumption}{Assumptions}

\newcommand{\rrr}[1]{{#1}}

\newcommand{\defeq}{\overset{\Delta}{=}}
\newcommand{\signplus}{\text{sign}_+}
\newcommand{\sign}{\text{sign}}

\usepackage[caption=false,font=footnotesize,labelfont=sf,textfont=sf]{subfig}

\begin{document}
\title{Distributed Evolution Strategies for Black-box Stochastic Optimization}

\author{Xiaoyu He, Zibin Zheng, Chuan Chen, Yuren Zhou, Chuan Luo, and Qingwei Lin
\IEEEcompsocitemizethanks{\IEEEcompsocthanksitem[] Xiaoyu He, Zibin Zheng, Chuan Chen, and Yuren Zhou are with the School of Computer Science and Engineering, Sun Yat-sen University, Guangzhou 510006, P. R. China. Chuan Luo and Qingwei Lin are with Microsoft Research, P. R. China. Xiaoyu He is also with the School of Computer Science and Engineering, Nanyang Technological University,  Singapore 639798. (E-mail: hxyokokok@foxmail.com (X. He), zhzibin@mail.sysu.edu.cn (Z. Zheng), chenchuan@mail.sysu.edu.cn (C. Chen), zhouyuren@mail.sysu.edu.cn (Y. Zhou), chuan.luo@microsoft.com (C. Luo), qlin@microsoft.com (Q. Lin))} 
\IEEEcompsocitemizethanks{\IEEEcompsocthanksitem[] * Corresponding Author: Z. Zheng}}

\markboth{Journal of \LaTeX\ Class Files,~Vol.~14, No.~8, August~2015}%
{Shell \MakeLowercase{\textit{et al.}}: Bare Demo of IEEEtran.cls for Computer Society Journals}

\IEEEtitleabstractindextext{%
\begin{abstract}
This work concerns the evolutionary approaches to distributed stochastic black-box optimization, in which each worker can individually solve an approximation of the problem with nature-inspired algorithms. We propose a distributed evolution strategy (DES) algorithm grounded on a proper modification to evolution strategies, a family of classic evolutionary algorithms, as well as a careful combination with existing distributed frameworks. On smooth and nonconvex landscapes, DES has a convergence rate competitive to existing zeroth-order methods, and can exploit the sparsity, if applicable, to match the rate of first-order methods. The DES method uses a Gaussian probability model to guide the search and avoids the numerical issue resulted from \rrr{finite-difference} techniques in existing zeroth-order methods. The DES method is also fully adaptive to the problem landscape, as its convergence is guaranteed with any parameter setting. We further propose two alternative sampling schemes which significantly improve the sampling efficiency while leading to similar performance. Simulation studies on several machine learning problems suggest that the proposed methods show much promise in reducing the convergence time and improving the robustness to parameter settings.
\end{abstract}

}

\maketitle

\IEEEdisplaynontitleabstractindextext
\IEEEraisesectionheading{\section{Introduction}\label{sec:introduction}}
We consider the following stochastic optimization problem:
\begin{equation}
	\min_{\bm{x}\in \mathbb{R}^n} f(\bm{x}) = \mathbb{E}\left[F(\bm{x};\bm{\xi})\right]
	\label{eq:definition-SOP}
\end{equation}
where $\bm{x} \in \mathbb{R}^{n}$ is the decision vector, $\bm{\xi}$ is a random variable, $F$ is an unconstrained real-valued function, and $\mathbb{E}\left[\cdot\right]$ denotes the expectation taken over the distribution of $\bm{\xi}$.
Problems of this type have a long history dating back to 1950's~\cite{robbins_stochastic_1951} and are still at the heart of many modern applications in machine learning~\cite{bottou_optimization_2018,sun_survey_2020}, signal processing~\cite{pereyra_survey_2016}, and automatic control~\cite{lee_optimization_2020}. 
For example, we can let $\bm{\xi}$ be a data point and $F$ a loss assessing how $\bm{\xi}$ fits to a statistic model parametrized by $\bm{x}$; the problem (\ref{eq:definition-SOP}), in this way, then provides a universal formulation that captures a wide range of machine learning tasks~\cite{gambella_optimization_2021}. 
The hardness of stochastic optimization mainly comes from the inherent noise nature, the possibly high dimensionality, and the complexity of objective landscapes. Despite this hardness, significant progress in the resolution of problem (\ref{eq:definition-SOP}) has been made via exploring the gradient (first-order) or Hessian (second-order) information of the component function $F$.
A variety of first-order and second-order stochastic optimization methods have been developed in recent years, enjoying both the theoretical and practical benefits; see \cite{curtis_adaptive_2020,mokhtari_stochastic_2020} for a comprehensive survey. However, when the landscape characteristics (differentiability, smoothness, convexity, etc.) are unknown, stochastic optimization remains a challenging task.

In this paper, we are particularly interested in solving problem (\ref{eq:definition-SOP}) in distributed black-box settings. Concretely, there are $M$ workers having access to the distribution of $\bm{\xi}$, but, for a given $\bm{x}$, they can only evaluate the stochastic objective value $F(\bm{x};\bm{\xi})$.
The workers may run individually and exchange information periodically through a parameter server, so they can minimize $f$	in a collaborative manner. But apart from the decision vector $\bm{x}$, what they can share during the collaboration is limited to the function values (zeroth-order information), excluding the sharing of gradient or curvature information (which is case of existing first-/second-order distributed methods). The consideration of this setting is motivated by two scenarios in the real world. The first scenario is related to the on-device machine learning, sometimes referred to as federated learning~\cite{konecny_federated_2016,konecny_federated_2016-1} or edge intelligence~\cite{wang_-edge_2019}. It is known that machine learning practitioners seldom derive gradients manually; instead, they rely on automatic differentiation~\cite{baydin_automatic_2018} which computes the gradient during the function evaluation using the chain rule. The automatic differentiation tools work well on usual PCs, but they have a relatively large memory cost which may be prohibitive on mobile devices. In addition, automatic differentiation only runs on certain software environment and may cause compatibility issues in distributed settings.

The second applicable scenario is the parallel solving of stochastic black-box problems. Consider minimizing a time-consuming black-box function defined over a massive amount of data and the goal is to achieve acceleration with a multicore machine. 
Due to its black-box nature, the objective function may not be thread-safe, and therefore we have to use the process-level parallelization where the data is distributed to multiple processes. 
Sensor selection~\cite{liu_sensor_2016} and high-dimensional cox regression~\cite{kvamme_time--event_2019} are representative examples that are suitable for this scenario.
These problems are in fact white-box, but the gradient evaluation is much more expensive than the function evaluation; treating them as black-box would heavily reduce the demand on computational resources.  
Generally, distributed black-box optimization offers a powerful search paradigm when calculating the gradient is expensive or infeasible, and it also retains the advantages from classical distributed computing frameworks in handling big data. 

Black-box optimization methods, sometimes known as zeroth-order or derivative-free optimization methods, require only the availability of objective function values. They were among the earliest optimization methods in the history, while having attracted renewed interest recently due to the ubiquity of black-box models. The community has made several efforts in bringing the simplicity and universality of black-box optimization methods to the distributed world. A cornerstone of this research line is the Gaussian smoothing technique~\cite{nesterov_random_2017}, a randomized finite-difference method that admits building a smooth surrogate of the original objective function with only zeroth-order information. With Gaussian smoothing, we can get a computationally cheap gradient estimator in the black-box setting, thereby making it possible to reuse existing first-order methods. Various black-box distributed optimization (DBO) methods have been proposed, based on the idea of hybridizing Gaussian smoothing with established distributed optimization methods~\cite{li_gradient-free_2015,yuan_gradient-free_2015,yuan_zeroth-order_2016,gu_faster_2018,liu_signsgd_2019,sahu_decentralized_2020,wang_distributed_2021}. These methods are typically easy to implement: the only work to do is to replace the real gradient with the one produced by Gaussian smoothing. The disadvantage is that they suffer a dimension-dependent slowdown in convergence rate, which is the cost must be paid for the absence of gradient information~\cite{duchi_optimal_2015}.

Current development in DBO methods has not yet been entirely successful; several common issues can be identified and should be carefully addressed. The first issue is the introduction of the smoothing parameter, which keeps a trade-off in improving the gradient estimation accuracy while avoiding roundoff errors~\cite[Chapter 8]{nocedal_numerical_2006}.
Tuning the smoothing parameter is onerous, and could become even more tricky in the distributed setting. This is caused by that its optimal value depends on the computing environment but different workers may have different environment. The second issue is the lack of adaptivity, in the sense that decision makers have to tune the step-sizes in workers or in the server or at both sides. Existing step-size adaptation rules cannot be generalized to black-box settings easily, as the gradient estimators produced by Gaussian smoothing does not meet the usual assumptions designed for first-order methods. There also exist algorithm-specific issues. For example, in \cite{liu_signsgd_2019}, the authors found a well-developed distributed algorithm, signSGD~\cite{bernstein_signsgd_2018}, may fail to approach the optimality when extended to black-box settings, probably because of the unfordable sampling effort required for reducing the bias caused by Gaussian smoothing. 
For the above reasons, schemes that are based on Gaussian smoothing do not provide a truly seamless transformation of first-order distributed methods to the black-box setting. This calls for the need in developing new DBO methods based on completely different frameworks.

We propose in this work a new DBO method based on evolution strategies (ESs)~\cite{beyer_evolution_2002,hansen_evolution_2015,li_evolution_2020}, a popular family of nature-inspired methods that excel in black-box real-valued optimization. Unlike Gaussian smoothing, ESs do not try to approximate the gradient or its surrogate, but instead guide the search with a probability distribution and gradually update this distribution on the fly. Moreover, the update of distribution is adaptive, requiring no knowledge about the landscape characteristics and very less user-supplied parameters. These features make ESs a strong candidate in designing new DBO methods and seem promising in addressing the aforementioned issues involved in Gaussian smoothing. ESs also possess a useful feature that they only use the comparison results of the objective function values among solutions, rather than their exact values~\cite{auger_linear_2016}. This is likely to improve the robustness in the presence of noise, as the noise would not matter unless it changes the comparison results~\cite{astete-morales_evolution_2015}.
On the other hand, ESs are originally designed for non-distributed noise-less optimization and have not been extended to distributed settings. In fact, the major components of ESs are grounded on heuristics and a rigorous convergence analysis is still missing when applied on problems like (\ref{eq:definition-SOP}). The goal of this paper is to help bridge this gap by describing ideas that can improve the applicability and rigorousness of ESs in the distributed stochastic setting. 
In particular, we propose a \underline{d}istributed \underline{e}volution \underline{s}trategy (DES) with characteristics highlighted below:
\begin{itemize}
	\item DES adopts a synchronous architecture that employs ESs to perform worker-side local updates and allows delayed averaging of individual decision vectors to reduce the communication overhead. It also supports server-side momentum, which is found to improve the performance in practice.
	\item When the local ES update is driven by an isotropic Gaussian distribution and when the function landscape is nonconvex, DES is competitive with existing zeroth-order methods in terms of iteration complexity. When certain sparsity assumption is met, DES can even align with the convergence rate of first-order methods.
	\item DES is fully adaptive in the sense that its convergence is guaranteed with any initial settings. Moreover, no numerical difference is involved so users will not worry about the roundoff errors.
	\item We propose two alternative probability distributions for generating mutation vectors in local updates. This significantly reduces the computation cost in high-dimensional settings. 
\end{itemize}	

In the remainder of this article, we first describe some related work in Section 2. In Section 3 we describe the details of DES and analyze its convergence properties. We then provide in Section 4 two alternative sampling methods and discuss their impact on the algorithm performance. Section 5 uses simulation studies to investigate the performance of our proposals. The article is concluded in Section 6. This paper has a supplement containing all proofs of our theoretical findings, as well as additional experimental results.

\textbf{Notation} Vectors are written in bold lowercase. We use $\left\|\bm{x}\right\|_p$ to denote the $\ell_p$ norm of $\bm{x}$. In addition, we use $\left\|\bm{x}\right\|$ to denote \rrr{a generic vector norm} and $\left\|\bm{x}\right\|_\ast$ its dual norm. We use $\mathbb{E}\left[\cdot\right]$ to denote the expectation, $\mathbb{V}\left[\cdot\right]$ the variance, $\mathbb{I}\left\{\cdot\right\}$ the indicator function, and $\mathbb{P}\left\{\cdot\right\}$ the probability.
We use $\bm{0}$ and $\bm{I}$ to denote respectively the zero vector and the identity matrix of appropriate dimensions. $\mathcal{N}(\bm{0},\bm{I})$ denotes the multivariate isotropic Gaussian distribution. \rrr{We use $\bm{e}_i$ to denote the $i$-th column of $\bm{I}$, i.e., the vector with 1 at the $i$-th coordinate and 0s elsewhere.}

\section{Related work}
\label{sec:related-work}
Distributed optimization is an active field in the optimization and machine learning communities. 
Our proposal belongs to the class of synchronous distributed optimization methods, which has been extensively studied in the first-order setting. Representative works include \cite{DBLP:conf/ijcai/ZhouC18,zhang_parallel_2016,wang_cooperative_2019} and they are all based on the federated averaging (FedAvg) framework~\cite{mcmahan_communication-efficient_2017}. These methods usually use stochastic gradient descent (SGD) as the worker-side solver while adopting different schemes to improve communication efficiency. 
Theoretically, these first-order methods could be generalized straightforwardly to the black-box setting via Gaussian smoothing; but to the best of our knowledge, heretofore there exist no generic zeroth-order approaches in the synchronous distributed setting. The closest work is the zeroth-order version of signSGD, ZO-signSGD, described in~\cite{liu_signsgd_2019}; however, this method requires communication per iteration and does not guarantee global convergence. Another relevant method is FedProx~\cite{li_federated_2020} which does not rely on the specification of local solvers. FedProx technically admits using zeroth-order solvers at the worker-side, but it requires an additional regularization parameter to guarantee local functions becoming strongly convex; in this sense, it is not applicable when the function is black-box. On the other hand, there exist several DBO methods built on asynchronous parallel~\cite{lian_asynchronous_2015,gu_faster_2018} or multi-agent architectures~\cite{yuan_zeroth-order_2016,sahu_decentralized_2020}; but they are not applicable in the synchronous distributed setting which is the main focus of this work.


Choosing the step-size is critical in implementing stochastic optimization methods, as one cannot simply use a line search when the landscape is noisy. To avoid the tedious step-size tuning phase, a variety of adaptation schemes have been proposed for first-order stochastic methods, where the step-size is updated with historical first-order information. Remarkable examples includes~\cite{ward_adagrad_2019,reddi_convergence_2018,levy_online_2017,duchi_adaptive_2011}. However, only a few of these adaptation schemes have been extended to the distributed setting, e.g., in \cite{reddi_adaptive_2020,xie_local_2020,tong_effective_2020}, and they still require a manually selected step-size for each worker. The method proposed in this work, on the contrary, can automatically choose step-sizes for both the server and the workers, and seems to be the first one that achieves such ``full adaptivity''.

Moving beyond the classical approaches that are based on rigorous mathematic tools, studies on stochastic optimization are very scarce in the evolutionary computation community. Almost all existing studies consider a more generic setting, the noisy optimization, and do not explore the expectation structure of problem (\ref{eq:definition-SOP}); see \cite{rakshit_noisy_2017} for a survey. 
It is found in \cite{beyer_toward_2017,hellwig_steady_2018} that, via simple resampling, modern ESs originally designed for deterministic optimization may achieve the best known convergence rate on noisy landscapes~\cite{qian_effectiveness_2018}. These studies, however, require assumptions that are completely different from the ones used in classical literatures. It is still unknown how evolutionary algorithms perform on problem (\ref{eq:definition-SOP}) with more generic assumptions.

Various studies on distributed optimization exist in the evolutionary community; related methodologies and tools have been nicely summarized in \cite{gong_distributed_2015,harada_parallel_2020}. As evolutionary approaches are usually population-based, these studies mostly focus on the parallel acceleration of the function evaluations of population, but have seldom touched the data decentralization (which is the case of this study). 
In this study, the distributed framework is mainly designed to achieve data decentralization; but parallelization is also supported in a synchronous manner.

\section{The Proposed Method: DES} \label{sec:DES}
In this section, we first propose a modified ES method for non-distributed deterministic optimization and then use it as a building block to develop the DES algorithm. Although this work focus on black-box optimization, we need the following assumptions to analyze the performance of DES. Unless stated otherwise, we assume $\mathbb{R}^n$ is equipped with some generic vector norm $\|\cdot\|$ and its dual norm is denoted by $\|\cdot\|_*$.

\begin{assumption} \label{assumption:smoothness}
The function $F$ has Lipschitz continuous gradient with constant $L$ for any $\bm{\xi}$, i.e., 
\[
\left\|\nabla F\left(\bm{x};\bm{\xi}\right) - \nabla F\left(\bm{y};\bm{\xi}\right)\right\|_* \le L\left\|\bm{x}-\bm{y}\right\| \;\;\; \forall \bm{x},\bm{y}\in \mathbb{R}^n.
\]
\end{assumption}

\begin{assumption} \label{assumption:variance-boundedness}
The gradient of $F$ has bounded variance, i.e., 
\[\mathbb{E}\left[\left\|\nabla F\left(\bm{x};\bm{\xi}\right) - \nabla f\left(\bm{x}\right)\right\|_*^2\right] \le \sigma^2 \;\;\; \forall \bm{x} \in \mathbb{R}^n.\]
\end{assumption}

\begin{assumption} \label{assumption:iid-data}
Every worker has access to the distribution of $\bm{\xi}$ independently and identically.
\end{assumption}

\Cref{assumption:smoothness,assumption:variance-boundedness} are customary in the analysis of stochastic optimization. They are useful when using gradients in measuring the optimality on nonconvex landscapes. \Cref{assumption:iid-data} is somewhat restrictive; but it is required to reduce the global variance via minibatching at the worker-side. On the other hand, as an adaptive method, our method does not assume the gradients to be universally bounded, and this is an advantage over several existing methods (e.g., \cite{ward_adagrad_2019,reddi_convergence_2018}).

\subsection{A modified ES for deterministic optimization} \label{ss:deterministic-ES}
We first consider the simplest ES framework, usually termed as $(1+1)$-ES in the literature, where in each iteration a parent produces a single offspring using mutation and the one with a better objective value becomes the new parent. The mutation is typically performed with an isotropic Gaussian perturbation and its variance is gradually updated. The pseudo-code of this method is given in \Cref{alg:simple-ES}. Specifically, it maintains a vector $\bm{x}_k \in \mathbb{R}^n$ to encode the parent solution and a scalar $\alpha_k$ the standard variance. The vector $\bm{u}_k \in \mathbb{R}^n$ (called mutation vector) is drawn from the standard Gaussian distribution and then used to construct the offspring given by $\bm{x}_k + \alpha_k \bm{u}_k$. Hereinafter we call $\alpha_k$ the step-size because it (approximately) determines the length of the descent step. 


The only difference of our implementation to existing ones lies in the specification of step-sizes: here we use a pre-defined diminishing rule (in Line 2) while almost all modern ESs adopt a comparison-based adaptation rule. Precisely, most ESs obtain $\alpha_{k+1}$ via multiplying $\alpha_k$ by some factor that depends on whether the offspring is better than the parent. This admits the step-size to shrink exponentially fast, so ESs may achieve linear convergence on certain landscapes~\cite{akimoto_drift_2018}. In this work, however, the objective landscape is generally nonconvex, so we cannot expect more than sublinear convergence~\cite{agarwal_information-theoretic_2012}. It suggests a thorough redesign of the step-size rule.

Our choice of the step-size rule $\alpha_k = \alpha_0 / \sqrt{k+1}$ is to align with the known convergence rate on deterministic nonconvex functions, $\mathcal{O}\left(1/K\right)$, measured by the squared gradient norm. This is  illustrated in the following theorem.	

\begin{figure}[tb]
\begin{algorithm}[H]
	\caption{A modified ES implementation for deterministic nonconvex optimization}
	\small
	\label{alg:simple-ES}
	\begin{algorithmic}[1]
	\Require $\bm{x}_0 \in \mathbb{R}^n$: initial solution; $\alpha_0 \in \mathbb{R}_+$: initial step-size
	\For {$k = 0, 1, \cdots, K-1$}
		\State $\alpha_k = \alpha_0/\sqrt{k+1} $
		\State Sample $\bm{u}_k$ from $\mathcal{N}(\bm{0},\bm{I})$
		\If{ $f(\bm{x}_k + \alpha_k \bm{u}_k) \le f(\bm{x}_k)$}
			\State $\bm{x}_{k+1} = \bm{x}_k + \alpha_k \bm{u}_k$
		\Else
			\State $\bm{x}_{k+1} = \bm{x}_k$
		\EndIf
	\EndFor
 \end{algorithmic} 
 \end{algorithm}
\end{figure}

\begin{theorem} \label{theorem:convergence-simple-ES}
Let \Cref{assumption:smoothness} hold with the self-dual $\ell_2$ norm, i.e., $\|\cdot\| = \|\cdot\|_* = \|\cdot\|_2$.
Assume the function $f$ is bounded below by $f_*$.
The iterations generated by \Cref{alg:simple-ES} satisfy
\begin{equation} \label{eq:convergene-rate-simple-ES-l2-norm}
\begin{split}
\frac{1}{K}\sum_{k=0}^{K-1} & \mathbb{E}\left[\left\|\nabla f\left(\bm{x}_k\right)\right\|_2\right] \\
\le & \sqrt{\frac{2\pi}{K}}\left(\frac{f(\bm{x}_0)-f_*}{\alpha_0} + \alpha_0 Ln\left(1+\log K\right)\right).
\end{split}
\end{equation}
\end{theorem}

Define $\Delta_f = f\left(\bm{x}_0\right) - f_*$.
The bound in \cref{eq:convergene-rate-simple-ES-l2-norm} is minimized at $\alpha_0 = \Theta\left( \sqrt{\frac{\Delta_f}{Ln}}\right)$; in this case, we have, via taking the square on both sides, the following rate for ES: 
\begin{equation} \label{eq:modified-ES-optimal-rate}
\left(\frac{1}{K}\sum_{k=0}^{K-1} \mathbb{E}\left[\left\|\nabla f\left(\bm{x}_k\right)\right\|_2\right]\right)^2
\le \tilde{\mathcal{O}}\left(\frac{\Delta_f Ln}{K}\right)
\end{equation}
where $\tilde{\mathcal{O}}$ hides the negligible $\log K$ term in the $\mathcal{O}$ notation.
Whereas, for comparison, the best known bound for gradient descent is
\begin{equation} \label{eq:typical-rate-GD}
\frac{1}{K}\sum_{k=0}^{K-1} \mathbb{E}\left[\left\|\nabla f\left(\bm{x}_k\right)\right\|_2^2\right]
\le \mathcal{O}\left(\frac{\Delta_f L}{K}\right),	
\end{equation}
or, if the gradient is estimated using Gaussian smoothing,
\begin{equation} \label{eq:typical-rate-ZOGD}
\frac{1}{K}\sum_{k=0}^{K-1} \mathbb{E}\left[\left\|\nabla f\left(\bm{x}_k\right)\right\|_2^2\right]
\le \mathcal{O}\left(\frac{\Delta_f Ln}{K}\right).	
\end{equation}
See \cite{ghadimi_stochastic_2013} for these results. 
These bounds are quite similar, expect for the difference in measuring the optimality. It suggests that 1) the proposed modified ES is competitive with zeroth-order gradient descent methods that are based on Gaussian smoothing, and 2) is only $n$ times slower than first-order gradient descent methods. 
\rrr{The slowdown compared to first-order methods is probably due to that the mutation in ES is not necessarily a descent step and it has a dimension-dependent variance.}
The advantage of ES is twofold: it does not need to estimate the gradient and it converges with any step-size setting.

\subsection{Implementation of DES}
We now describe the DES method for handling distributed stochastic problems. \Cref{alg:DES} provides the pseudo-code for our method. DES adopts the well-known federated averaging framework and uses the deterministic ES proposed in \Cref{ss:deterministic-ES} as worker-side solvers. Its search process is divided into $T$ rounds, and in the $t$-th round, the server maintains a solution $\bm{x}_t \in \mathbb{R}^n$, a step-size $\alpha_0^t \in \mathbb{R}_+$, and an optional momentum term $\bm{m}_t \in \mathbb{R}^n$. The step-size should decrease at a $1/T^{0.25}$ rate to achieve convergence. The momentum term is to enhance the robustness of the server-side updates.

At the beginning of the $t$-th round, the server broadcasts $\bm{x}_t$ and $\alpha_0^t$ to all $M$ workers, and the workers use them as their initial solutions and step-sizes respectively (in Lines 3-4). Each worker $i$ then draws a minibatch $\mathcal{D}_i$ of size $b$ randomly\footnote{To simplify the analysis, throughout this work, we assume the minibatch to be drawn uniformly with replacement.} and constructs a stochastic approximated function $f_i$ (in Lines 5-6). The minibatch $\mathcal{D}_i$ is fixed during this round and thus the function $f_i$ is considered as deterministic. The $i$-th worker then optimizes $f_i$ using the deterministic ES with a budget of $K$ iterations. At the $k$-th iteration of the $i$-th worker, we denote respectively the solution and step-size as $\bm{v}_{i,k}^t$ and $\sigma_k^t$. After the worker-side search phase terminates, all workers upload their final output (i.e., $\bm{v}_{i,K}^t$), and then the server computes an averaged descent step, denote by $\bm{d}_{t+1}$, in Line 17. Before the end of the $t$-th round, as shown in Lines 18-19, the server accumulates the descent step into the momentum $\bm{m}_{t+1}$, with a parameter $\beta$ controlling the rate, and finally obtains the new solution $\bm{x}_{t+1}$ via moving $\bm{x}_t$ along the momentum direction. Note that in the final step we do not specify a step-size; the magnitude of how the solution is updated is implicitly controlled by the deterministic ES at the worker-side. This is the critical step for achieving full adaptivity.

\begin{figure}[thb]
\begin{algorithm}[H]
	\caption{DES}
	\small
	\label{alg:DES}
	\begin{algorithmic}[1]
	\Require $\bm{x}_0 \in \mathbb{R}^n$: initial solution; $\alpha \in \mathbb{R}_+$: initial step-size; $\beta \in \left[0,\sqrt{\frac{1}{2\sqrt{2}}}\right)$: momentum parameter; $b \ge \sqrt{T}$: minibatch size
	\For {$t = 0, 1, \cdots, T-1$}
		\For {$i = 1,2,\cdots,M$ \textbf{in parallel}} 
			\State $\bm{v}_{i,0}^t = \bm{x}_t$
			\State $\alpha_0^t = \alpha/(t+1)^{0.25}$
			\State Draw a minibatch $\mathcal{D}_i$ of size $b$ 
			\State Define $f_i(\bm{x}) = \frac{1}{b}\sum_{\bm{\xi} \in \mathcal{D}_i} F(\bm{x};\bm{\xi})$
			\For {$k = 0,1,\cdots,K-1$}
				\State $\alpha_k^t = \alpha_0^t/(k+1)^{0.5}$
				\State Sample $\bm{u}_{i,k}^t$ from $\mathcal{N}(\bm{0},\bm{I})$
				\If{ $f_i\left(\bm{v}_{i,k}^t + \alpha_k^t \bm{u}_{i,k}^t\right) \le f_i(\bm{v}_{i,k}^t)$}
					\State $\bm{v}_{i,k+1}^t = \bm{v}_{i,k}^t + \alpha_k^t \bm{u}_{i,k}^t$
				\Else
					\State $\bm{v}_{i,k+1}^t = \bm{v}_{i,k}^t$
				\EndIf
			\EndFor
		\EndFor
		\State $\bm{d}_{t+1} = \frac{1}{M}\sum_{i=1}^M \bm{v}_{i,K}^t - \bm{x}_t$
		\State $\bm{m}_{t+1} = \beta \bm{m}_t + (1-\beta) \bm{d}_{t+1}$
		\State $\bm{x}_{t+1} = \bm{x}_t + \bm{m}_{t+1}$
	\EndFor
 \end{algorithmic} 
 \end{algorithm}
\end{figure}

\subsection{Convergence properties}
We now analyze the convergence behavior of DES. Firstly we consider a general setting where the optimality is measured by the $\ell_2$ norm of the gradient. 

\begin{theorem} \label{theorem:convergence-DES-l2}
Let \Cref{assumption:smoothness,assumption:variance-boundedness,assumption:iid-data} hold with the self-dual $\ell_2$ norm, i.e., $\|\cdot\| = \|\cdot\|_* = \|\cdot\|_2$.
Assume the function $f$ is bounded below by $f_*$ and choose $0 \le \beta < \sqrt{\frac{1}{2\sqrt{2}}}, b \ge \sqrt{T}$.
The iterations generated by \Cref{alg:DES} satisfy
\begin{multline} \label{eq:convergene-rate-DES-l2-norm}
\frac{1}{T} \sum_{t=0}^{T-1} \mathbb{E}\left[\|\nabla f(\bm{x}_t)\|_2\right] 
\le \frac{\sqrt{2\pi}}{T^{3/4}}\frac{f\left(\bm{x}_0\right) - f_*}{\alpha\sqrt{K}} \\
+ \frac{\sqrt{n}}{T^{1/4}}\left(
	2\alpha L\left(\sqrt{2\pi n} \Psi  
			+ \frac{80\beta\sqrt{K}}{3}\right)
	+ \frac{8\sqrt{2\pi}\sigma}{3}    
	\right)
\end{multline}
where 
\begin{equation} \label{eq:psi-definition}
	\Psi = \left(\left(\frac{2}{1-2\sqrt{2}\beta^2} + \frac{1}{2}\right)\sqrt{K} + \frac{1}{2\sqrt{K}} \right) (1+\log K)+\sqrt{K}
\end{equation}
\end{theorem}

Here we briefly discuss our theoretical result and its implications.
\begin{remark}[Convergence rate] \normalfont
When $K$ is fixed, the DES method achieves $\mathcal{O}\left(T^{-1/4}\right)$ rate in terms of $\frac{1}{T} \sum_{t=0}^{T-1} \mathbb{E}\left[\|\nabla f(\bm{x}_t)\|_2\right]$. If, in addition, setting $\alpha = \Theta(n^{-1/2}L^{-1})$, we achieve
\begin{equation} \label{eq:optimal-rate-DES-l2}
	\left(\frac{1}{T} \sum_{t=0}^{T-1} \mathbb{E}\left[\|\nabla f(\bm{x}_t)\|_2\right]\right)^2 \le \mathcal{O}\left(\sigma^2\frac{n}{\sqrt{T}}\right).
\end{equation}
The dependence on $T$ aligns with the best known bound for zeroth-order stochastic methods, e.g., in \cite{ghadimi_stochastic_2013}, which can be rewritten as
\begin{equation} \label{eq:best-known-bound-zeroth-order-stochastic}
\frac{1}{T}\sum_{t=0}^{T-1} \mathbb{E}\left[\|\nabla f(\bm{x}_t)\|_2^2\right] \le \mathcal{O}\left(\sigma\sqrt{\frac{\Delta_f L n}{T}}\right)
\end{equation}
where $\Delta_f = f\left(\bm{x}_0\right)-f_*$.
Our method has a worse dependence on $\sigma$. However, the best known bound in \cref{eq:best-known-bound-zeroth-order-stochastic} requires $\sigma$ to be known when setting the step-size; so it remains unknown whether the dependence of $\sigma$ is improvable in a real black-box setting. Our obtained rate \cref{eq:optimal-rate-DES-l2}, in fact, matches the rate of adaptive gradient methods~\cite{reddi_adaptive_2020} in terms of the $\sigma$-dependence. The convergence of DES is less dependent on the function landscape characteristics (e.g., $\Delta_f$ and $L$), at the cost of having a worse dimension-dependence. This indicates that DES might suffer from the curse of dimensionality but could be better in handling ill-conditioning and robust to initialization.
\end{remark}

\begin{remark}[Minibatching] \normalfont
The setting $b \ge \sqrt{T}$ is critical in achieving convergence. This requirement is not usual for first-order methods or Gaussian smoothing based zeroth-order methods, since for these methods the gradient variance can be scaled down by choosing a sufficiently small step-size. The DES method only relies on the comparison results among solutions and does not try to estimate the gradient, so the bias of the descent step could accumulate and prevent convergence unless a large minibatch is used to explicitly reduce the noise. Note that similar issues are encountered in the signSGD method~\cite{bernstein_signsgd_2018} where the descent step becomes biased due to the sign operation. signSGD, however, requires $b \ge T$ to achieve convergence whereas in our method it is relaxed to $b \ge \sqrt{T}$.
\end{remark}

\begin{remark}[Adaptivity] \normalfont
The DES method is fully adaptive in the sense that it converges with any valid parameter setting and relies no knowledge about landscape characteristics (e.g., values of $L$ and $\sigma$). In contrast to existing distributed adaptive gradient methods such as \cite{reddi_adaptive_2020,xie_local_2020}, DES does not need the gradient to be uniformly bounded and does not involve a non-adaptive worker-side step-size.
\end{remark}

\begin{remark}[Momentum] \normalfont
The bound in \cref{eq:convergene-rate-DES-l2-norm} suggests that the optimal $\beta$ is 0, but in experiments we found choosing $\beta > 0$ in most cases leads to better performance. This is probably because the suggested rate is overestimated, so it does not reflect how the momentum influences the algorithm performance. The impact of this parameter will be investigated using simulation studies.
\end{remark}

It is found from \cref{eq:optimal-rate-DES-l2} that the DES method suffers a dimension-dependence slowdown in convergence. We note, however, that when the landscape exhibits certain sparse structure, DES may automatically exploit such sparsity and achieve speedup. This is formally stated below:

\begin{theorem} \label{theorem:convergence-DES-l1}
Let \Cref{assumption:smoothness,assumption:variance-boundedness,assumption:iid-data} hold with the $\ell_\infty$ norm, i.e., $\|\cdot\| = \|\cdot\|_\infty$ and $\|\cdot\|_* = \|\cdot\|_1$.
Assume the function $f$ is bounded below by $f_*$ and choose $0 \le \beta < \sqrt{\frac{1}{2\sqrt{2}}}, b \ge \sqrt{T}$. If $\|\nabla f(\bm{x})\|_0 \le s$ for any $\bm{x}\in \mathbb{R}^n$ and some constant $s \le n$, then the iterations generated by \Cref{alg:DES} satisfy
\begin{multline} \label{eq:convergene-rate-DES-l1-norm}
\frac{1}{T}\sum_{t=0}^{T-1} \mathbb{E}[\|\nabla f(\bm{x}_t)\|_1]   
\le \frac{\sqrt{2\pi s}}{T^{3/4}} \frac{f\left(\bm{x}_0\right) - f_*}{\alpha\sqrt{K}} \\
	+ \frac{8\sqrt{\log(\sqrt{2}n)}}{T^{1/4}} \Bigg\{
		\alpha L \left(\sqrt{2\pi s\log(\sqrt{2}n)} \Psi + \frac{10\beta\sqrt{K}}{3}  \right) \\
		 + \frac{2\sqrt{2\pi s} \sigma}{3} \Bigg\}
\end{multline}
where $\Psi$ is defined in \cref{eq:psi-definition}.
\end{theorem}

\begin{remark}[Adaptation to sparsity] \normalfont
The rate established above only poly-logarithmically depends on the dimension. With any setting of $\alpha$ and noting the fact $\|\nabla f(\bm{x})\|_1 \ge \|\nabla f(\bm{x})\|_2$, we have
\[
	\left(\frac{1}{T}\sum_{t=0}^{T-1} \mathbb{E}[\|\nabla f(\bm{x}_t)\|_2]\right)^2 \le \tilde{\mathcal{O}}\left(\frac{\sigma^2}{\sqrt{T}}\right)
\]
which is nearly independent of the dimension, as in most first-order methods. We emphasize the improvement in the dimension-dependence is achieved automatically when the landscape is sparse, without any modification made to the algorithm.
\end{remark}

\section{Alternative Sampling Schemes}
One bottleneck of the DES method implemented in \Cref{sec:DES} is the generation of \rrr{mutation vectors}, in which a huge amount of Gaussian random numbers are required. It is known that generating Gaussian random numbers is usually expensive, and it may cause efficiency issue in high-dimensional settings. In this section we propose two alternative probability models which can be used in DES for improving the sampling efficiency.

\subsection{Mixture sampling for fast mutation}
In \Cref{alg:DES}, each worker has to perturb its maintained solution in all coordinates, leading to the $O(n)$ complexity per-iteration. Our scheme to improve this is to only perturb a small subset of the coordinates. Specifically, at each worker's iteration we uniformly and randomly sample a subset of $l$ coordinates with replacement, where $l \ll n$ is a small integer. Then, on each selected coordinate, we perturb the current solution with a univariate random noise. This two-level sampling strategy yields a mixture distribution since its samples follow a mixture of $n$ univariate probability models defined individually on each coordinate. Statistical characteristics of this mixture distribution is completely determined by the parameter $l$ and the underlying univariate model. In the following we provide two ways in designing the mixture sampling scheme.

The first scheme is to use Gaussian distribution on each selected coordinate and we call it ``\textit{mixture Gaussian sampling}''.  This scheme works via replacing the Gaussian distribution (e.g., $\mathcal{N}(\bm{0},\bm{I})$ in \Cref{alg:DES}) with the probability model defined below:
\begin{definition} \label{definition:mixture-Gaussian-distribution}
We call a random vector $\bm{u} \in \mathbb{R}^n$ is obtained from the mixture Gaussian sampling if it can be expressed as
\[
	\bm{u} = \sqrt{\frac{n}{l}}\sum_{j=1}^l \bm{e}_{r_j} z_j 
\]	
where \rrr{$z_1,\cdots,z_l$ are scalars drawn independently from $\mathcal{N}(0,1)$, and $r_1,\cdots,r_l$ are integers drawn uniformly from $\{1,\cdots,n\}$ with replacement}. We denote its underlying probability model by $\mathcal{M}_l^G$.
\end{definition}

The second scheme is to use, on each selected coordinate, the Rademacher distribution which belongs to the sub-Gaussian family. We call this scheme ``\textit{mixture Rademacher sampling}''. In this case, the mutation vector is drawn from the following distribution:
\begin{definition} \label{definition:mixture-Rademacher-distribution}
We call a random vector $\bm{u} \in \mathbb{R}^n$ is obtained from the mixture Rademacher sampling if it can be expressed as
\[
	\bm{u} = \sqrt{\frac{n}{l}}\sum_{j=1}^l \bm{e}_{r_j} z_j 
\]	
where $z_1,\cdots,z_l$ are independent scalars to be either 1 or -1 with 50\% chance, and $r_1,\cdots,r_l$ are integers drawn uniformly from $\{1,\cdots,n\}$ with replacement. We denote its underlying probability model by $\mathcal{M}_l^R$.
\end{definition}

The coefficient $\sqrt{\frac{n}{l}}$ in the above definitions is to normalize the probability model to achieve the identity covariance matrix, which will be illustrated in the subsequent analyses. When $l \le n$, we can implement the above sampling schemes efficiently in DES, via a loop of length $l$ applied on the solutions maintained at the worker-side. \Cref{alg:DES-mixture-sampling} gives the detailed implementations of this idea. When $ l \ll n$, the time complexity for sampling can be reduced to $O(l)$, and this will save the computing time considerably when $n$ is large.

\begin{figure}[tb]
\begin{algorithm}[H]
	\caption{DES with mixture sampling}
	\small
	\label{alg:DES-mixture-sampling}
	\begin{algorithmic}[1]
	\Require $\bm{x}_0 \in \mathbb{R}^n$: initial solution; $\alpha \in \mathbb{R}_+$: initial step-size; $\beta \in \left[0,\sqrt{\frac{1}{2\sqrt{2}}}\right)$: momentum parameter; $b \ge \sqrt{T}$: minibatch size; $l \in \mathbb{Z}_+$: mixture parameter
	\For {$t = 0, 1, \cdots, T-1$}
		\For {$i = 1,2,\cdots,M$ \textbf{in parallel}} 
			\State $\bm{v}_{i,0}^t = \bm{x}_t$
			\State $\alpha_0^t = \alpha/(t+1)^{0.25}$
			\State Draw a minibatch $\mathcal{D}_i$ of size $b$ 
			\State Define $f_i(\bm{x}) = \frac{1}{b}\sum_{\bm{\xi} \in \mathcal{D}_i} F(\bm{x};\bm{\xi})$
			\For {$k = 0,1,\cdots,K-1$}
				\State $\alpha_k^t = \alpha_0^t/(k+1)^{0.5}$
				\State $\bm{w} = \bm{v}_{i,k}^t$
				\For {$j = 1,\cdots,l$}
					\State Draw $r$ randomly uniformly from $\{1,\cdots,n\}$ with replacement
					\State Option I (mixture Gaussian sampling):
					\Statex \qquad\qquad\qquad\qquad $z \sim \mathcal{N}(0,1)$
					\State Option II (mixture Rademacher sampling):
					\Statex \qquad\qquad\qquad\qquad $z$ is either -1 or 1 with 50\% chance
					\State $w_r = w_r + \alpha_k^t\sqrt{\frac{n}{l}}z$
				\EndFor
				\If{ $f_i(\bm{w}) \le f_i(\bm{v}_{i,k}^t)$}
					\State $\bm{v}_{i,k+1}^t = \bm{w}$
				\Else
					\State $\bm{v}_{i,k+1}^t = \bm{v}_{i,k}^t$
				\EndIf
			\EndFor
		\EndFor
		\State $\bm{d}_{t+1} = \frac{1}{M}\sum_{i=1}^M \bm{v}_{i,K}^t - \bm{x}_t$
		\State $\bm{m}_{t+1} = \beta \bm{m}_t + (1-\beta) \bm{d}_{t+1}$
		\State $\bm{x}_{t+1} = \bm{x}_t + \bm{m}_{t+1}$
	\EndFor
 \end{algorithmic} 
 \end{algorithm}
\end{figure}

\subsection{Behavior of DES with mixture sampling}
We first discuss the statistic characteristics of proposed two sampling schemes.

The mixture Gaussian sampling, in the case of $l \rightarrow \infty$, will degenerate to the standard Gaussian sampling. This limiting case, to some extent, is useless as it will make the sampling even more expensive.
Therefore, we are more interested in the $l \ll n$ case. The following describes the statistical properties that are required in understanding the mixture sampling schemes. Since the probability model is symmetric by design, we will focus on its second-order and fourth-order moments.

\begin{proposition} \label{proposition:properties-of-mixture-Gaussian-sampling}
Let $l \in \mathbb{Z}_+$ and $\bm{u} \in \mathbb{R}^n$.
If $\bm{u} \sim \mathcal{M}^G_l$, we have $\mathbb{V}[\bm{u}] = \bm{I}$ and
\begin{equation} \label{eq:fourth-order-moment-mixture-Gaussian}
	\mathbb{E}[|\bm{y}^T\bm{u}|^4] = 3\left(\frac{n}{l}\|\bm{y}\|_4^4 + \frac{l-1}{l}\|\bm{y}\|_2^4\right),\;\;\forall \bm{y} \in \mathbb{R}^n.
\end{equation}
\end{proposition}

The above shows that the mixture Gaussian sampling will generate mutation vectors having exactly the same covariance matrix as the standard Gaussian sampling, regardless of the $l$ value. In addition, since $n \|\bm{y}\|_4^4 \ge \|\bm{y}\|_2^4 \ge \|\bm{y}\|_4^4$, we know
\[
	\frac{\mathbb{E}[|\bm{y}^T\bm{u}|^4]}{\|\bm{y}\|_2^4} \in \left[3,3\frac{n+l-1}{l}\right],
\]
which then indicates that any 1-dimensional projection of $\mathcal{M}_l^G$ will have a larger kurtosis than Gaussian. Implications of this property are twofold. Firstly, the mixture Gaussian sampling method is more likely to generate outliers in the mutation phase, so if the landscape is highly multimodal, DES equipped with $\mathcal{M}_l^G$ would have a greater chance to escape local optima. Secondly, this makes DES prefer exploration than exploitation, and hence, it may degrade the performance. We note, \rrr{as will be demonstrated later, that such a performance degradation is insignificant when the gradient is dense}.

Similarly, the mixture Rademacher sampling can be characterized as below.

\begin{proposition} \label{proposition:properties-of-mixture-Rademacher-sampling}
Let $l \in \mathbb{Z}_+$ and $\bm{u} \in \mathbb{R}^n$.
If $\bm{u} \sim \mathcal{M}^R_l$, we have $\mathbb{V}[\bm{u}] = \bm{I}$ and
\begin{equation} \label{eq:fourth-order-moment-mixture-Rademacher}
	\mathbb{E}[|\bm{y}^T\bm{u}|^4] = \frac{n}{l}\|\bm{y}\|_4^4 + 3\frac{l-1}{l}\|\bm{y}\|_2^4,\;\;\forall \bm{y} \in \mathbb{R}^n.
\end{equation}	
\end{proposition}

Again, the mixture Rademacher sampling is more likely to produce outlier mutation vectors than the standard Gaussian sampling, while they have the same covariance matrix. 
But it is found, via comparing \cref{eq:fourth-order-moment-mixture-Rademacher} with \cref{eq:fourth-order-moment-mixture-Gaussian}, that $\mathcal{M}_l^R$ can scale down the kurtosis of $\mathcal{M}_l^G$ by a factor about $1/3$ for sufficiently large $n$. In this sense, the mixture Rademacher sampling can be considered as a trade-off between the standard Gaussian sampling and the mixture Gaussian sampling.

In the following, we analyze the convergence performance of DES when equipped with the mixture sampling schemes. For expository purposes, we assume $\beta=0$ and only consider the $\ell_2$ norm case, though our analysis can be extended directly to a more general setting.

\begin{theorem} \label{theorem:convergence-DES-mixture-Gaussian-l2}
Let \Cref{assumption:smoothness,assumption:variance-boundedness,assumption:iid-data} hold with the self-dual $\ell_2$ norm, i.e., $\|\cdot\| = \|\cdot\|_* = \|\cdot\|_2$.
Assume the function $f$ is bounded below by $f_*$ and choose $\beta = 0, b \ge \sqrt{T}$.
\rrr{If $\|\nabla f(\bm{x})\|_2^4 / \|\nabla f(\bm{x})\|_4^4 \ge \tilde{s}$ for any $\bm{x}\in \mathbb{R}^n$ and some constant $\tilde{s} \in [1,n]$,}
then the iterations generated by \Cref{alg:DES-mixture-sampling} with mixture Gaussian sampling satisfy
\begin{multline} \label{eq:convergene-rate-DES-mixtur-Gaussian-sampling-l2-norm}
		\frac{1}{T}
	\rrr{\sum_{t =0}^{T-1}} \mathbb{E}\left[\left\|\nabla f\left(\bm{x}_t\right)\right\|_2\right]
	\le \rrr{\sqrt{3 + \frac{3n}{\tilde{s}l}}} \left\{\frac{2}{T^{3/4}}\frac{f\left(\bm{x}_0\right)-f_*}{\alpha\sqrt{K}} \right. \\ 
			\left. + \frac{4\sqrt{n}}{T^{1/4}}\left(\frac{4}{3}\sigma + L\sqrt{n} \hat{\Psi} \alpha\right)\right\}
\end{multline}
where 
\[
	\hat{\Psi} = \left(\frac{1}{2\sqrt{K}} + \frac{5}{2}\sqrt{K}\right)(1+\log K) + \frac{1}{\sqrt{K}}.
\]
\end{theorem}

\rrr{
\begin{remark}[Impact of the denseness] \normalfont
The bound in \cref{eq:convergene-rate-DES-mixtur-Gaussian-sampling-l2-norm} is generally looser than that for DES with standard Gaussian sampling.
For example, consider setting $\alpha = \Theta (n^{-1/2} L^{-1})$, then we obtain the convergence rate
\[
	\left(\frac{1}{T} \sum_{t=0}^T \mathbb{E}[\|\nabla f(\bm{x}_t)\|_2]\right)^2 
	\le \mathcal{O}\left(\frac{\sigma^2 n^2}{\tilde{s}l\sqrt{T}}\right),
\]
which could be $\mathcal{O}\left(\frac{n}{\tilde{s}l}\right)$ times slower than the rate given in \cref{eq:optimal-rate-DES-l2}. 
The involved constant $\tilde{s}$, by the definition of vector norms, always exists in the range $[1,n]$.
In fact, as has been pointed out in \cite{hurley_comparing_2009}, the quantity $\|\bm{y}\|_4^4/\|\bm{y}\|_2^4$ measures the sparseness of a vector $\bm{y} \in \mathbb{R}^n$, so the constant $\tilde{s}$ here can be viewed as a lower bound of the denseness of the gradient $\nabla f(\bm{x})$. If the gradient is relatively dense, e.g., all coordinates in the gradient are of a similar magnitude, then $\tilde{s}$ will be close to $n$. In this case, the convergence rate with mixture Gaussian sampling will coincide with that with standard Gaussian sampling. 
We may therefore conclude, by comparing \Cref{theorem:convergence-DES-mixture-Gaussian-l2,theorem:convergence-DES-l1}, that the mixture sampling is more suitable for dense problems whereas the standard Gaussian sampling is preferred for sparse problems.
\end{remark}}

\begin{theorem} \label{theorem:convergence-DES-mixture-Rademacher-l2}
Let \Cref{assumption:smoothness,assumption:variance-boundedness,assumption:iid-data} hold with the self-dual $\ell_2$ norm, i.e., $\|\cdot\| = \|\cdot\|_* = \|\cdot\|_2$.
Assume the function $f$ is bounded below by $f_*$ and choose $\beta = 0, b \ge \sqrt{T}$.
\rrr{If $\|\nabla f(\bm{x})\|_2^4 / \|\nabla f(\bm{x})\|_4^4 \ge \tilde{s}$ for any $\bm{x}\in \mathbb{R}^n$ and some constant $\tilde{s} \in [1,n]$,}
then the iterations generated by \Cref{alg:DES-mixture-sampling} with mixture Rademacher sampling satisfy
\begin{multline} \label{eq:convergene-rate-DES-mixtur-Rademacher-sampling-l2-norm}
	\frac{1}{T}
	\rrr{\sum_{t =0}^{T-1}} \mathbb{E}\left[\left\|\nabla f\left(\bm{x}_t\right)\right\|_2\right]
	\le \rrr{\sqrt{3 + \frac{n}{\tilde{s}l}}} \left\{\frac{2}{T^{3/4}}\frac{f\left(\bm{x}_0\right)-f_*}{\alpha\sqrt{K}} \right. \\ 
			\left. + \frac{4\sqrt{n}}{T^{1/4}}\left(\frac{4}{3}\sigma + L\sqrt{n} \hat{\Psi} \alpha\right)\right\}
\end{multline}
where $\hat{\Psi}$ is defined as in \Cref{theorem:convergence-DES-mixture-Gaussian-l2}.
\end{theorem}

The bound corresponding to the mixture Rademacher sampling \rrr{is slightly tighter than that for the mixture Gaussian sampling}. This could make a considerable difference in practice when $n$ is large. Our empirical studies show that in certain cases the mixture Rademacher sampling could be better than the mixture Gaussian sampling, while their performance is \rrr{in general} similar.

\section{Simulation Study}
In this section we perform simulations to investigate the empirical performance of the proposed methods. 

\subsection{Experimental settings}
We consider three binary classification problems arising in machine learning and statistics. They include \rrr{logistic} regression (LR), nonconvex support vector machine (NSVM), and linear support vector machine (LSVM) with a hinge loss. For these problems, the random sample $\bm{\xi}$ corresponds to a pair of \rrr{input vector $\bm{z}$ and target label $y$}, and the objective function takes the finite-sum form:
\begin{equation*} 
	f(\bm{x}) = \frac{1}{N} \sum_{i = 1}^N F(\bm{x};\bm{\xi}_i) 
	= \rrr{\frac{1}{N} \sum_{i = 1}^N} loss(\bm{x};\bm{z}_i,y_i) +\frac{\lambda_p}{2} \|\bm{x}\|^2_2,
\end{equation*}
where $loss$ is the loss function and  $\lambda_p$ is the regularization parameter. We fix $\lambda_p = 10^{-6}$ throughout this study. The loss function is defined as
\begin{itemize}
	\item Logistic Regression (LR)
	\begin{equation*}
		loss(\bm{x};\bm{z},y) = \log(1 + \exp(-y (\bm{x}^T\bm{z}))) 
	\end{equation*}
	\item Nonconvex Support Vector Machine (NSVM)
	\begin{equation*}
		loss(\bm{x};\bm{z},y) = 1 - \tanh(y (\bm{x}^T\bm{z})) 
	\end{equation*}
	\item Linear Support Vector Machine (LSVM)
	\begin{equation*}
		loss(\bm{x};\bm{z},y) = \max\left\{0, 1 - y (\bm{x}^T \bm{z})\right\}.
	\end{equation*}
\end{itemize}
LR is the simplest, being strongly convex and smooth. NSVM is nonconvex but smooth. LSVM is not smooth so it does not meet our assumption; we choose it to verify the robustness of our proposals.

Six datasets\footnotetext{All datasets are available at \url{https://www.csie.ntu.edu.tw/~cjlin/libsvmtools/datasets}. The mnist dataset is transformed into binary class based on whether the label (digital) is grater than 4.} widely used for benchmarking stochastic optimization methods are selected and their properties are briefly summarized in \Cref{tab:statistics-datasets}. For each dataset, 80\% data are chosen for training and the remaining 20\% are for testing. We partition the training samples uniformly into $M$ pieces with no overlap, and each piece is stored at a counterpart worker. 

\begin{table}[thb]
  \caption{Statistics of the used datasets.}
  \label{tab:statistics-datasets}
  \centering
  \setlength{\tabcolsep}{20pt}
\begin{tabular}{cccc}
\toprule
dataset & $n$     & $N$     \\
\midrule
ijcnn1 			& 22	& 49990		\\
SUSY 				& 18	& 5000000	\\
covtype 			& 54	& 581012	\\
mnist 			& 780 	& 60000		\\
real-sim 			& 20958	& 72309		\\
rcv1 				& 47236 & 677399 	\\
\bottomrule
\end{tabular}%
\end{table}

We implement four algorithms for comparison, including federated zeroth-order gradient method (Fed-ZO-GD), federated zeroth-order SGD (Fed-ZO-SGD), zeroth-order signSGD method (ZO-signSGD), and the standard ES with cumulative step-size adaptation (ES-CSA). 
Fed-ZO-GD, Fed-ZO-SGD, and ZO-signSGD are distributed algorithms based on gradient estimation. ES-CSA is non-distributed but we have made some modifications to enable distributed optimization.
Their configurations are described below:
\begin{itemize}
	\item Fed-ZO-GD. It is implemented by replacing the worker-side solver of DES with the Gaussian smoothing based gradient descent method, so it can be considered as a plain combination of FedAvg and the zeroth-order gradient descent method. Each worker individually chooses a random minibatch of size $b$ in each round and runs zeroth-order gradient descent for $K' = K/2$ iterations with the step-size $\alpha_k^t = \frac{\alpha_0^t}{k+1} = \frac{\alpha}{(k+1)\sqrt{t+1}}$. We choose the central-difference in Gaussian smoothing, so each worker takes about $Kb$ function evaluations per round.

	\item Fed-ZO-SGD. It is a zeroth-order extension of the standard federated SGD algorithm, where each worker's iteration uses an individually random minibatch of size $b$. Each worker's SGD runs for $K' = K/2$ iterations with the step-size $\alpha_k^t = \frac{\alpha_0^t}{\sqrt{k+1}} = \frac{\alpha}{\sqrt{(k+1)(t+1)}}$. It uses the same setting for Gaussian smoothing as in Fed-ZO-GD.

	\item ZO-signSGD. This method is originally proposed in \cite{liu_signsgd_2019} and we adopt its variant with majority vote for distributed optimization. In each round, each worker computes $K' = K/2$ gradient estimators, takes the sign of their average, and then uploads the result to the server. Each gradient estimator is obtained from a central-difference Gaussian smoothing over a minibatch batch of size $b$. The server performs global updates using the sign vector with step-size $\alpha^t = \frac{\alpha}{\sqrt{t+1}}$. 

	\item ES-CSA. We use the standard $(\mu;\lambda)$-ES described in \cite{hansen_evolution_2015} with slight modifications for date decentralization. In each round, the server generates a population of $\lambda$ solutions with a standard multivariate Gaussian distribution and broadcasts the whole population to each worker. The workers then evaluate the population with their local data. The server sums up, for each solution, the results collected from the workers and obtain the corresponding objective value. The best $\lambda$ ones in the population are chosen and their recombination becomes the new population mean. In this setting, each worker takes $\lambda \frac{N}{M}$ function evaluations per round. The standard cumulative step-size adaptation is used and the initial step-size \rrr{is set to $\alpha$}.
\end{itemize}

For the three gradient-based methods, we use the central-difference Gaussian smoothing which takes two function evaluations on each data sample; so the setting $K'=K/2$ ensures that the total number of function evaluations per round and per worker is $Kb$, being consistent with DES. For CSA-ES, the population size is set to $\lambda = MKb/N$; under this setting, all algorithms have exactly the same number of function evaluations per round.

For all algorithms, we choose $b = 1000$, $M=10$. We choose $K = 100$ if $n \le 100$ and 500 if $n > 100$. Each algorithm is assigned with a budget of $EN$ function evaluations, where $E=1000$ if $n \le 100$ and 5000 if $n > 100$. For algorithms relying on Gaussian smoothing, the finite-difference radius is $\mu = 10^{-6}$. The momentum parameter in DES is set to $\beta = 0.5$. All algorithms are run for 8 times individually for each pair of dataset and problem and the median results are reported. DES with the mixture Gaussian sampling and the mixture Rademacher sampling schemes are denoted by DES-mG and DES-mR, respectively; their mixture parameter is set to $l=8$.

\subsection{Overall performance} \label{ss:overall-performance}
We first test DES as well as the competitors on all three problems and over all six datasets.
The initial step-size $\alpha$ for each algorithm is chosen from $\{0.1, 1, 10\}$ using a grid-search.
\rrr{\Cref{fig:training-curve-part-one,fig:training-curve-part-two} report the convergence behavior of the algorithms, measured with the median training error versus the number of rounds.}
It is found that the DES methods with either standard Gaussian sampling or mixture sampling are the best performers in all cases. Belonging to the same ES family, our implementation of DES is significantly better than the non-distributed implementation of ES-CSA, where the latter performs the worst in most cases. This is caused by that the standard ES is for deterministic optimization and does not explore the stochastic characteristics of the objective function. 
Fed-ZO-SGD is the best one among the competitors and is competitive with DES in certain cases. Fed-ZO-GD, in most cases, is not competitive with DES, ZO-signSGD, or Fed-ZO-SGD.

\begin{figure*}[tb] 
\centering
\subfloat{\includegraphics[width=0.6\textwidth]{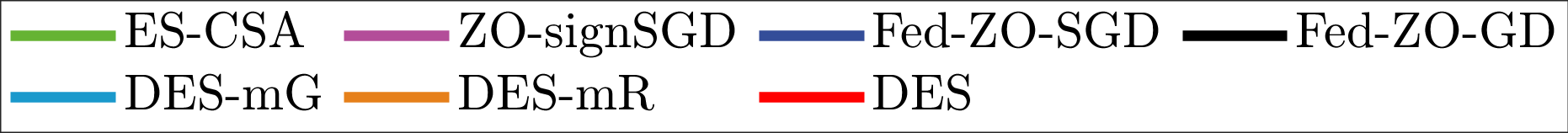}} \\[-2ex]
\addtocounter{subfigure}{-1}
\subfloat[LR, rcv1]{\includegraphics[width=0.33\textwidth]{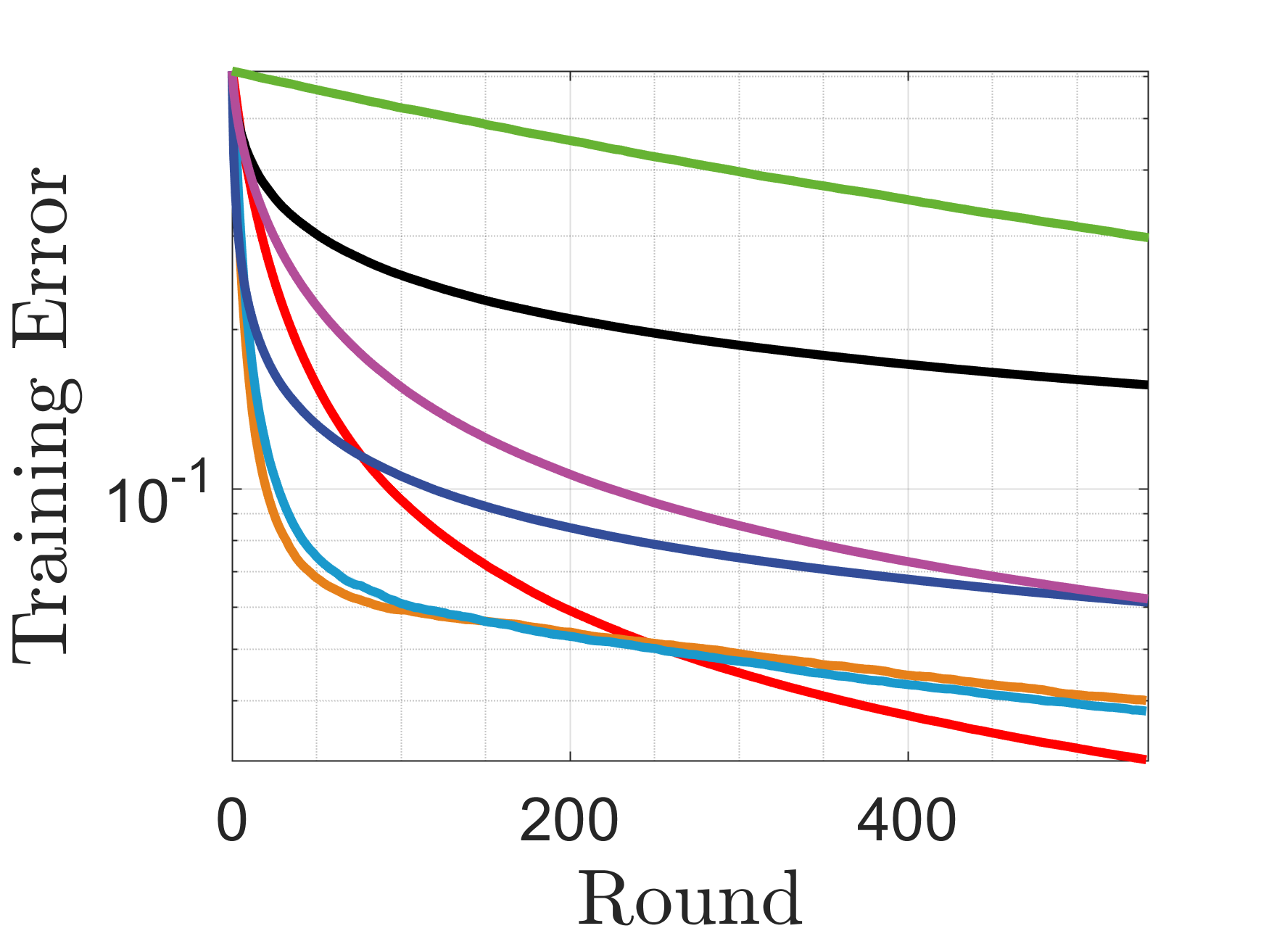}}
\subfloat[NSVM, rcv1]{\includegraphics[width=0.33\textwidth]{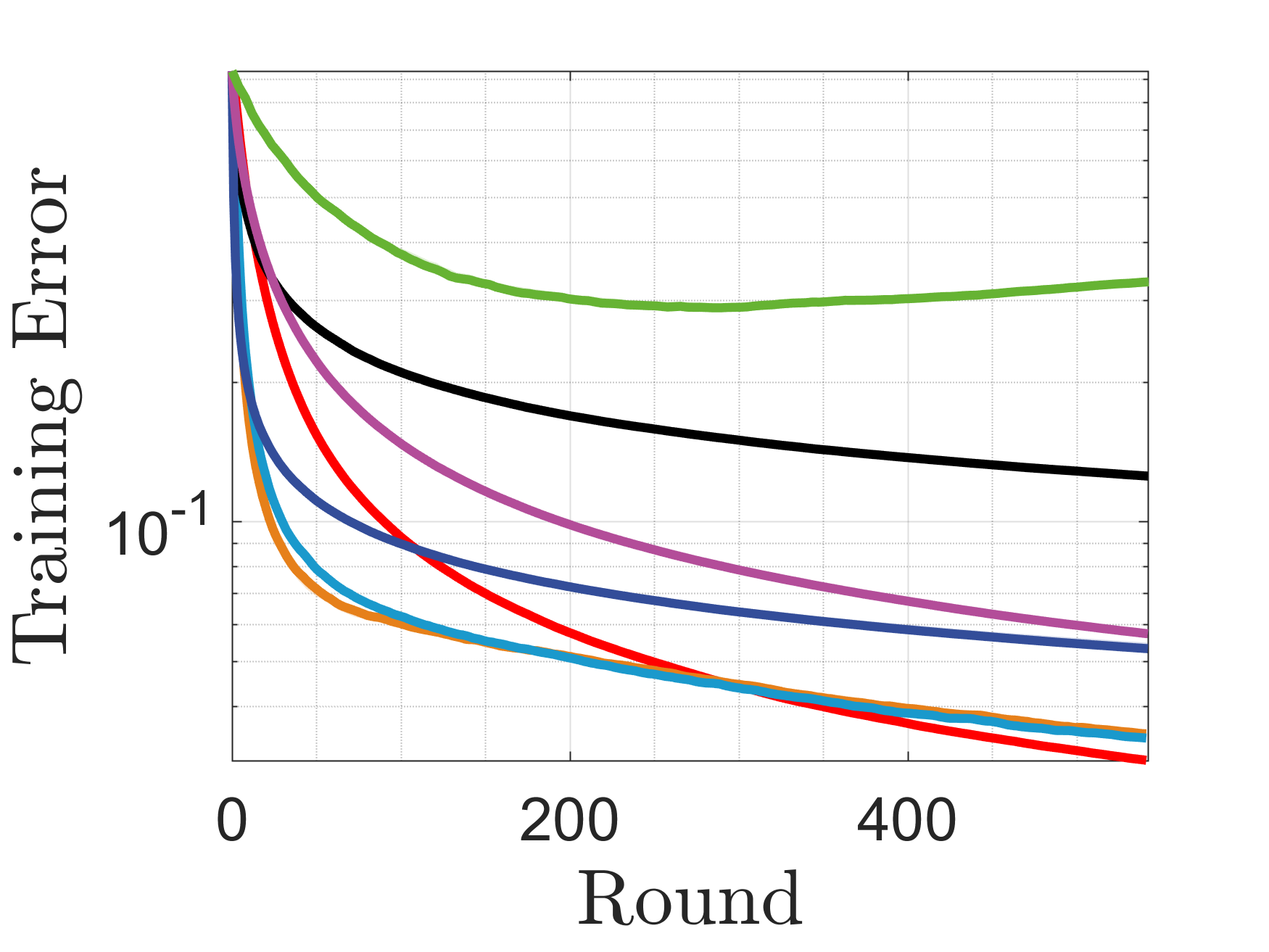}}
\subfloat[LSVM, rcv1]{\includegraphics[width=0.33\textwidth]{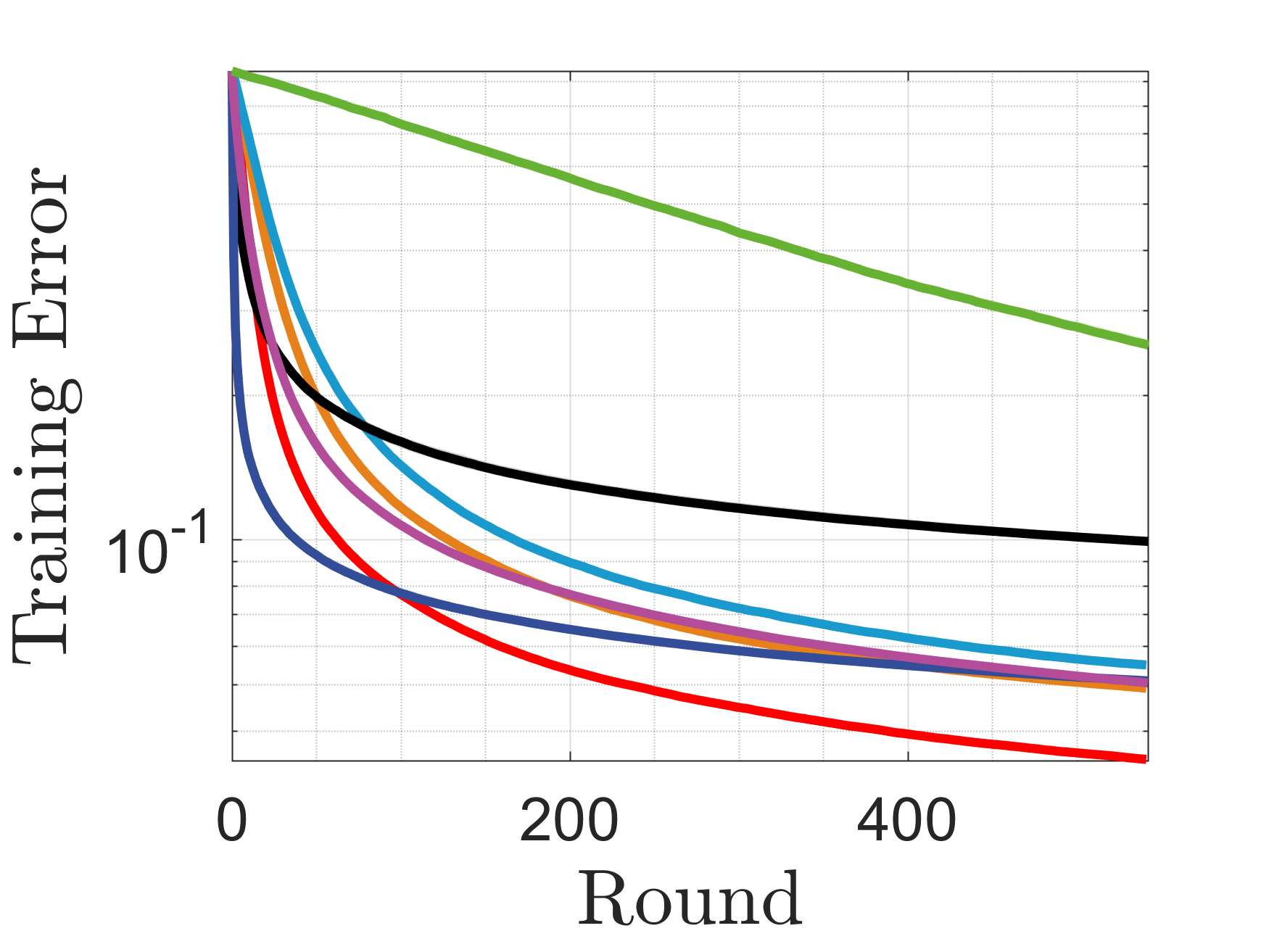}}
\hfil
\subfloat[LR, SUSY]{\includegraphics[width=0.33\textwidth]{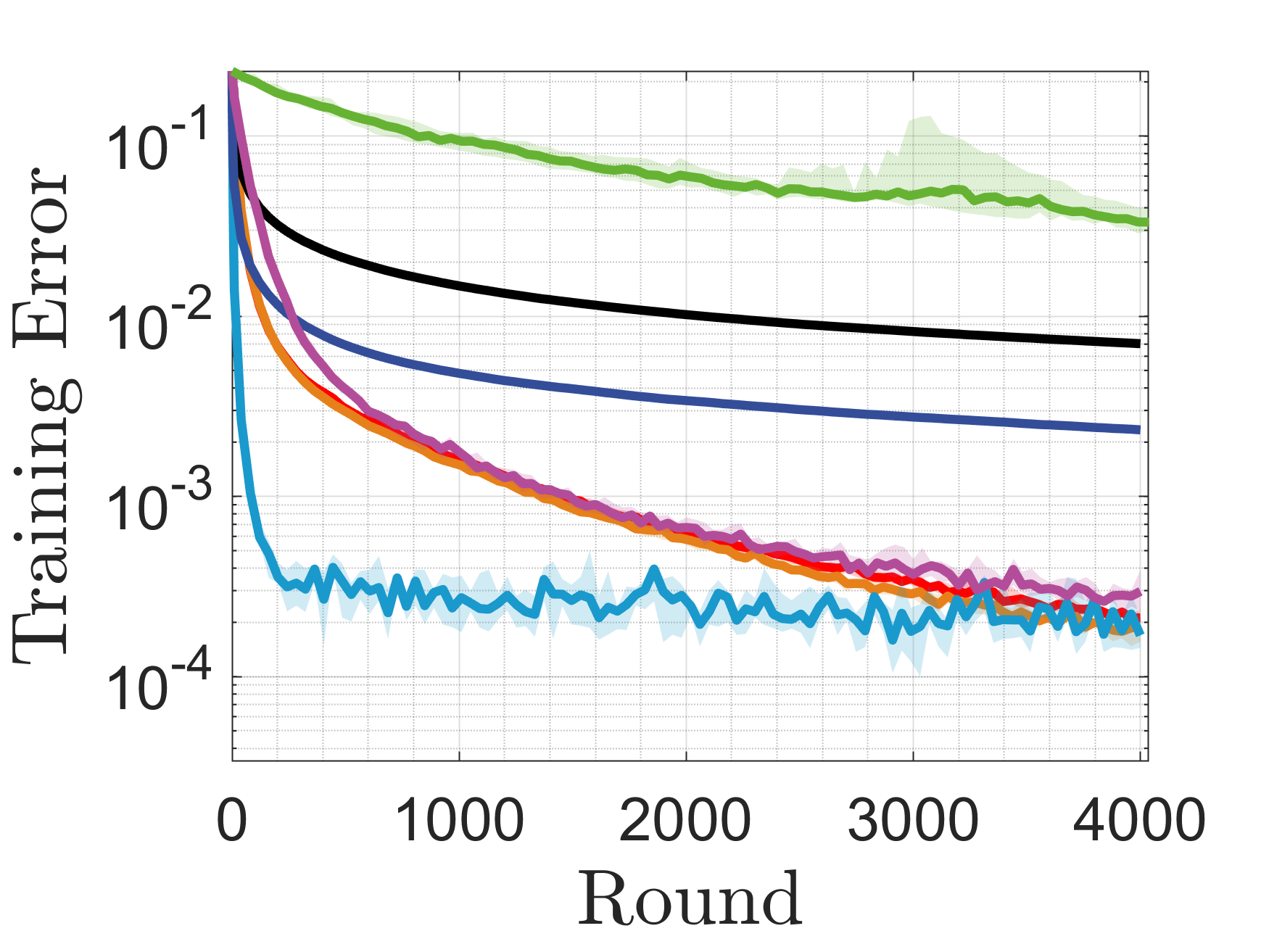}}
\subfloat[NSVM, SUSY]{\includegraphics[width=0.33\textwidth]{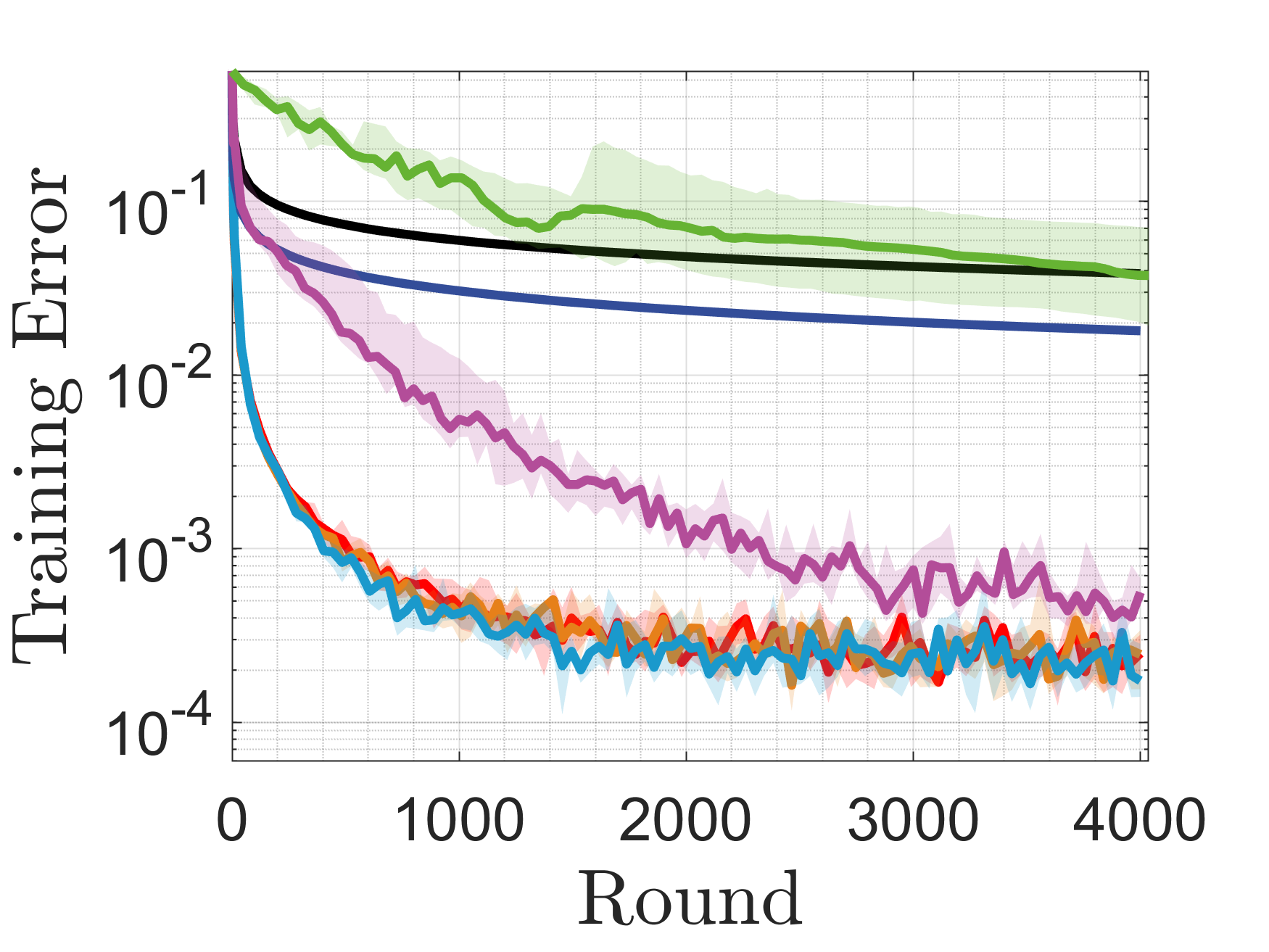}}
\subfloat[LSVM, SUSY]{\includegraphics[width=0.33\textwidth]{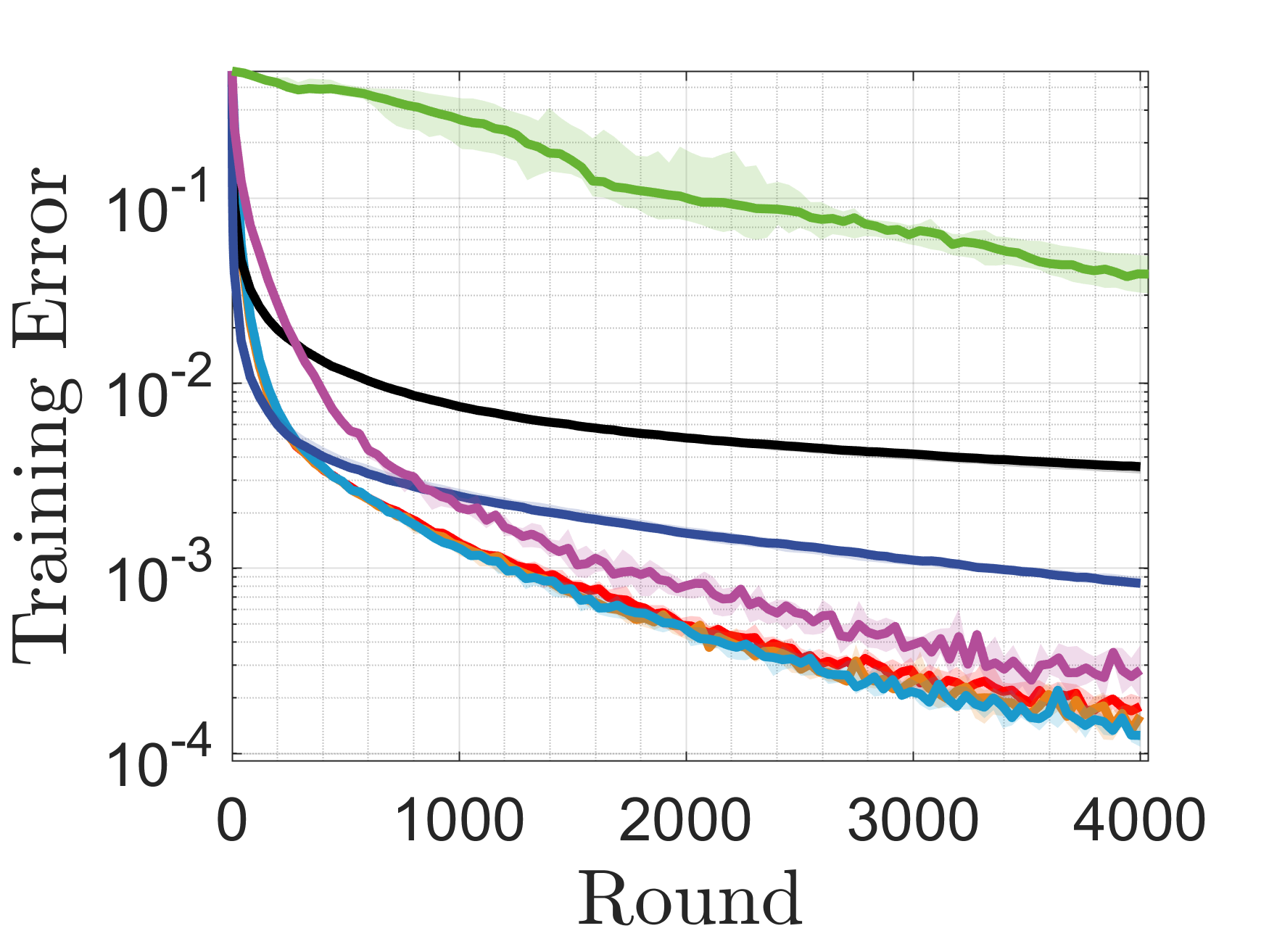}}
\hfil
\subfloat[LR, mnist]{\includegraphics[width=0.33\textwidth]{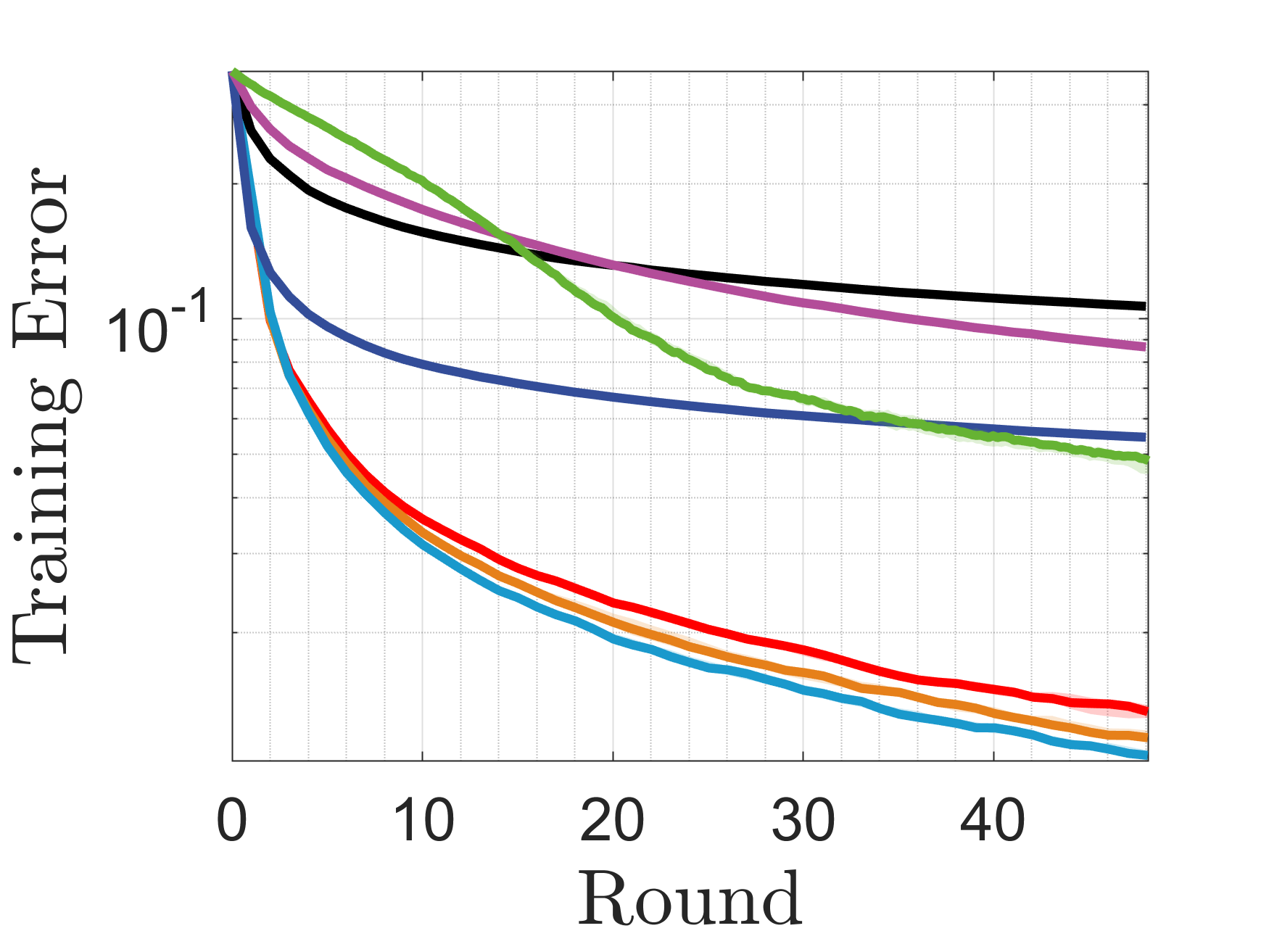}}
\subfloat[NSVM, mnist]{\includegraphics[width=0.33\textwidth]{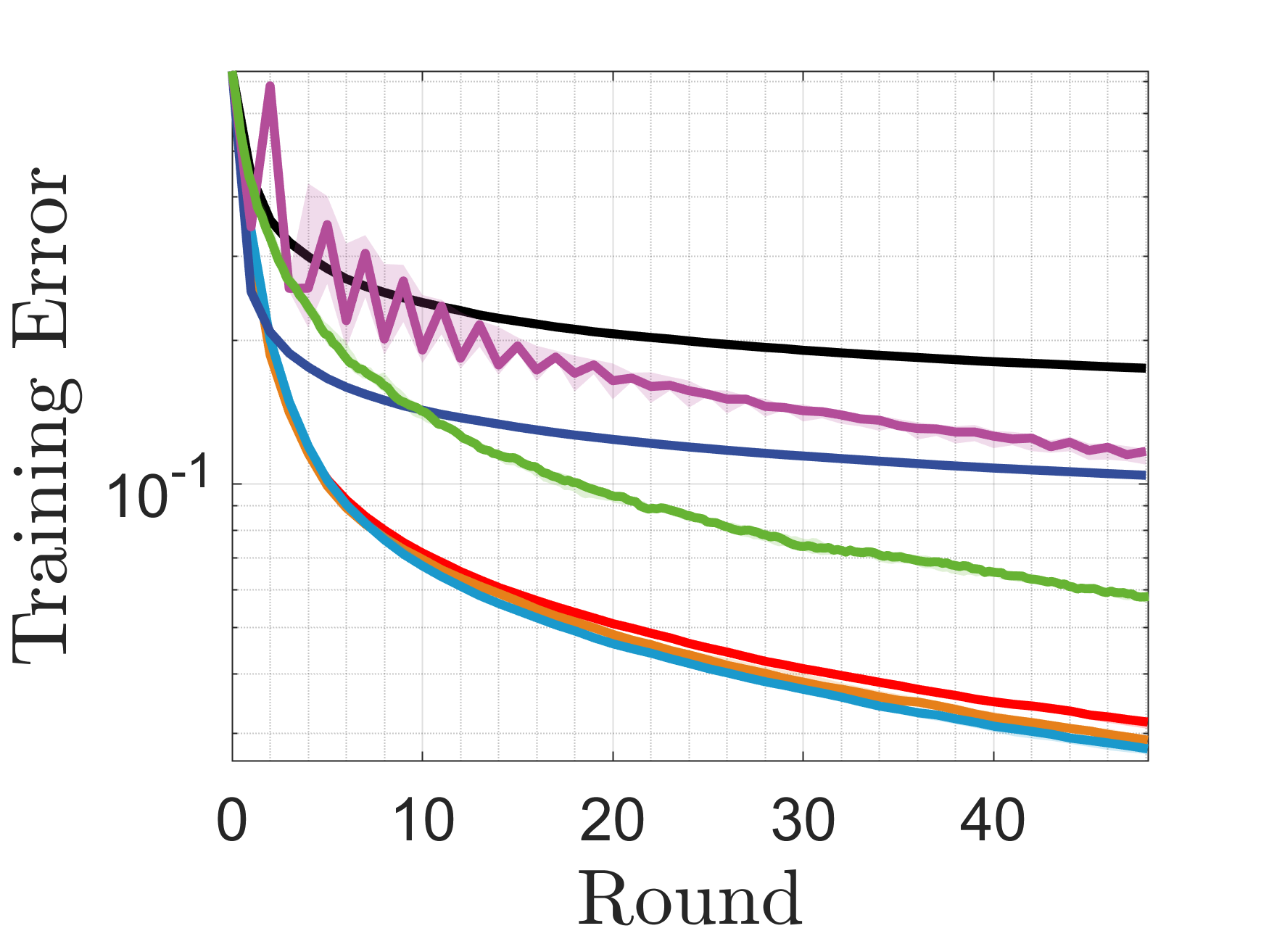}}
\subfloat[LSVM, mnist]{\includegraphics[width=0.33\textwidth]{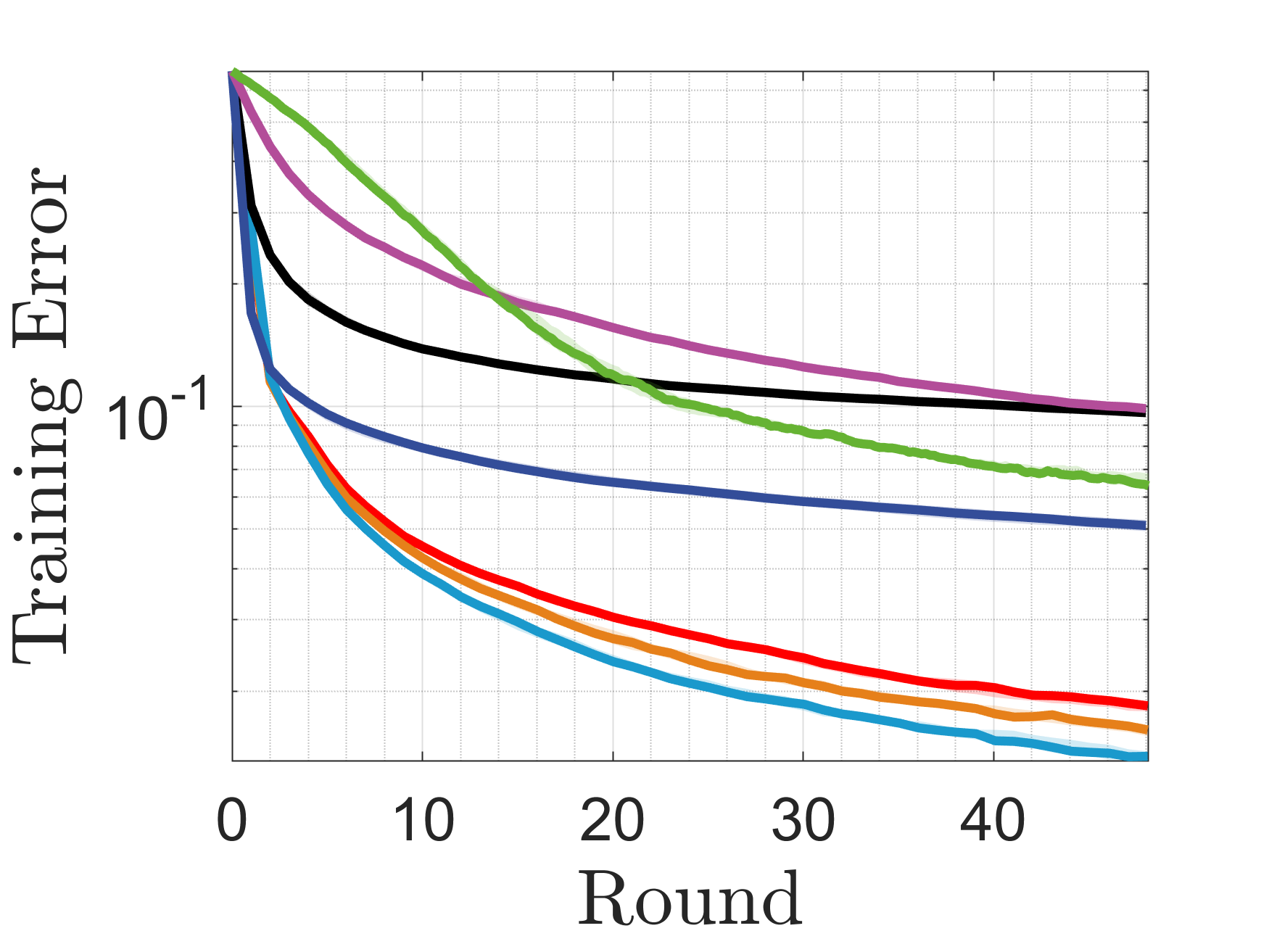}}
\caption{Comparison on \rrr{rcv1, SUSY, and mnist datasets}. The curve displays the training error versus the number of rounds and the corresponding shaded area extends from the 25th to 75th percentiles over the results obtained from all independent runs.}
\label{fig:training-curve-part-one}
\end{figure*}

\begin{figure*}[tb]
\centering
\subfloat{\includegraphics[width=0.6\textwidth]{figs/legend_curve}} \\[-2ex]
\addtocounter{subfigure}{-1}
\subfloat[LR, real-sim]{\includegraphics[width=0.33\textwidth]{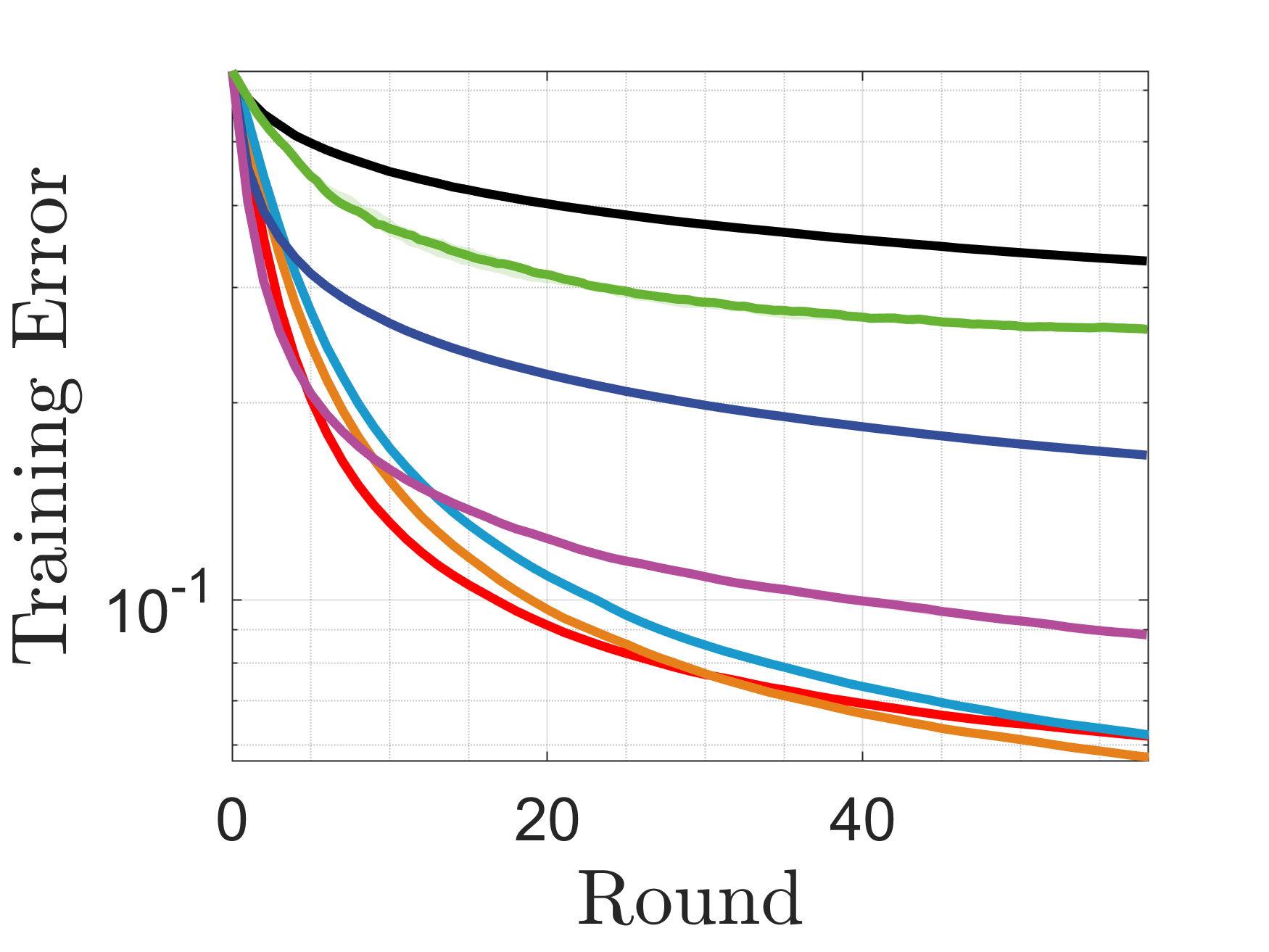}}
\subfloat[NSVM, real-sim]{\includegraphics[width=0.33\textwidth]{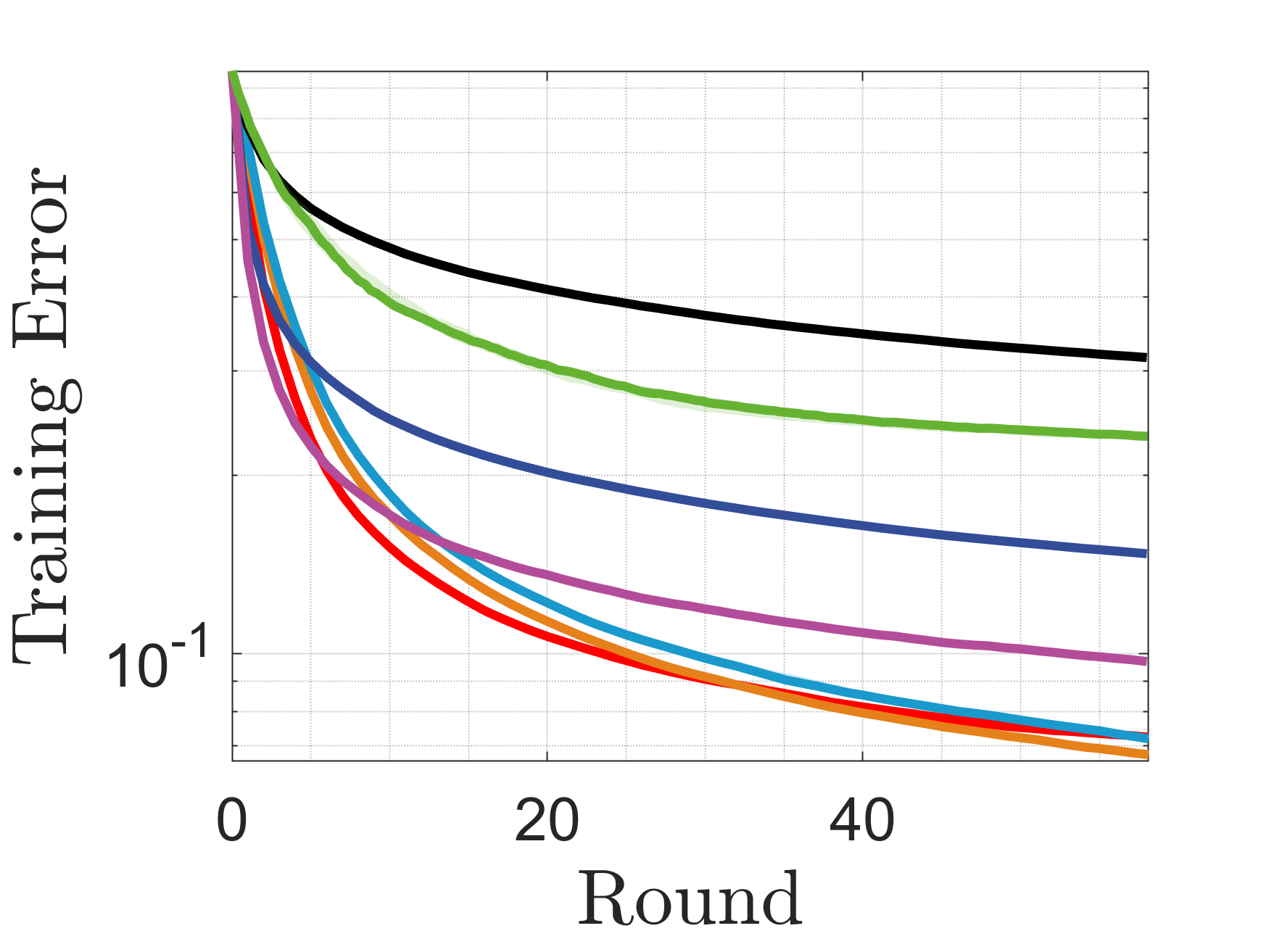}}
\subfloat[LSVM, real-sim]{\includegraphics[width=0.33\textwidth]{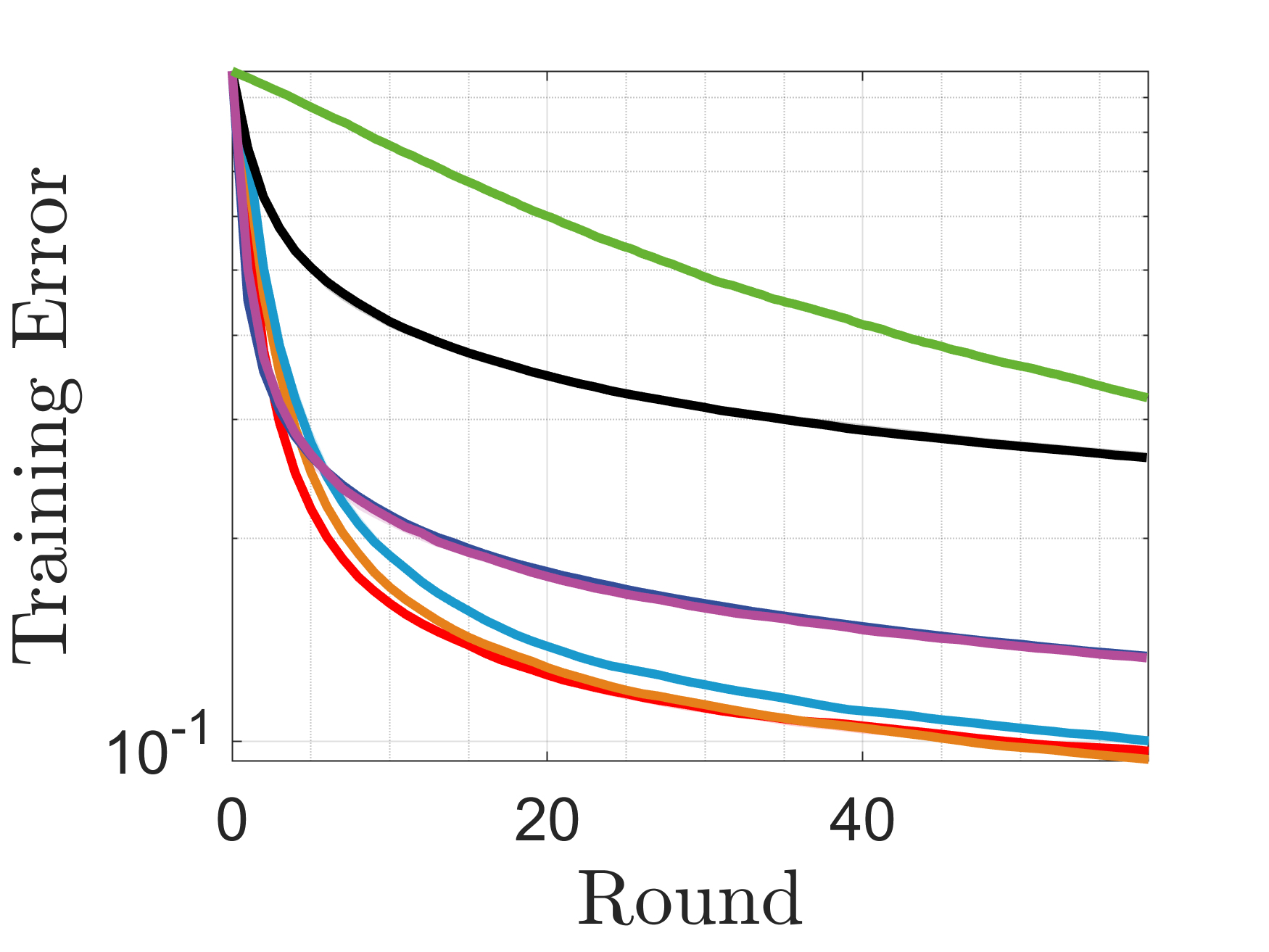}}
\hfil
\subfloat[LR, ijcnn1]{\includegraphics[width=0.33\textwidth]{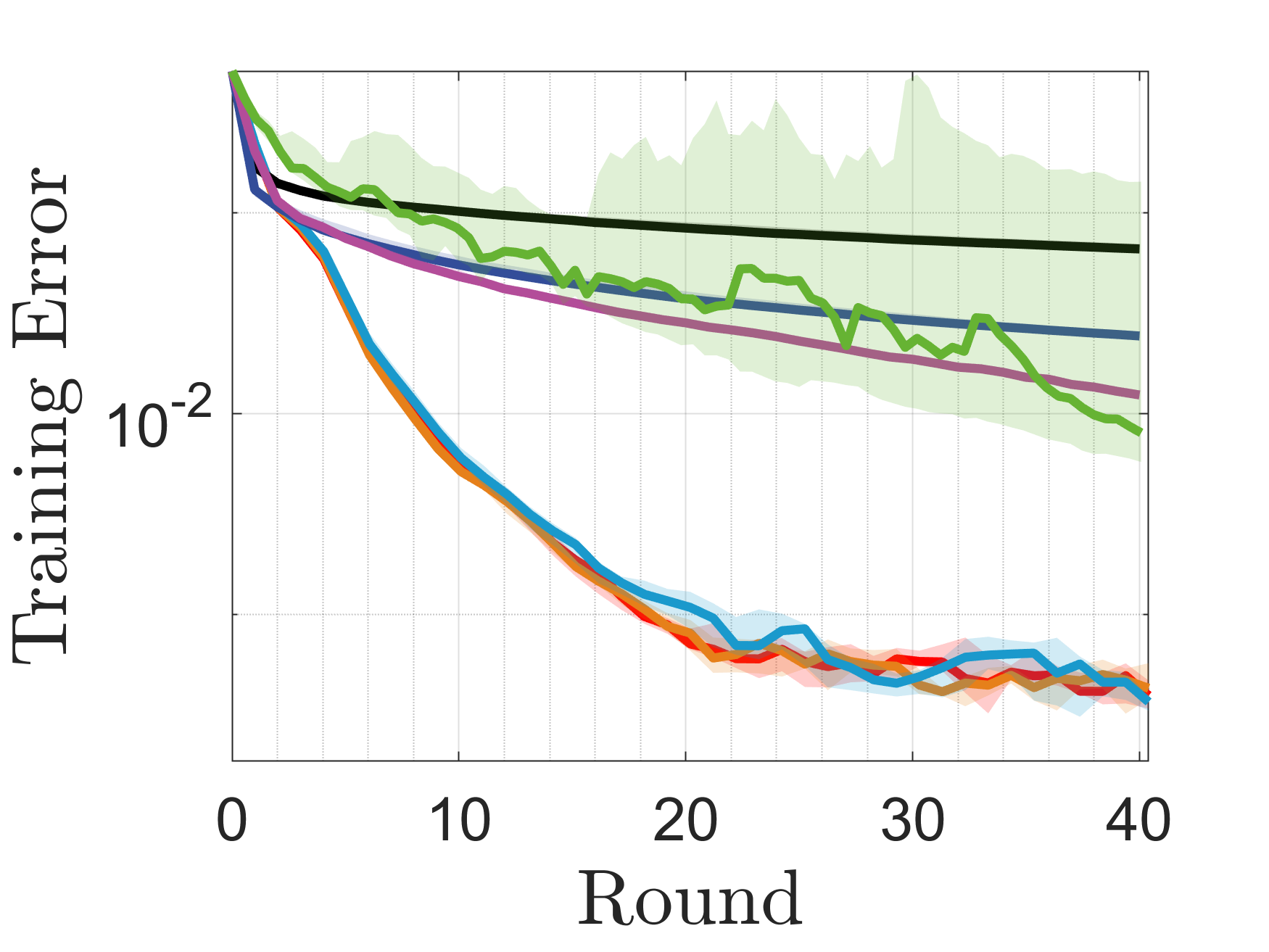}}
\subfloat[NSVM, ijcnn1]{\includegraphics[width=0.33\textwidth]{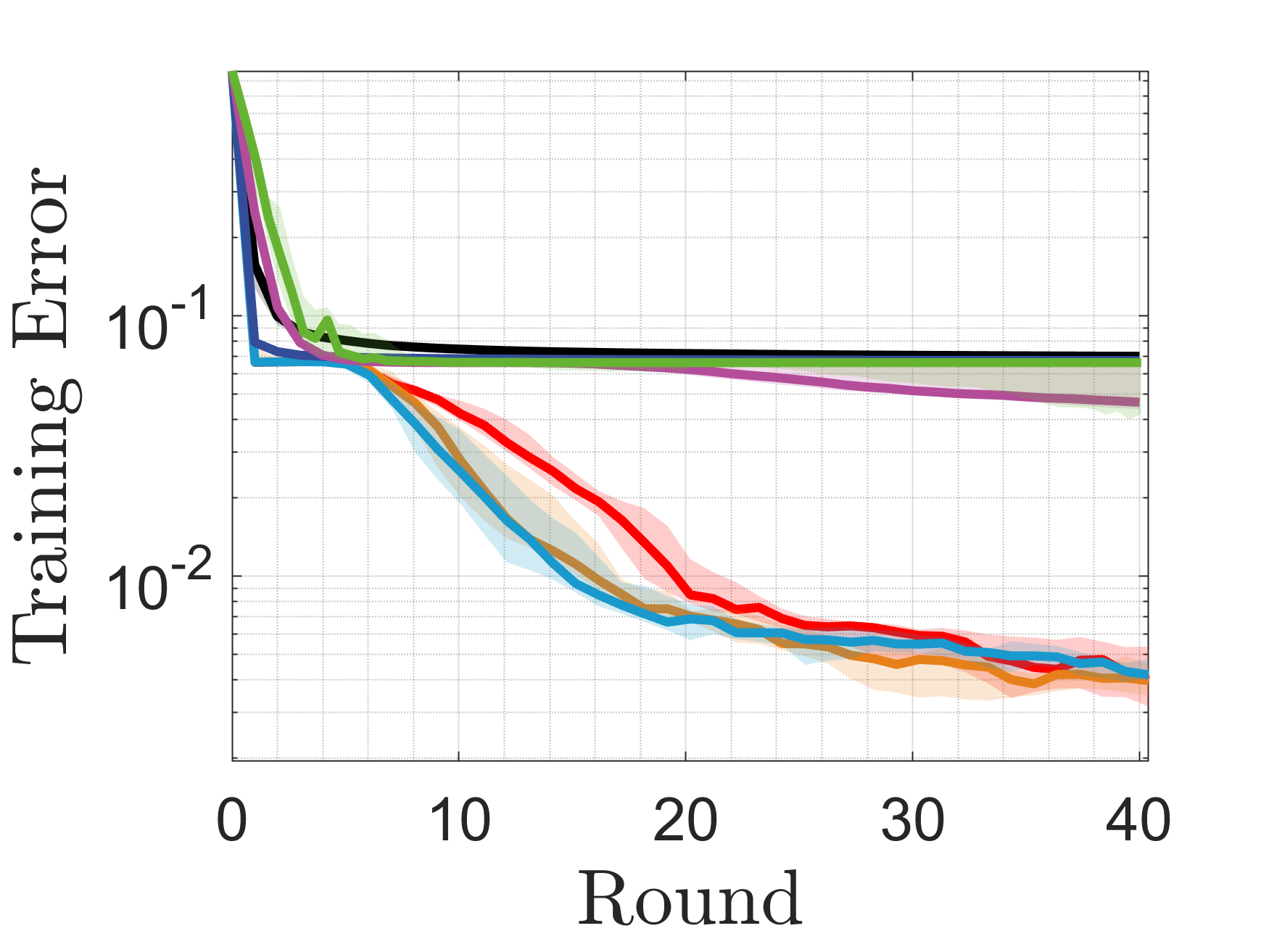}}
\subfloat[LSVM, ijcnn1]{\includegraphics[width=0.33\textwidth]{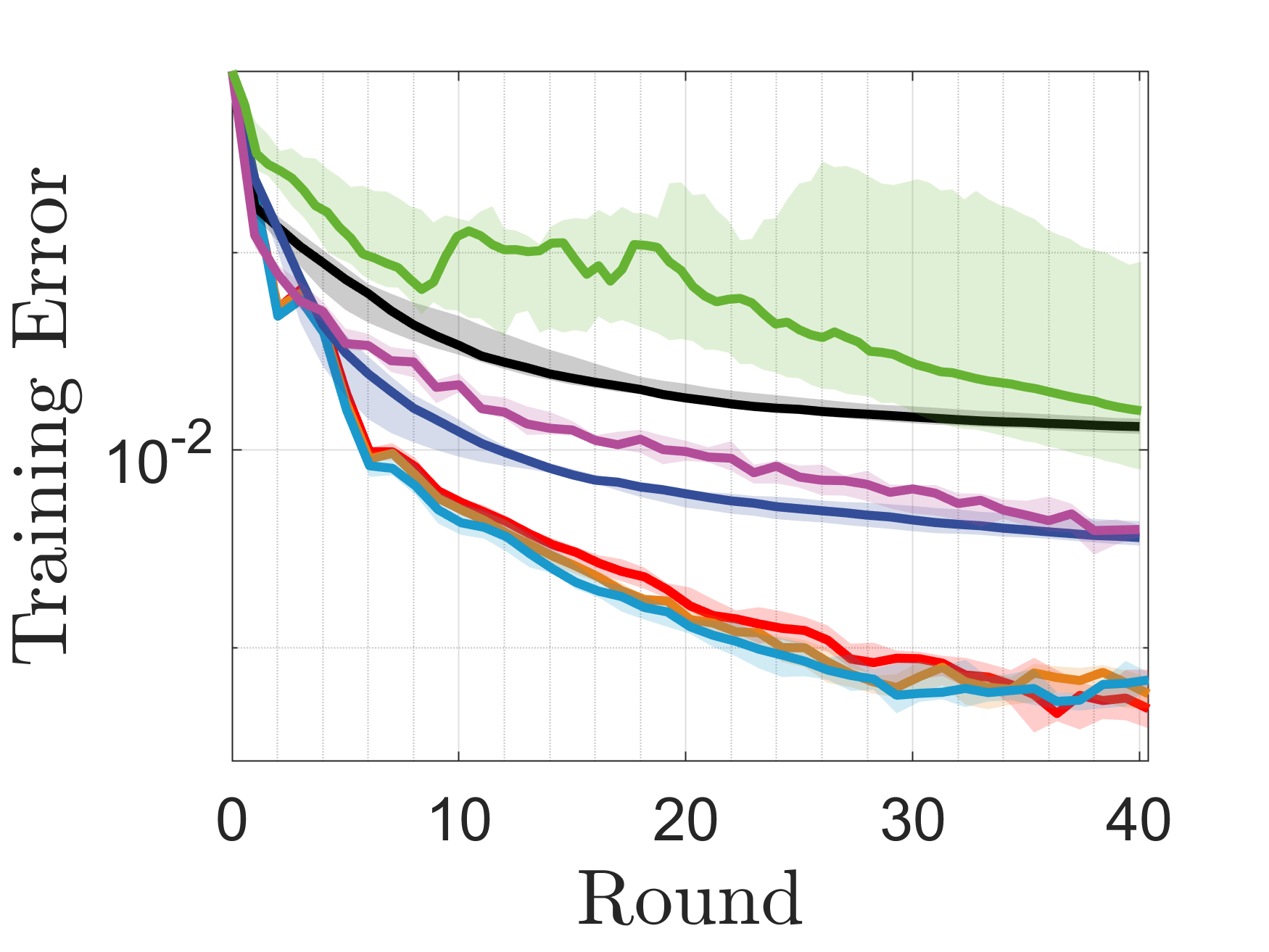}}
\hfil
\subfloat[LR, covtype]{\includegraphics[width=0.33\textwidth]{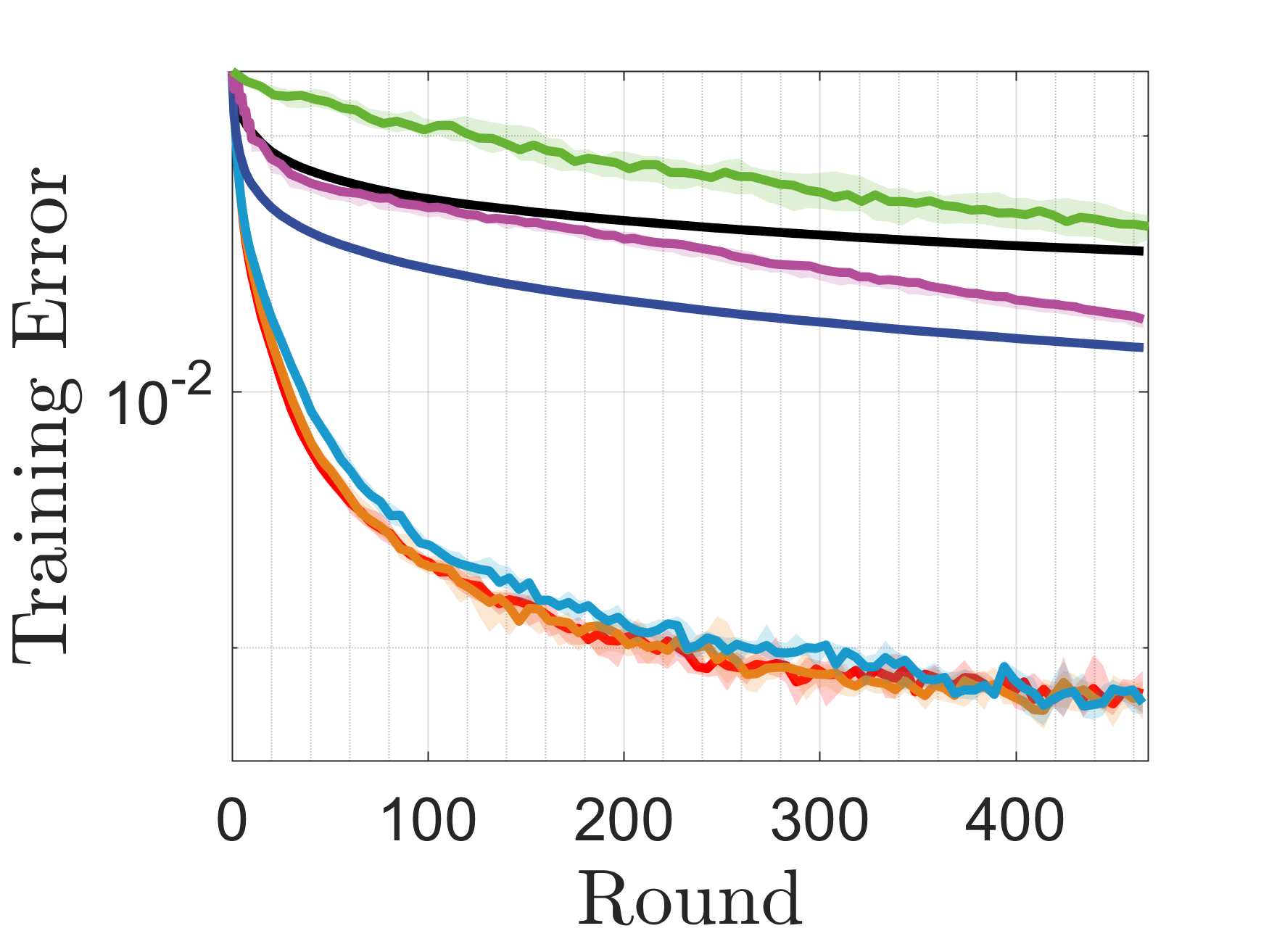}}
\subfloat[NSVM, covtype]{\includegraphics[width=0.33\textwidth]{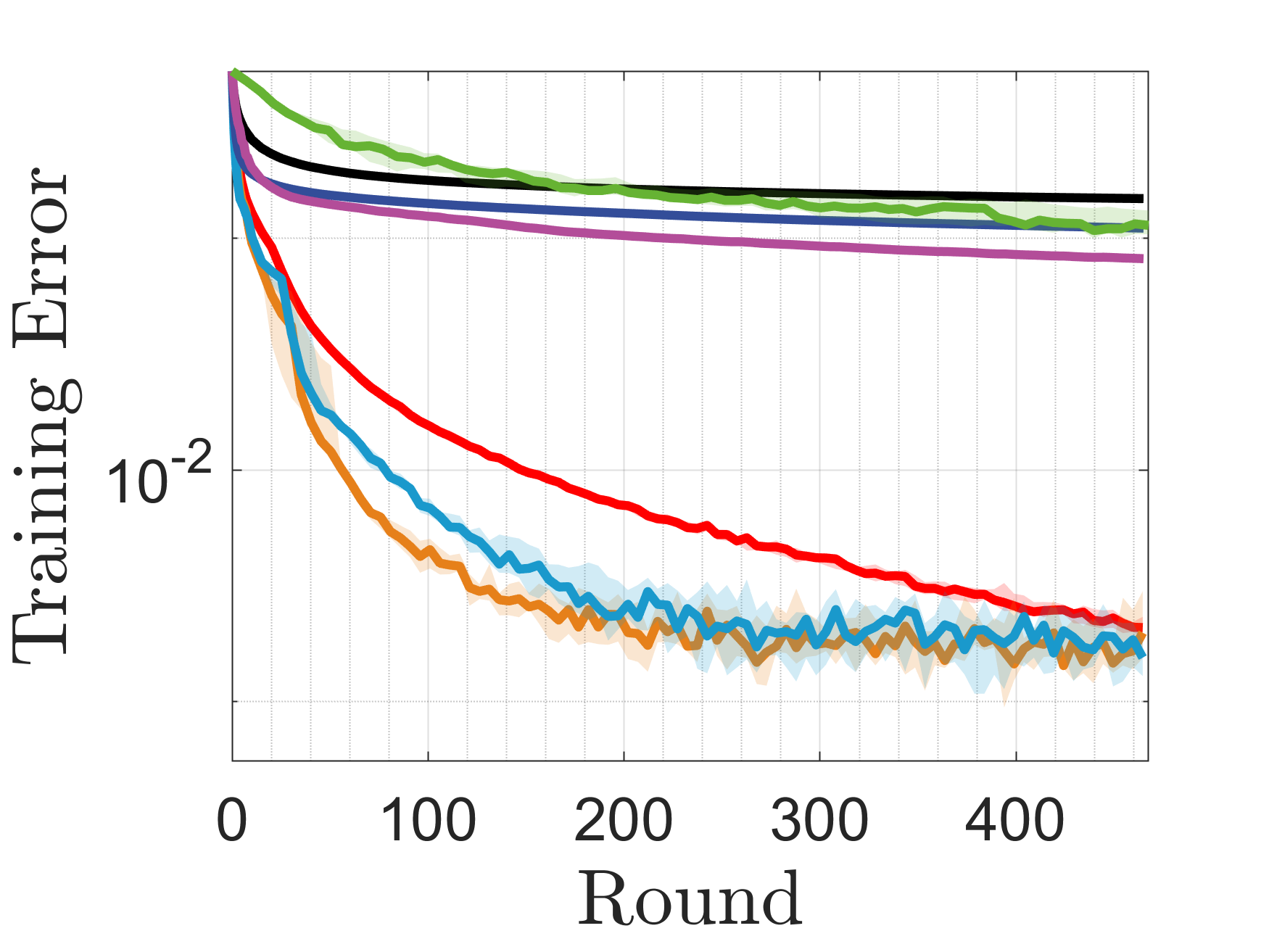}}
\subfloat[LSVM, covtype]{\includegraphics[width=0.33\textwidth]{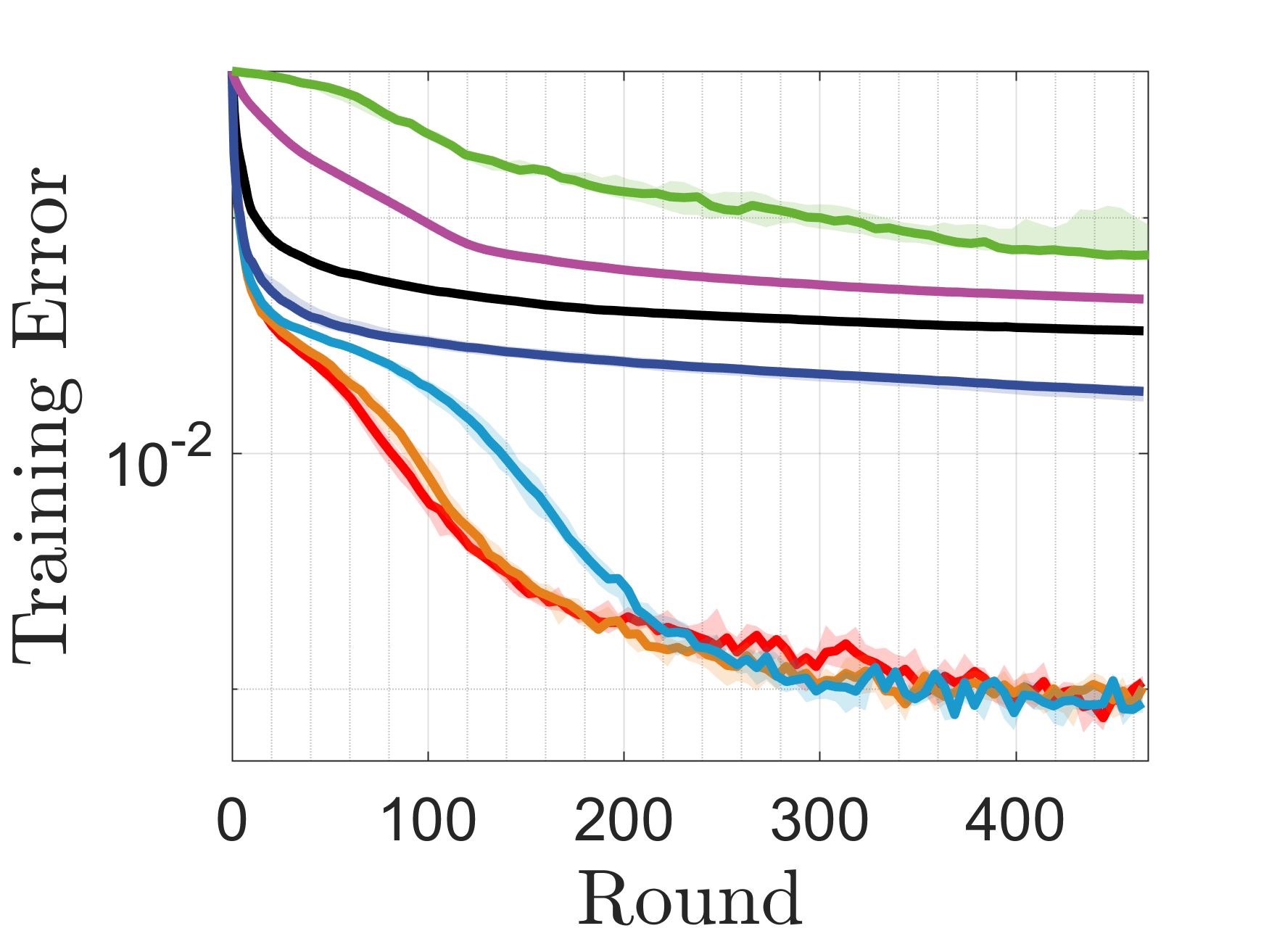}}
\caption{\rrr{Comparison on ijcnn, covtype, and real-sim datasets. The curve displays the training error versus the number of rounds and the corresponding shaded area extends from the 25th to 75th percentiles over the results obtained from all independent runs.}}
\label{fig:training-curve-part-two}
\end{figure*}

We also observe that, when implemented in DES, the standard Gaussian sampling, the mixture Gaussian sampling, and the mixture Rademacher sampling do not show significant difference in performance. Although our analyses in \Cref{theorem:convergence-DES-mixture-Gaussian-l2,theorem:convergence-DES-mixture-Rademacher-l2} suggest the possibility that mixture sampling might degrade the convergence, the results here show that the degradation, if exists, is in general negligible. 
In many cases, in fact, mixture sampling can even improve the performance. This suggests that mixture sampling could be used as the default scheme for DES, given its availability in improving the sampling efficiency.

\rrr{The experimental results obtained on testing sets are reported in the supplement.} In general, the generalization performance of the DES methods are consistent with their training performance. 

\subsection{Adaptation of step-size}
The theoretical analyses have demonstrated that DES converges with any initial step-size; and in this subsection we provide more empirical evidence. We first verify the performance of DES and the other competitors under different initial settings. In order to evaluate their performance over all problems and all datasets, we adopt the performance profile~\cite{dolan_benchmarking_2002}, a classic tool for visual comparison. The profile of an algorithm is the curve of the fraction of its solved test instances\footnote{Test instance denotes the pair of problem and dataset.} (denoted by $\rho(\tau)$) versus the amount of allocated computational budget (denoted by $\tau$). The computational budget is measured by the ratio of the required number of rounds to that required by the best performer. We say an algorithm can solve a test instance if its obtained objective function value $f'$ satisfies $f(\bm{x}_0) - f' > \delta (f(\bm{x}_0) - f_*')$ where $\delta\in (0,1)$ controls the accuracy and $f_*'$ is the best objective value obtained among all algorithms. An algorithm with high values of $\rho(\tau)$ or one that is located at the \rrr{top left} of the figure is preferable. In this section, the objective function value is measured by the training loss.

\Cref{fig:profile_stepsize} plots the performance profiles of DES (with the standard Gaussian sampling) as well as the three competitors, with initial step-sizes chosen from $\{0.1, 1, 10\}$. We choose $\delta=0.1$ in plotting the profiles.
The curves of DES are mostly lie to the left of the others, demonstrating that the relative performance of DES is in general robust to the step-size setting. For small $\tau$, the profile of DES with $\alpha=0.1$ lies to the right of Fed-ZO-SGD with $\alpha=10$, and overlaps with that of the other methods; this indicates that $\alpha=0.1$ is too small for DES to achieve fast decrease in early stage. But when a sufficient amount of computation budget \rrr{(e.g., $\tau \ge 10$) is allowed}, then such a step-size setting can nevertheless lead to the performance comparable to Fed-ZO-SGD with the best tuned step-size. Fed-ZO-GD is not robust to the step-size setting. Its profile for $\alpha=0.1$ is not shown in the plot, implying that with this setting Fed-ZO-GD cannot solve any test instance.

\begin{figure}[tb] 
\centering
\subfloat{\includegraphics[width=0.475\textwidth]{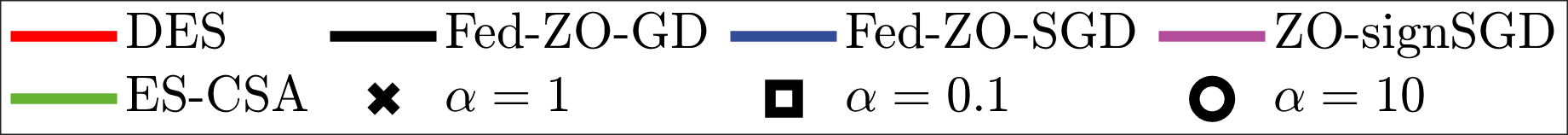}} \\[0.5ex]
\addtocounter{subfigure}{-1}
\includegraphics[width=0.5\textwidth]{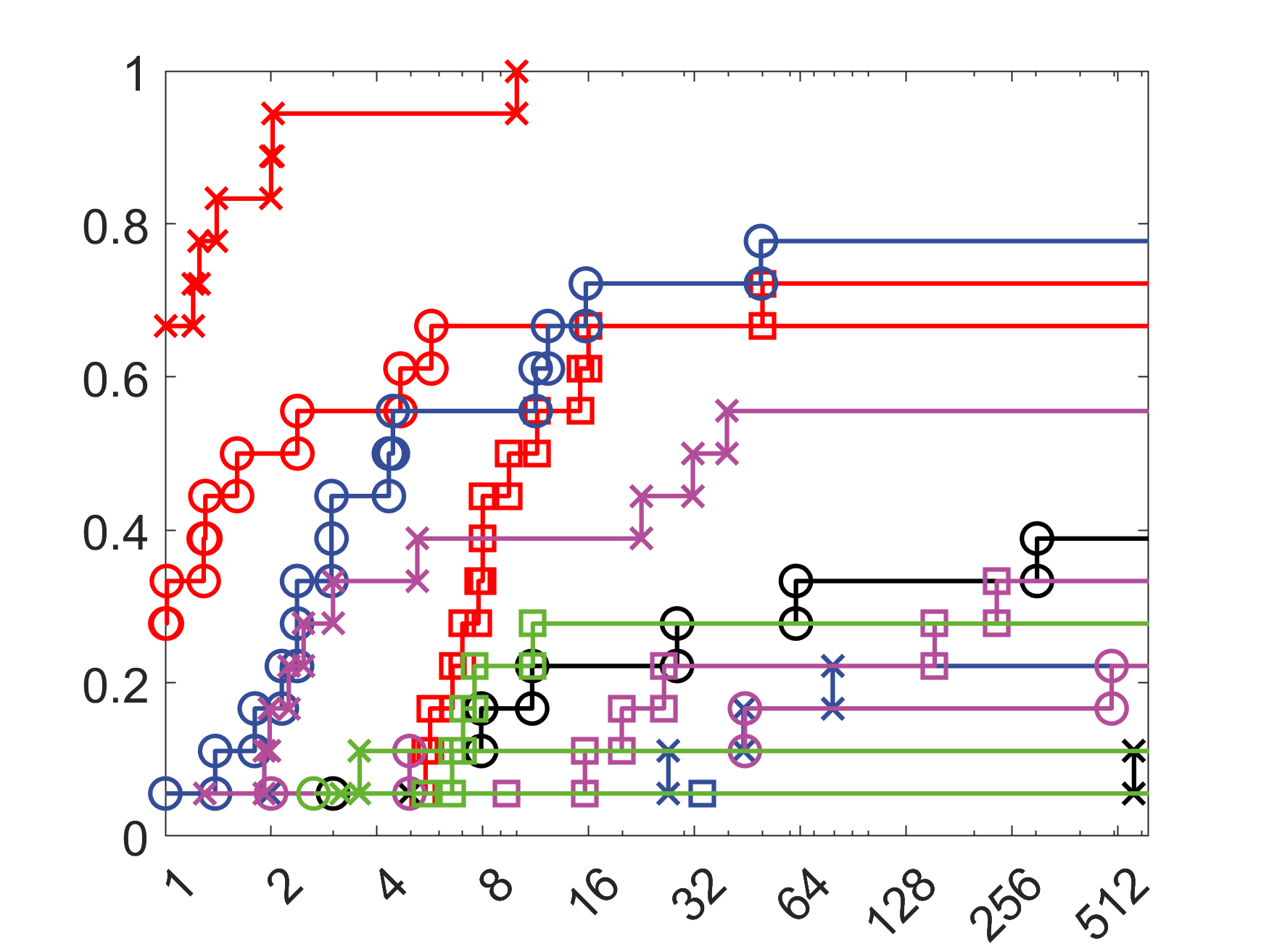}
\caption{Performance profiles ($\delta=0.1$) of different algorithms with different initial step-sizes. Results are obtained on all test instances.}
\label{fig:profile_stepsize}
\end{figure}

\Cref{fig:stepsize_comparison_SUSY} provides, as an representative, the convergence trajectories of the algorithms with different initial step-sizes. In general, on the two convex problems (i.e., LR and LSVM), the performance of DES is quite insensitive to the initial step-size; all three settings admit approaching similar results in the long run. ES-CSA exhibits similar adaptation ability, albeit with relatively poor performance. The other gradient based methods are sensitive to step-size settings, leading to quite different solutions even in the convex problems. On the nonconvex problem NSVM, the initial value of the step-size seems to have a considerable influence on all methods, possibly because of that the step-size setting is critical in escaping local optima. In this case, large initial step-sizes seem to yield faster convergence, but may also lead to early stagnation.

\begin{figure*}[tb] 
\centering
\subfloat{\includegraphics[width=0.90\textwidth]{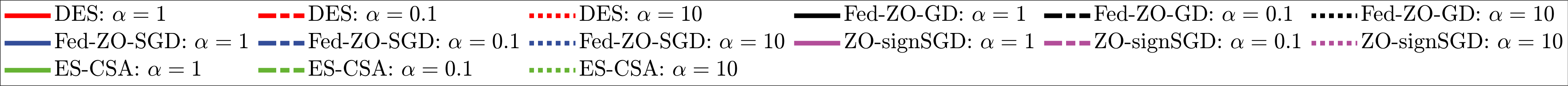}} \\[-2ex]
\addtocounter{subfigure}{-1}
\subfloat[LR]{\includegraphics[width=0.33\textwidth]{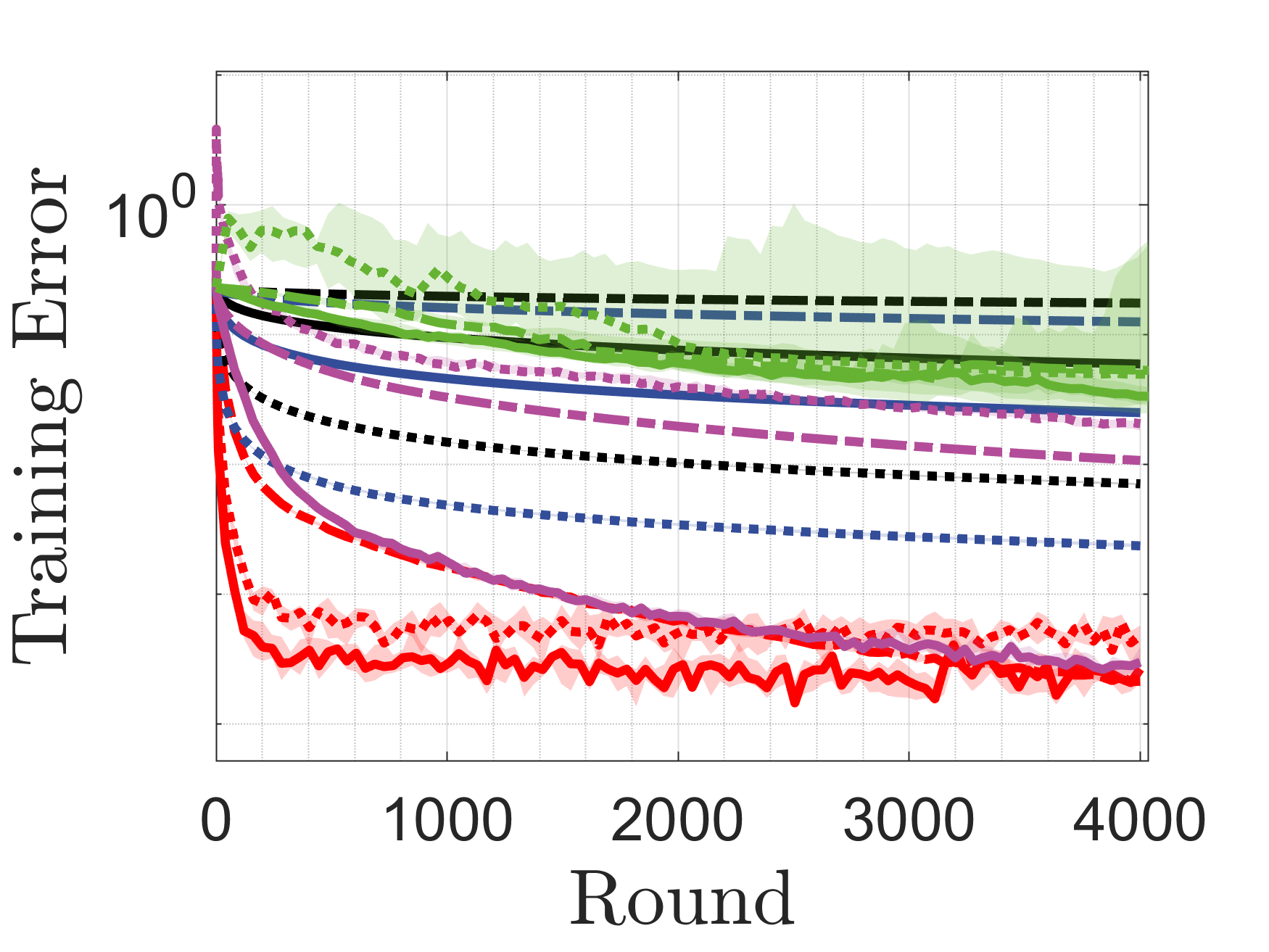}}
\subfloat[NSVM]{\includegraphics[width=0.33\textwidth]{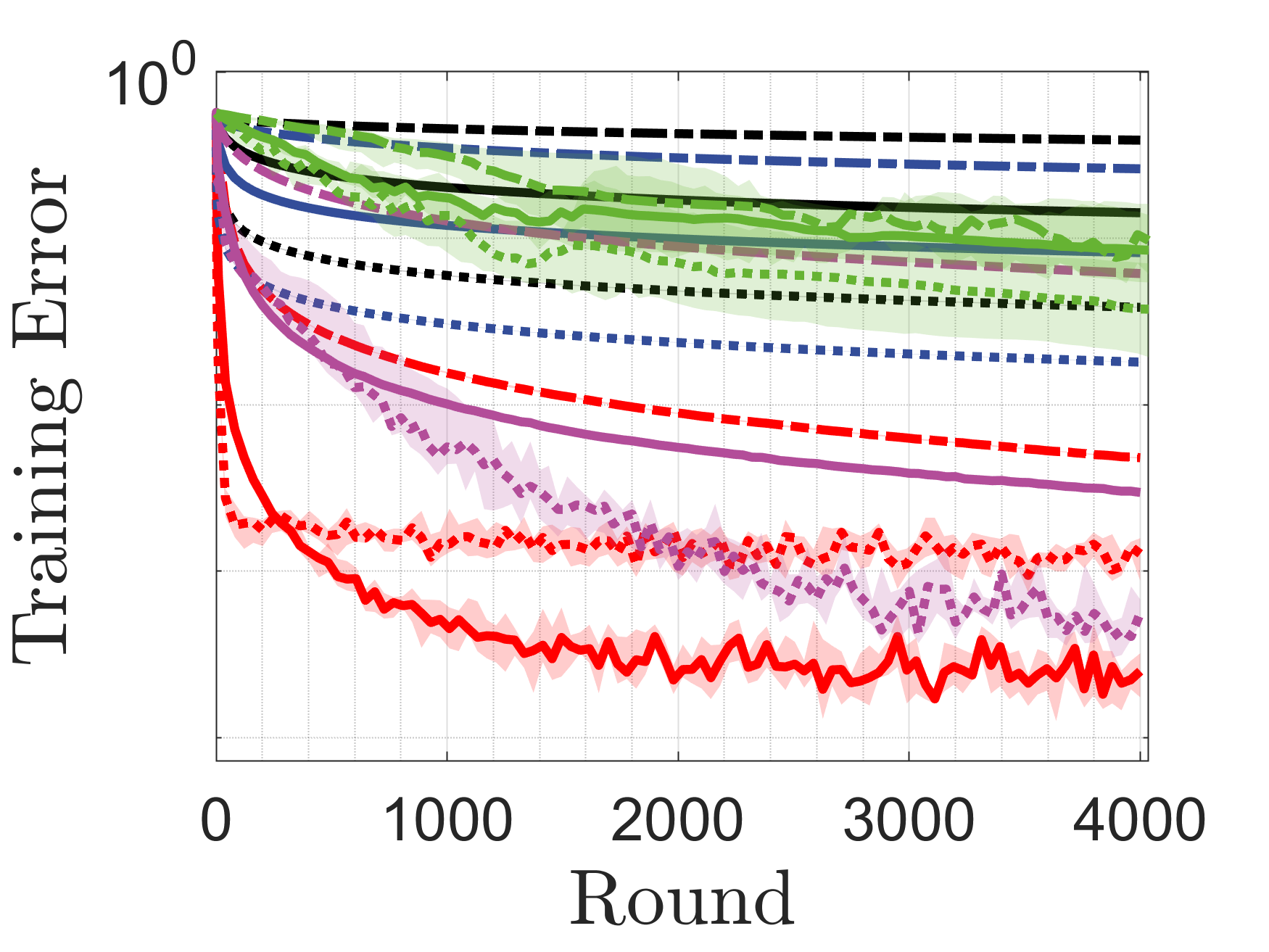}}
\subfloat[LSVM]{\includegraphics[width=0.33\textwidth]{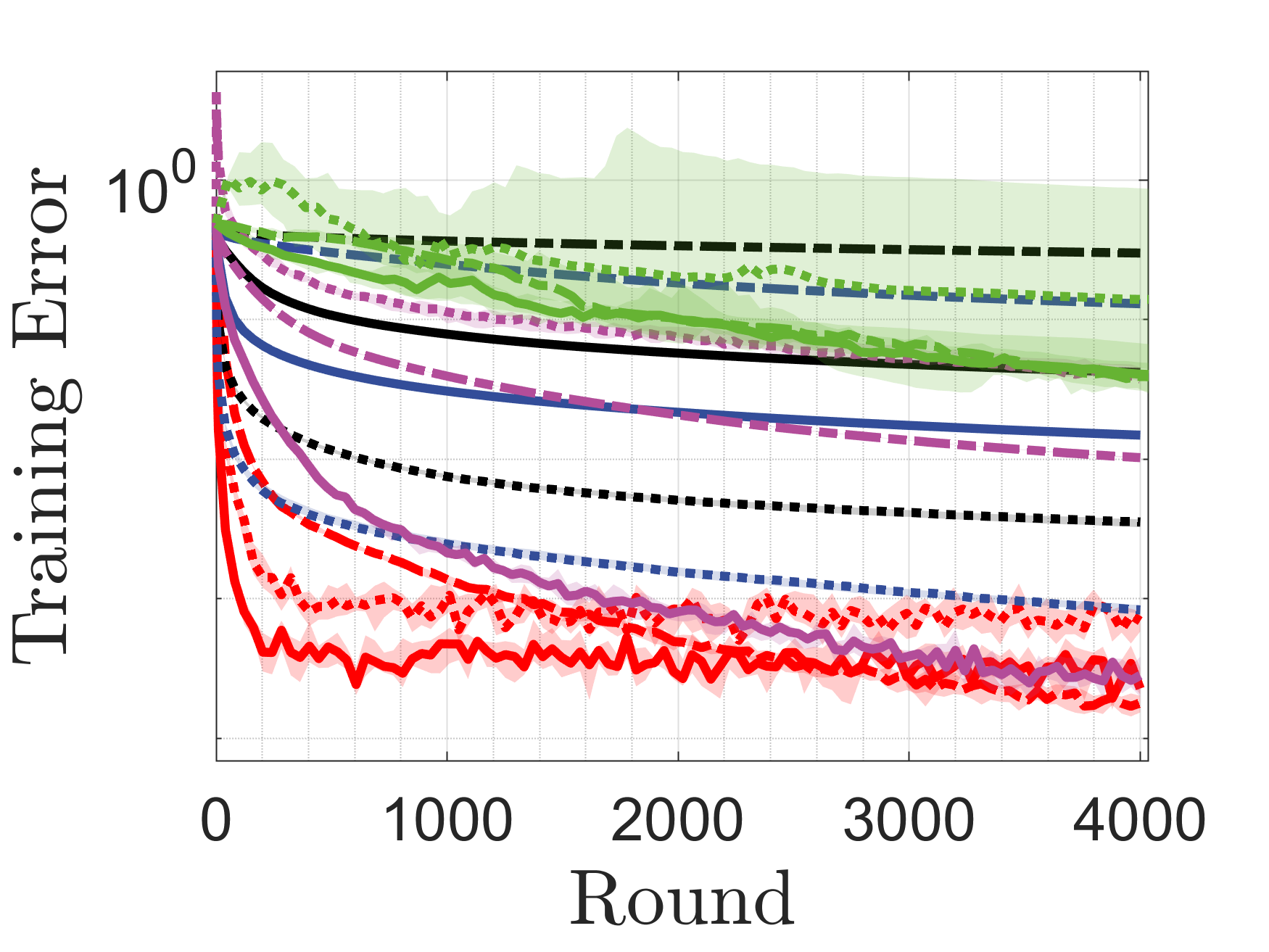}}
\caption{Convergence on SUSY with different initial step-size settings. The curve displays the training error versus the number of rounds and the corresponding shaded area extends from the 25th to 75th percentiles over the results obtained from all independent runs.}
\label{fig:stepsize_comparison_SUSY}
\end{figure*}

\subsection{Impact of momentum}
The convergence rate established previously does not reflects its dependence on the momentum parameter, so here we investigate this empirically. Consider the mixture Rademacher sampling based DES method, with $\beta$ chosen from $\{0, 0.2, 0.4, 0.6, 0.8\}$ and $\alpha$ fixed to 1. All other settings are the same as those in \Cref{ss:overall-performance}. Note that in the theoretical analyses we have required 
\begin{equation} \label{eq:assumption-on-beta}
	\beta \le \sqrt{\frac{1}{2\sqrt{2}}} \lessapprox 0.6
\end{equation}
for technical reasons. So the choice $\beta=0.8$ is to verify whether the above requirement is necessary in practice.

\Cref{fig:profile_momentum} gives the profile plot obtained on all test instances, measured with two different $\delta$ settings. Note that the smaller $\delta$ is, the higher solution-accuracy the curve reflects. It is found that the momentum mechanism becomes useless in the low accuracy domain; as setting $\beta$ to 0 is enough to solve nearly 80\% test instances within a very limited amount of computational budget. In this case, setting $\beta$ to 0.8 is indeed harmful to the performance.
To approach good performance in high accuracy, on the contrary, an appropriate setting of this parameter is generally helpful and could influence the final results. 
Again, we observe that the setting $\beta=0.8$ leads to poor performance, indicating that the assumption \cref{eq:assumption-on-beta} seems to be mandatory.  
But it is worthy nothing that the choice of $\beta$ is not critical to the relative performance of DES compared with the other competitors; we suggest to fix its value in the range $[0.2, 0.6]$ in all situations.

\begin{figure}[tb] 
\centering
\subfloat[Low solution-accuracy case: $\delta=0.05$]{\includegraphics[width=0.33\textwidth]{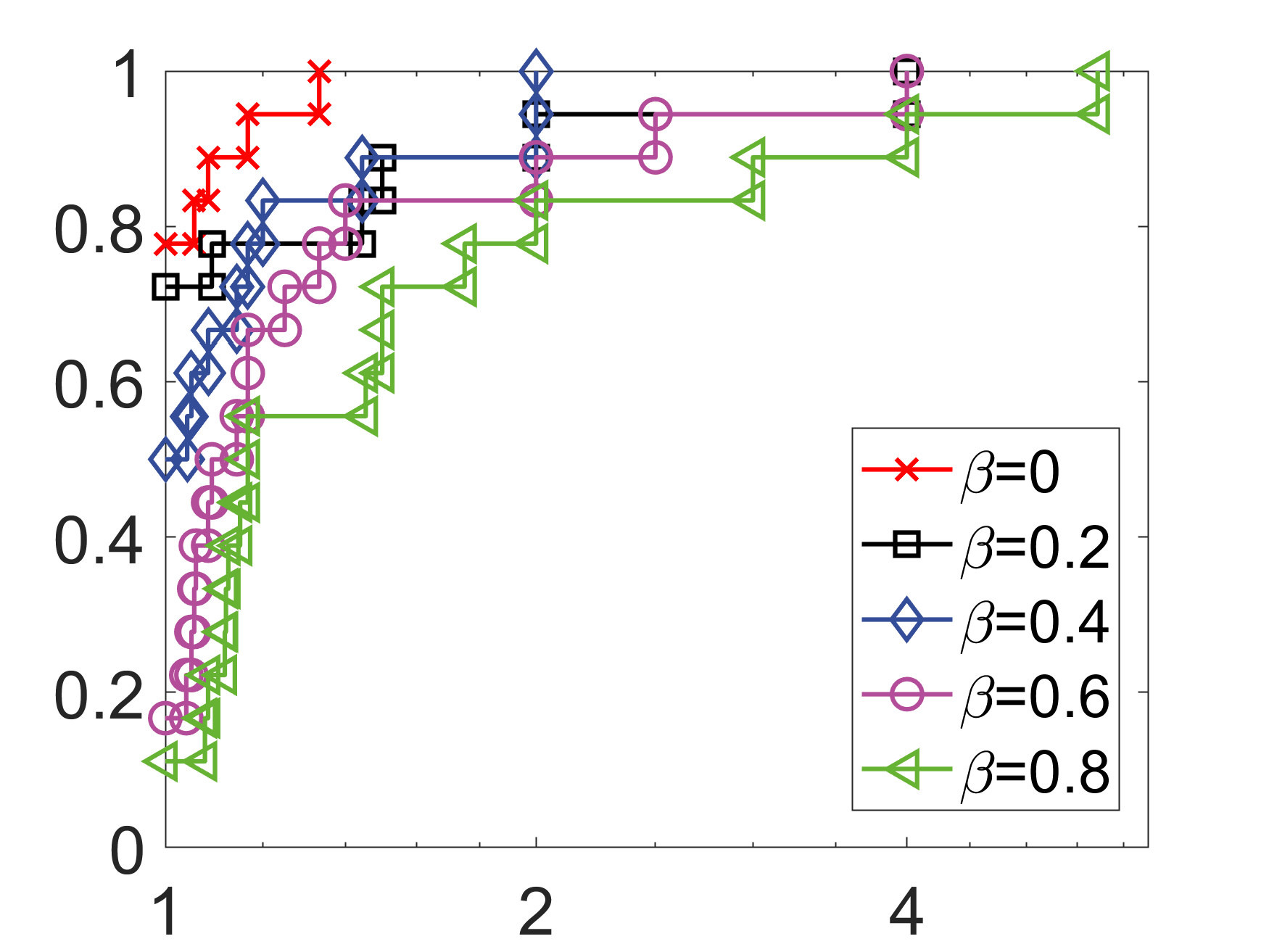}} \\
\subfloat[High solution-accuracy case: $\delta=0.001$]{\includegraphics[width=0.33\textwidth]{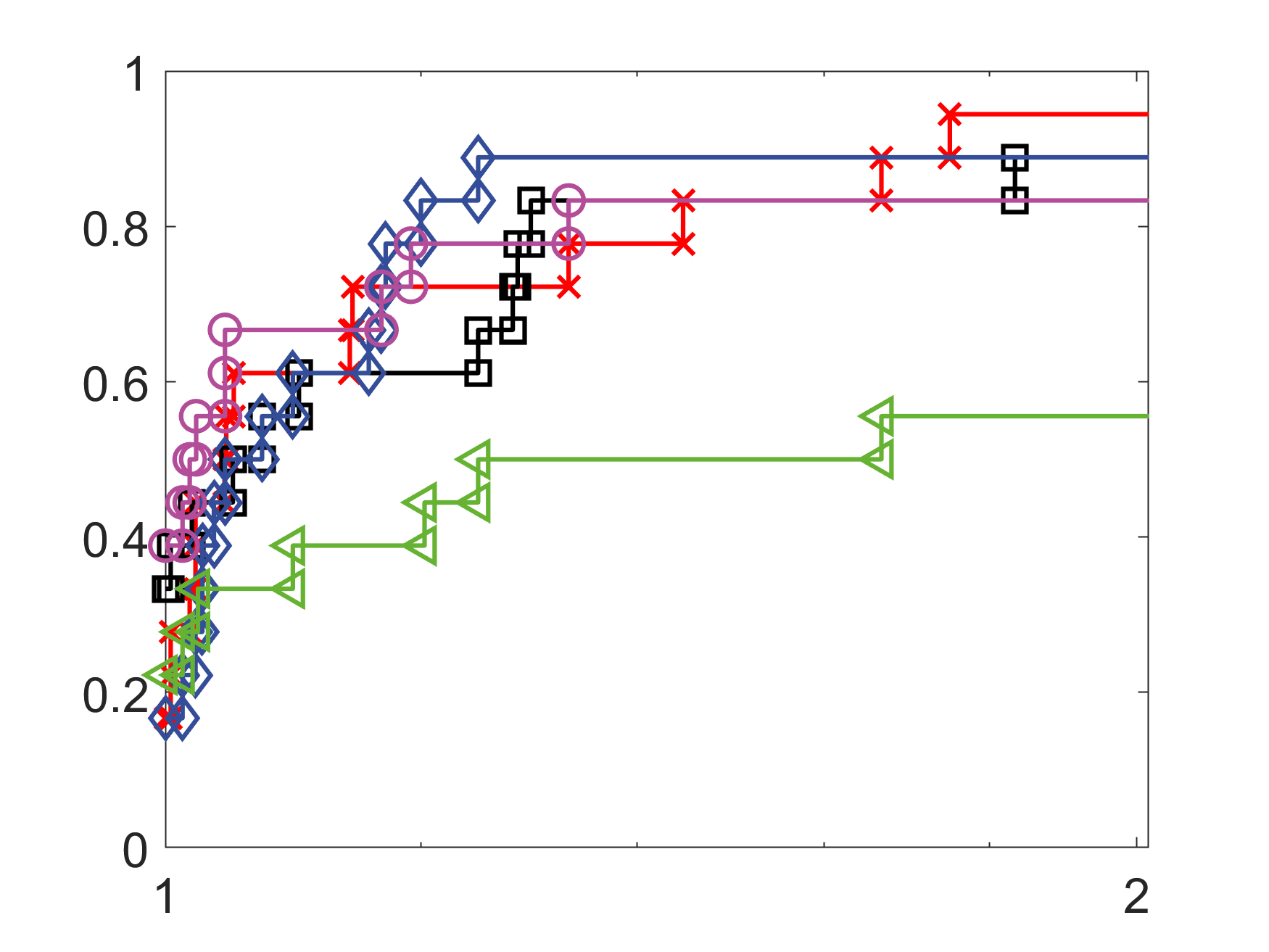}}
\caption{Performance profiles of DES with different momentum parameters. Results are obtained on all test instances. The mixture Rademacher sampling scheme is used in implementing DES.}
\label{fig:profile_momentum}
\end{figure}

\subsection{Impact of minibatch size}
Here we verify the impact of minibatch size on the algorithm performance. We consider the mixture Rademacher sampling based DES method and choose $\beta$ from $\{100, 500, 1000, 1500, 2000\}$. All other settings are the same as in \Cref{ss:overall-performance}. 

\Cref{fig:profile_batchsize} reports the results obtained on all test instances via performance profile. It is clearly that whether minibatch size matters depends on which accuracy we would like to achieve. In the low accuracy case ($\delta=0.05$), choosing a small minibatch $b=100$ can solve at least 50\% test instances very quickly, although suffering early termination later. In this case, using a large minibatch does not lead to significant improvement in performance. Oppositely, the impact of minibatch size becomes quite clear in high accuracy case ($\delta=0.001$) where increasing $b$ consistently improves the number of test instances that can be solved. This observation matches our theoretical analyses and suggests that a large minibatch is generally better if the computational cost is \rrr{affordable at the worker-side}.

\begin{figure}[tb] 
\centering
\subfloat[Low solution-accuracy case: $\delta=0.05$]{\includegraphics[width=0.33\textwidth]{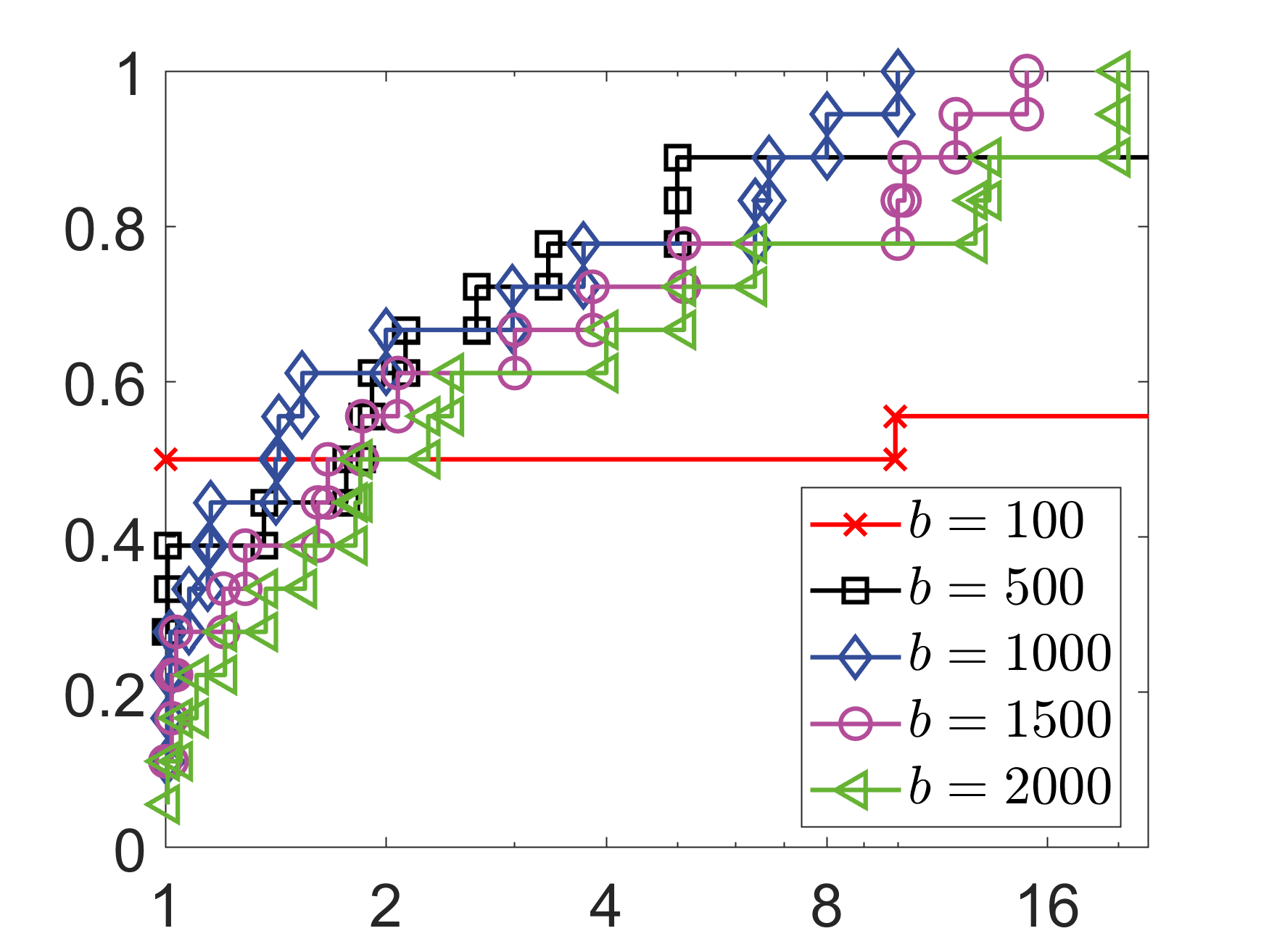}} \\
\subfloat[High solution-accuracy case: $\delta=0.001$]{\includegraphics[width=0.33\textwidth]{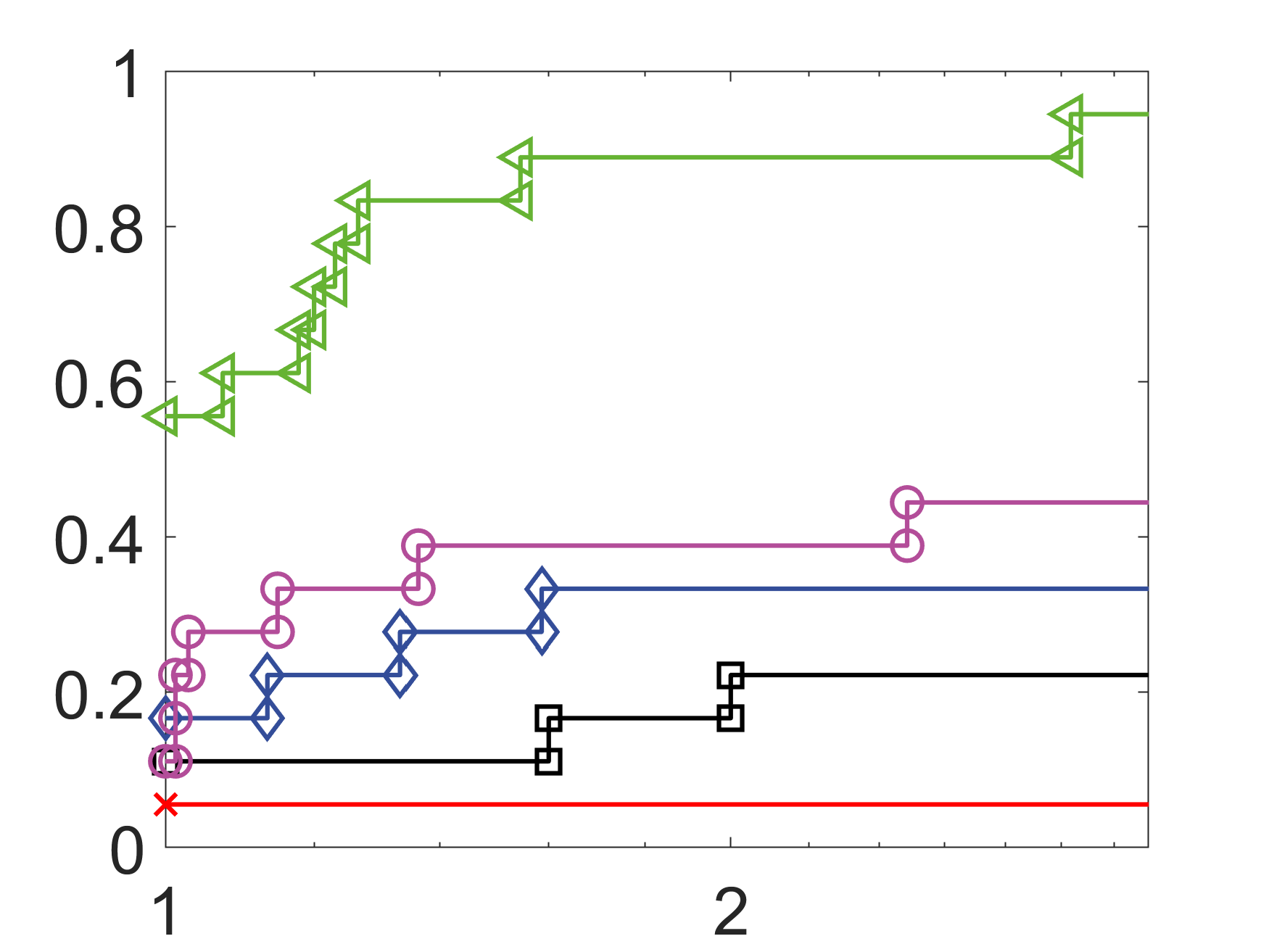}}
\caption{Performance profiles of DES with different minibatch sizes. Results are obtained on all test instances. The mixture Rademacher sampling scheme is used in implementing DES.}
\label{fig:profile_batchsize}
\end{figure}

\section{Conclusion}
In this work we propose the DES method via modifying the classic evolution strategy method and adapting it to the distributed setting. Our method uses a Gaussian probability model to guide the worker's local update, so it avoids finite-difference based smoothing techniques which might cause numerical issues. We have analyzed its convergence properties compared to existing zeroth-order and first-order methods, demonstrating its adaptivity to objective landscapes and the exploitation ability towards sparsity. Two alternative sampling schemes have been suggested and we find they lead to an improvement in sampling efficiency with no obvious degradation in performance.  The current implementation of DES, however, does not support heterogeneous data distribution, which seems to be a common issue for those based on biased descent step; see \cite{liu_signsgd_2019,bernstein_signsgd_2018} for an example. The idea of bias correction suggested in \cite{karimireddy_error_2019} seems to address this issue, and is worth a try in further development of DES. This idea, nevertheless, would be incompatible with the comparison-based nature of the ES family. We would like to continue resolving this in the future. We hope our work on DES will serve as a starting point for generalizing the rich studies in evolutionary computation communities to the distributed world.

\bibliographystyle{IEEEtran}
\bibliography{hxy}

\onecolumn
\appendices
\begin{titlepage}
\begin{center}
{\LARGE\bfseries Supplementary Appendices}
\end{center}
\global\let\newpage\relax
\end{titlepage}

\section{Proof of \Cref{theorem:convergence-simple-ES}}
\begin{proof}
For convenience define the following scalar operations
\begin{equation} \label{eq:definition-sign-signplus}
	\sign(a) = 
	\begin{cases}
		1 & \text{\; if $a \ge0$}	 \\
		-1 & \text{\; if $a < 0$}
	\end{cases}	\;\;\; \text{and}\;\;\;
	\signplus(a) = \frac{\sign(a)+1}{2}=
	\begin{cases}
		1 & \text{\; if $a \ge0$}	 \\
		0 & \text{\; if $a < 0$}
	\end{cases}.
\end{equation}
Note that the $\sign(\cdot)$ is different from the usual operation of taking sign, as in our definition it returns 1 when performed on 0. 
In addition, we have the following useful identities:
\begin{equation} \label{eq:sign_identity}
\sign(a) b = \left(-1 + 2\mathbb{I}\left\{\sign(a) = \sign(b)\right\}\right) |b|
\end{equation}
and
\begin{equation} \label{eq:indicator-neq} 
\mathbb{I}\left\{\sign(a) = \sign(b)\right\} = \mathbb{I}\left\{|a+b|\ge|b|\right\}
\end{equation}
which can be verified easily. 

With the $\sign$ operation defined in (\ref{eq:definition-sign-signplus}), the iterations generated by \Cref{alg:simple-ES} can be rewritten as
\begin{equation} \label{eq:update-rule-simple-ES}
	\bm{x}_{k+1} = \bm{x}_{k} + \alpha_{k} \signplus\left(f\left(\bm{x}_k\right) - f\left(\bm{x}_k + \alpha_k \bm{u}_k\right)\right)\bm{u}_k.
\end{equation}

With \Cref{assumption:smoothness}, we can bound the per-iteration progress as
\begin{equation*} 
\begin{split}
f\left(\bm{x}_{k+1}\right) - f\left(\bm{x}_k\right) 
& \le \nabla f\left(\bm{x}_k\right)^T (\bm{x}_{k+1} - \bm{x}_k) + \frac{L}{2}\left\|\bm{x}_{k+1}-\bm{x}_k\right\|^2 \\
& \overset{(\ref{eq:update-rule-simple-ES})}\le \alpha_k \signplus\left(f\left(\bm{x}_k\right) - f\left(\bm{x}_k + \alpha_k \bm{u}_k\right)\right)\nabla f\left(\bm{x}_k\right)^T\bm{u}_k + \frac{L\alpha_k^2}{2} \left\|\bm{u}_k\right\|^2 \\
& \overset{(\ref{eq:definition-sign-signplus})}= \frac{1}{2}\alpha_k \bm{u}_k + \frac{1}{2}\alpha_k \underbrace{\left(\sign\left(f\left(\bm{x}_k\right) - f\left(\bm{x}_k + \alpha_k \bm{u}_k\right)\right)\right)\nabla f\left(\bm{x}_k\right)^T\bm{u}_k}_{\defeq \mathfrak{A}} + \frac{L\alpha_k^2}{2} \left\|\bm{u}_k\right\|^2.
\end{split}
\end{equation*}
Taking expectation with respect to $\bm{u}_k$ at both sides, and according to \Cref{lemma:bound-on-variance}, we have
\begin{equation} \label{eq:expected-progress}
\mathbb{E}_k \left[f\left(\bm{x}_{k+1}\right)\right] - f\left(\bm{x}_k\right) 
\le \frac{1}{2}\alpha_k \mathbb{E}_k \left[\mathfrak{A}\right] + \frac{L\alpha_k^2}{2} U
\end{equation}
where $\mathbb{E}_k$ denotes the expectation conditioned on the randomness at the $k$-th iteration. 

We now bound the term $\mathfrak{A}$ using identities (\ref{eq:sign_identity}) and (\ref{eq:indicator-neq}):
\begin{equation} \label{eq:bound-on-decrease}
\begin{split}
\mathfrak{A} 
& \overset{(\ref{eq:sign_identity})}= \left(-1 + 2\mathbb{I}\left\{\sign\left(f\left(\bm{x}_k\right) - f\left(\bm{x}_k + \alpha_k \bm{u}_k\right)\right) = \sign\left(\nabla f\left(\bm{x}_k\right)^T\bm{u}_k\right)\right\}\right) \left|\nabla f\left(\bm{x}_k\right)^T\bm{u}_k\right|\\
& = \left(-1 + 2\mathbb{I}\left\{\sign\left(f\left(\bm{x}_k\right)-f\left(\bm{x}_k + \alpha_k \bm{u}_k\right)\right) = \sign\left(\alpha_k \nabla f\left(\bm{x}_k\right)^T\bm{u}_k\right)\right\}\right) \left|\nabla f\left(\bm{x}_k\right)^T\bm{u}_k\right|\\
& \overset{(\ref{eq:indicator-neq})}= \left(-1 + 2\mathbb{I}\left\{\left| f\left(\bm{x}_k + \alpha_k \bm{u}_k\right) -f\left(\bm{x}_k\right)- \alpha_k \nabla f\left(\bm{x}_k\right)^T\bm{u}_k \right| \ge \alpha_k \left| \nabla f\left(\bm{x}_k\right)^T\bm{u}_k\right|\right\}\right) \left|\nabla f\left(\bm{x}_k\right)^T\bm{u}_k\right| \\
& \le \left(-1 + 2\mathbb{I}\left\{\frac{L}{2} \left\|\alpha_k \bm{u}_k\right\|^2 \ge \alpha_k \left| \nabla f\left(\bm{x}_k\right)^T\bm{u}_k\right|\right\}\right) \left|\nabla f\left(\bm{x}_k\right)^T\bm{u}_k\right| 
\end{split}
\end{equation}
where the last inequality is due to \Cref{assumption:smoothness}.

Substituting (\ref{eq:bound-on-decrease}) into (\ref{eq:expected-progress}) gives
\begin{equation} \label{eq:expected-progress-2}
\begin{split}
\mathbb{E}_k \left[f\left(\bm{x}_{k+1}\right)\right] & - f\left(\bm{x}_k\right) \\
& \le \frac{\alpha_k}{2} \mathbb{E}_k \left[\left(-1 + 2\mathbb{I}\left\{\frac{\alpha_k L}{2} \left\| \bm{u}_k\right\|^2 \ge \left| \nabla f\left(\bm{x}_k\right)^T\bm{u}_k\right|\right\}\right) \left|\nabla f\left(\bm{x}_k\right)^T\bm{u}_k\right|\right] + \frac{L\alpha_k^2}{2} U	\\
& = -\frac{\alpha_k}{2} \mathbb{E}_k \left[\left|\nabla f\left(\bm{x}_k\right)^T\bm{u}_k\right|\right]
	+\alpha_k \underbrace{\mathbb{E}_k \left[\mathbb{I}\left\{\frac{\alpha_k L}{2} \left\| \bm{u}_k\right\|^2 \ge \left| \nabla f\left(\bm{x}_k\right)^T\bm{u}_k\right|\right\} \left|\nabla f\left(\bm{x}_k\right)^T\bm{u}_k\right|\right]}_{\defeq \mathfrak{B}} + \frac{L\alpha_k^2}{2} U \\	
& = -\frac{\alpha_k}{\sqrt{2\pi}} \left\|\nabla f\left(\bm{x}_k\right)\right\|_2
	+\alpha_k \mathfrak{B} + \frac{L\alpha_k^2}{2} U \\	
\end{split}
\end{equation}
where the last equality uses the fact
\begin{equation} \label{eq:expectation-half-gaussian}
\mathbb{E}[|\bm{y}^T\bm{u}|] = \sqrt{\frac{2}{\pi}} \|\bm{y}\|_2  \text{ for }\bm{u} \sim \mathcal{N}(\bm{0},\bm{I}).
\end{equation}

Since the distribution of $\bm{u}_k$ is isotropic, we can assume $\nabla f(\bm{x}_k) = \left\|\nabla f(\bm{x}_k)\right\|_2 \bm{e}_1$ where $\bm{e}_1 = (1,0,\cdots,0)^T$. Denoting $u_{k,i}$ as the $i$-th element of $\bm{u}_k$ and noting the assumption $\|\cdot\| = \|\cdot\|_2$, we have
\begin{equation}
\mathfrak{B} 
= \mathbb{E}_k \left[\mathbb{I}\left\{\frac{\alpha_k L}{2} \sum_{i=1}^n u_{k,i}^2 \ge \left\|\nabla f\left(\bm{x}_k\right)\right\|_2 \left|u_{k,1}\right|\right\} \left\|\nabla f\left(\bm{x}_k\right)\right\|_2 \left|u_{k,1}\right|\right].
\end{equation}

Now we decompose the expectation operation $\mathbb{E}_k$ into two steps: firstly taking the expectation over $u_{k,2},\cdots,u_{k,n}$ and secondly over $u_{k,1}$. That is,
\begin{equation*}
\begin{split}
\mathfrak{B} 
& = \mathbb{E}_{u_{k,1}}\mathbb{E}_{u_{k,2},\cdots,u_{k,n}} \left[\mathbb{I}\left\{\frac{\alpha_k L}{2} \sum_{i=1}^n u_{k,i}^2 \ge \left\|\nabla f\left(\bm{x}_k\right)\right\|_2 \left|u_{k,1}\right|\right\} \left\|\nabla f\left(\bm{x}_k\right)\right\|_2 \left|u_{k,1}\right|\right] \\
& = \mathbb{E}_{u_{k,1}} \left[\mathbb{P}_{u_{k,2},\cdots,u_{k,n}}\left\{\frac{\alpha_k L}{2} \sum_{i=1}^n u_{k,i}^2 \ge \left\|\nabla f\left(\bm{x}_k\right)\right\|_2 \left|u_{k,1}\right|\right\} \left\|\nabla f\left(\bm{x}_k\right)\right\|_2 \left|u_{k,1}\right|\right] \\
& \le \mathbb{E}_{u_{k,1}} \left[\frac{\alpha_k L}{2} \frac{u_{k,1}^2 +\sum_{i=2}^n \mathbb{E}_{u_{k,i}} \left[u_{k,i}^2\right]}{\left\|\nabla f\left(\bm{x}_k\right)\right\|_2 \left|u_{k,1}\right|} \left\|\nabla f\left(\bm{x}_k\right)\right\|_2 \left|u_{k,1}\right|\right] \\
& = \frac{\alpha_k L}{2}\mathbb{E}_{u_{k,1}} \left[ {u_{k,1}^2 +\sum_{i=2}^n \mathbb{E}_{u_{k,i}} \left[u_{k,i}^2\right]} \right]
= \frac{\alpha_k L}{2}\mathbb{E}_k \left[\left\|\bm{u}_k\right\|^2\right]. 
\end{split}
\end{equation*}
Here we use the Markov inequality applied on the components $u_{k,2},\cdots,u_{k,n}$.

Substituting the above bound into (\ref{eq:expected-progress-2}) and using \Cref{lemma:bound-on-variance}, we get
\begin{equation*}
\mathbb{E}_k \left[f\left(\bm{x}_{k+1}\right)\right] - f\left(\bm{x}_k\right)
\le -\frac{\alpha_k}{\sqrt{2\pi}} \left\|\nabla f\left(\bm{x}_k\right)\right\|_2 + L\alpha_k^2 U.
\end{equation*}
Taking the total expectation and summing over $k=0,1,\cdots,K-1$ give
\begin{equation} \label{eq:sum-of-gradients}
\sum_{k=0}^{K-1} \alpha_k \mathbb{E}\left[\left\|\nabla f\left(\bm{x}_k\right)\right\|_2\right]
\le \sqrt{2\pi} \left(f\left(\bm{x}_0\right) -f_* + L U \sum_{k=0}^{K-1}\alpha_k^2\right)
\overset{\cref{eq:bound-sum-1-series}}\le \sqrt{2\pi} \left( f\left(\bm{x}_0\right)-f_* + L U \alpha_0^2 (1+\log K)\right).
\end{equation}
On the other hand, we can lower bound the left-hand side as
\[
\sum_{k=0}^{K-1} \alpha_k \mathbb{E}\left[\left\|\nabla f\left(\bm{x}_k\right)\right\|_2\right]	
\overset{\cref{eq:bound-sum-0.5-series-2}}\ge 
\rrr{\sqrt{K}\alpha_0\left(\frac{1}{K} \sum_{k=0}^K 
\mathbb{E}\left[\left\|\nabla f\left(\bm{x}_k\right)\right\|_2\right]\right)}.
\]
Combing this with \cref{eq:sum-of-gradients} yields
\[
\frac{1}{K}\sum_{k=0}^{K-1} \mathbb{E}\left[\left\|\nabla f\left(\bm{x}_k\right)\right\|_2\right]	
\le \sqrt{\frac{2\pi}{K}} \left( \frac{f\left(\bm{x}_0\right)-f_*}{\alpha_0} + L U \alpha_0 (1+\log K)\right).
\]


The bound (\ref{eq:convergene-rate-simple-ES-l2-norm}) can be obtained via specifying $U=n$ according to \Cref{lemma:bound-on-variance}.
\end{proof}

\section{A Unified Implementation of DES and Fundamental Lemmas}
\rrr{
Before proving the main results \Cref{theorem:convergence-DES-l2,theorem:convergence-DES-l1,theorem:convergence-DES-mixture-Gaussian-l2,theorem:convergence-DES-mixture-Rademacher-l2}, we provide in this section some lemmas which will be used several times in the subsequent proofs.
Since we have two DES implementations (i.e., \Cref{alg:DES,alg:DES-mixture-sampling}) and they only differ in the way of generating mutation vectors, we suggest to analyze them in a unified manner. To this end, we provide in \Cref{alg:UDES} a unified implementation of DES which can recover both \Cref{alg:DES} and \Cref{alg:DES-mixture-sampling}. For example, it recovers \Cref{alg:DES} if the mutation vector $\bm{u}_{i,k}^t$ in Line 9 is drawn from the Gaussian distribution $\mathcal{N}(\bm{0},\bm{I})$. It is also logically equivalent to \Cref{alg:DES-mixture-sampling} when $\bm{u}_{i,k}^t$ is drawn from the mixture Gaussian distribution $\mathcal{M}_l^G$ or mixture Rademacher distribution $\mathcal{M}_l^R$.
Note that the lemmas derived in this section do not rely on the detailed distribution for the mutation vectors. We will also not specify the vector norm when using the assumptions. The only requirement is that the variance of the mutation vector $\bm{u}_{i,k}^t$ needs to be bounded by some constant $U$ (see Line 9 in \Cref{alg:UDES}). We will show in the next sections that this requirement indeed holds.} 


\begin{figure}[thb]
\begin{algorithm}[H]
	\caption{\rrr{Unified implementation of DES for convergence analyses}}
	\small
	\label{alg:UDES}
	\begin{algorithmic}[1]
	\Require $\bm{x}_0 \in \mathbb{R}^n$: initial solution; $\alpha \in \mathbb{R}_+$: initial step-size; $\beta \in \left[0,\sqrt{\frac{1}{2\sqrt{2}}}\right)$: momentum parameter; $b \ge \sqrt{T}$: minibatch size; $l \in \mathbb{Z}_+$: mixture parameter
	\For {$t = 0, 1, \cdots, T-1$}
		\For {$i = 1,2,\cdots,M$ \textbf{in parallel}} 
			\State $\bm{v}_{i,0}^t = \bm{x}_t$
			\State $\alpha_0^t = \alpha/(t+1)^{0.25}$
			\State Draw a minibatch $\mathcal{D}_i$ of size $b$ 
			\State Define $f_i(\bm{x}) = \frac{1}{b}\sum_{\bm{\xi} \in \mathcal{D}_i} F(\bm{x};\bm{\xi})$
			\For {$k = 0,1,\cdots,K-1$}
				\State $\alpha_k^t = \alpha_0^t/(k+1)^{0.5}$
				\State Generate a random vector $\bm{u}_{i,k}^t$ satisfying $\mathbb{E}\left[\|\bm{u}_{i,k}^t\|^2\right] \le U$ for some positive constant $U$ and some generic norm $\|\cdot\|$ 
				\State $\bm{v}_{i,k+1}^t = \bm{v}_{i,k}^t + \alpha_k^t \sign_+ \left(f_i(\bm{v}_{i,k}^t) - f_i\left(\bm{v}_{i,k}^t + \alpha_k^t \bm{u}_{i,k}^t\right)\right)$
				where $\sign_+$ is defined in \Cref{eq:definition-sign-signplus}
			\EndFor
		\EndFor
		\State $\bm{d}_{t+1} = \frac{1}{M}\sum_{i=1}^M \bm{v}_{i,K}^t - \bm{x}_t$
		\State $\bm{m}_{t+1} = \beta \bm{m}_t + (1-\beta) \bm{d}_{t+1}$
		\State $\bm{x}_{t+1} = \bm{x}_t + \bm{m}_{t+1}$
	\EndFor
 \end{algorithmic} 
 \end{algorithm}
\end{figure}

\rrr{In the following we give some lemmas regarding the iterations generated from \Cref{alg:UDES}.}
Due to the momentum mechanism, it is difficult to directly work with the solutions $\left\{\bm{x}_t\right\}$. 
Instead, we introduce a virtual sequence $\left\{\bm{z}_t\right\}$ which can be regarded as a counterpart of $\left\{\bm{x}_t\right\}$ without momentum:
\[
	\bm{z}_{t+1} = \frac{1}{1-\beta} \bm{x}_{t+1} - \frac{\beta}{1-\beta} \bm{x}_{t}. 
\]
To make it well-defined, we specify $\bm{x}_{-1} = \bm{x}_0$ such that $\bm{z}_0 = \bm{x}_0$. We will characterize the algorithm behavior with $\left\{\bm{z}_t\right\}$ and relate it to $\left\{\bm{x}_t\right\}$ in the last step.
Note that by this definition and according to the momentum rule (Lines 14-15 in \Cref{alg:UDES}) we have
\begin{equation} \label{eq:properties-virtual-sequence}
	\bm{z}_{t+1} - \bm{z}_t = \bm{d}_{t+1} 
	\text{\;\;\; and \;\;\; } 
	\left\|\bm{x}_t - \bm{z}_t\right\| = \frac{\beta}{1-\beta}\left\|\bm{x}_t - \bm{x}_{t-1}\right\|.
\end{equation}

\begin{lemma} \label{lemma:bound-on-descent-step}
\rrr{The descent step $\bm{d}_{t+1}$ in \Cref{alg:UDES} can be bounded as}
\begin{align}
\label{eq:descent-step-bound-square}\mathbb{E}\left[\left\|\bm{d}_{t+1}\right\|^2\right] & \le \left(\alpha_0^t\right)^2 UK \left(1+\log K\right), \\
\label{eq:descent-step-bound}\mathbb{E}\left[\left\|\bm{d}_{t+1}\right\|\right] & \le 2 \alpha_0^t \sqrt{KU}.
\end{align}
\end{lemma}

\begin{proof}
\rrr{According to Line 13 of \Cref{alg:UDES}} we have
\begin{equation*}
\begin{split}
	\mathbb{E}\left[\left\|\bm{d}_{t+1}\right\|^2\right]
	& \overset{(*)}{\le} \frac{1}{M}\sum_{i=1}^M \mathbb{E}\left[\left\|\bm{v}_{i,K}^t - \bm{x}_t\right\|^2\right] \\
	& = \frac{1}{M}\sum_{i=1}^M \mathbb{E}\left[\left\| \sum_{k=0}^{K-1} \bm{v}_{i,k+1}^t - \bm{v}_{i,k}^t\right\|^2\right] \\
	& \overset{(*)}{\le} \frac{K}{M}\sum_{i=1}^M \sum_{k=0}^{K-1}\mathbb{E}\left[\left\|  \bm{v}_{i,k+1}^t - \bm{v}_{i,k}^t\right\|^2\right] \\
	& \le \frac{K}{M}\sum_{i=1}^M \sum_{k=0}^{K-1}\left(\alpha_k^t\right)^2 \mathbb{E}\left[\left\| \bm{u}_{i,k}^t\right\|^2\right] \\
	& \le \left(\alpha_0^t\right)^2\frac{K}{M}\sum_{i=1}^M  \sum_{k=0}^{K-1} \frac{U}{k+1} 
\end{split}
\end{equation*}
where $(*)$ is due to Jensen's inequality. Applying \cref{eq:bound-sum-1-series} in \Cref{lemma:bound-partial-sum-p-series} gives \cref{eq:descent-step-bound-square}.

Similarly, the bound \cref{eq:descent-step-bound} can be obtained as
\begin{equation*}
\begin{split}
	\mathbb{E}\left[\left\|\bm{d}_{t+1}\right\|\right]
	& \le \frac{1}{M}\sum_{i=1}^M \mathbb{E}\left[\left\|\bm{v}_{i,K}^t - \bm{x}_t\right\|\right] \\
	& = \frac{1}{M}\sum_{i=1}^M \mathbb{E}\left[\left\| \sum_{k=0}^{K-1} \bm{v}_{i,k+1}^t - \bm{v}_{i,k}^t\right\|\right] \\
	& \overset{(*)}\le \frac{1}{M}\sum_{i=1}^M \sum_{k=0}^{K-1}\mathbb{E}\left[\left\|  \bm{v}_{i,k+1}^t - \bm{v}_{i,k}^t\right\|\right] \\
	& \le \frac{1}{M}\sum_{i=1}^M \sum_{k=0}^{K-1}\alpha_k^t \mathbb{E}\left[\left\| \bm{u}_{i,k}^t\right\|\right] \\
	& \le \alpha_0^t\frac{1}{M}\sum_{i=1}^M  \sum_{k=0}^{K-1} \sqrt{\frac{U}{k+1}}.
\end{split}
\end{equation*}
\rrr{where $(*)$ is due to Jensen's inequality and the last inequality is due to $\mathbb{E}\left[\|\bm{u}_{i,k}^t\|\right] \le \sqrt{\mathbb{E}\left[\|\bm{u}_{i,k}^t\|^2\right]} \le \sqrt{U}$.}
We can then reach \cref{eq:descent-step-bound} using \cref{eq:bound-sum-0.5-series} from \Cref{lemma:bound-partial-sum-p-series}.
\end{proof}

\begin{lemma} \label{lemma:bound-on-x-change}
\rrr{Assume $0 \le \beta < \sqrt{\frac{1}{2\sqrt{2}}}$.
The change of the sequence $\{\bm{x}_t\}$ in \Cref{alg:UDES} can be bounded as}
\begin{align}
\label{eq:x-change-bound-1}\frac{1}{T} \sum_{t=0}^{T-1} \mathbb{E}\left[\left\|\bm{x}_t - \bm{x}_{t-1}\right\|\right] 
	& \le \frac{160(1-\beta)\alpha\sqrt{KU}}{3T^{1/4}}, \\
\label{eq:x-change-bound-2}\mathbb{E}\left[\left\|\bm{x}_t - \bm{x}_{t-1}\right\|^2\right]
	& \le \frac{(1-\beta)^2}{\frac{1}{2\sqrt{2}}-\beta^2} UK\left(1+\log K\right) \left(\alpha_0^t\right)^2.
\end{align}
\end{lemma}

\begin{proof}
We first prove \cref{eq:x-change-bound-1}.
By construction, we have for $t > 1$
\[
\left\|\bm{x}_t - \bm{x}_{t-1}\right\| 
= \left\|\bm{m}_t\right\| 
= \left\|\beta \bm{m}_{t-1} + (1-\beta) \bm{d}_t\right\|
\le \beta \left\|\bm{m}_{t-1}\right\| + (1-\beta) \left\|\bm{d}_t\right\|.
\]
Expanding the above recursive bound gives
\[
\left\|\bm{x}_t - \bm{x}_{t-1}\right\| \le 
\left(\beta^{t-1} \|\bm{d}_1\| + \cdots + \beta \|\bm{d}_{t-1}\| + \|\bm{d}_t\|\right)(1-\beta).
\]
Taking expectation at both sides yields 
\[
\begin{split}
\mathbb{E}\left[\left\|\bm{x}_t - \bm{x}_{t-1}\right\|\right]
& \le (1-\beta)\sum_{j = 1}^{t} \beta^{t-j} \mathbb{E}\left[\left\|\bm{d}_j\right\|\right] \\
& \overset{\cref{eq:descent-step-bound}}{\le} (1-\beta)\sum_{j = 1}^{t} \beta^{t-j} 2\alpha_0^{j-1}\sqrt{KU} \\
& = 2(1-\beta)\alpha\sqrt{KU}\sum_{j = 1}^{t} \frac{\beta^{t-j}}{j^{0.25}} \\
& \overset{\cref{eq:beta-series-bound-1}}{\le} \frac{40(1-\beta)\alpha\sqrt{KU} }{t^{0.25}}
\end{split}
\]

Recall that we have defined $\bm{x}_0 = \bm{x}_{-1}$, so 
\[
\frac{1}{T} \sum_{t=0}^{T-1} \mathbb{E}\left[\left\|\bm{x}_t - \bm{x}_{t-1}\right\|\right]
\le \frac{1}{T} \sum_{t=1}^{T-1} \frac{40(1-\beta)\alpha\sqrt{KU} }{t^{0.25}} 
\rrr{\overset{\cref{eq:bound-sum-0.25-series}}{\le}} \frac{160(1-\beta)\alpha\sqrt{KU}}{3T^{1/4}}  
\]
and \cref{eq:x-change-bound-1} is proved.

\cref{eq:x-change-bound-2} is trivial for $t=0$.
For $t \ge 1$, it can be proved in a way similar to the above.

Firstly, we obtain via Jensen's inequality
\[
\left\|\bm{x}_t - \bm{x}_{t-1}\right\|^2
 = \left\|\bm{m}_t\right\|^2 = \left\|\beta \bm{m}_{t-1} + (1-\beta) \bm{d}_t\right\|^2
 \le 2 \beta^2 \left\| \bm{m}_{t-1}\right\|^2 + 2(1-\beta)^2\left\| \bm{d}_t\right\|^2.
\]
Expanding the momentum terms \rrr{$\{\bm{m}_{t-1}\}$} and taking expectation give 
\[
\begin{split}
\mathbb{E}\left[\left\|\bm{x}_t - \bm{x}_{t-1}\right\|^2\right] 
& \le 2(1-\beta)^2 \mathbb{E}\left[\left(2\beta^2\right)^{t-1} \left\|\bm{d}_1\right\|^2 + \cdots + \left(2\beta^2\right)^0\left\|\bm{d}_t\right\|^2\right] \\
& = 2(1-\beta)^2 \sum_{j=1}^t\left(2\beta^2\right)^{t-j} \mathbb{E}\left[\left\|\bm{d}_j\right\|^2\right] \\
& \overset{\cref{eq:descent-step-bound-square}}{\le} 2(1-\beta)^2 \sum_{j=1}^t\left(2\beta^2\right)^{t-j} \left(\alpha_0^{j-1}\right)^2 UK \left(1+\log K\right) \\
& = 2(1-\beta)^2 \sum_{j=1}^t \alpha^2 \frac{\left(2\beta^2\right)^{t-j}}{j^{0.5}} UK \left(1+\log K\right) \\
& \overset{\cref{eq:beta-series-bound-2}}{\le} \frac{2(1-\beta)^2 \alpha^2 UK \left(1+\log K\right)}{\sqrt{t}\left(1-2\sqrt{2}\beta^2\right)} \\
& = \sqrt{\frac{t+1}{t}}\frac{2(1-\beta)^2 \left(\alpha_0^t\right)^2 UK \left(1+\log K\right)}{1-2\sqrt{2}\beta^2}.
\end{split}
\]
The last step is due to the definition of $\alpha_0^t$. Now use the assumption $t \ge 1$ and we can reach \cref{eq:x-change-bound-2}.

\end{proof}

\begin{lemma} \label{lemma:bound-on-deviation-from-virual-sequence}
\rrr{
Assume $0 \le \beta < \sqrt{\frac{1}{2\sqrt{2}}}$. 
The worker drift in \Cref{alg:UDES} can be bounded as}
\begin{equation} \label{eq:client-drift-bound}
\mathbb{E}\left[\left\|\bm{v}_{i,k}^t - \bm{z}_t\right\|^2\right]
\le 
\frac{2}{1-2\sqrt{2}\beta^2} UK\left(1+\log K\right) \left(\alpha_0^t\right)^2.
\end{equation}
\end{lemma}

\begin{proof}
\[
\begin{split}
\mathbb{E}\left[\left\|\bm{v}_{i,k}^t - \bm{z}_t\right\|^2\right] & \le 2\mathbb{E}\left[\left\|\bm{v}_{i,k}^t - \bm{x}_t\right\|^2\right] + 2 \mathbb{E}\left[\left\|\bm{x}_t - \bm{z}_t\right\|^2\right] \\
& \overset{\cref{eq:properties-virtual-sequence}}{=} 2 \mathbb{E}\left[\left\|\bm{v}_{i,k}^t - \bm{x}_t\right\|^2\right] + 2 \left(\frac{\beta}{1-\beta}\right)^2\mathbb{E}\left[ \left\|\bm{x}_t - \bm{x}_{t-1}\right\|^2\right] \\
& \overset{\cref{eq:x-change-bound-2}}{\le} 2 \mathbb{E}\left[\left\|\bm{v}_{i,k}^t - \bm{x}_t\right\|^2\right] + \frac{2\beta^2}{\frac{1}{2\sqrt{2}}-\beta^2} UK\left(1+\log K\right) \left(\alpha_0^t\right)^2
\end{split} 
\]
where
\[
\begin{split}
\mathbb{E}\left[\left\|\bm{v}_{i,k}^t - \bm{x}_t\right\|^2\right]
& \le k \sum_{j=0}^{k-1} \mathbb{E}\left[\left\|\bm{v}_{i,j+1}^t - \bm{v}_{i,j}^t\right\|^2\right]
\le k \sum_{j=0}^{k-1} \left(\alpha_{j}^t\right)^2 \mathbb{E}\left[\left\|\bm{u}_{i,j}^t\right\|^2\right] \\
& \overset{\cref{eq:bound-sum-1-series}}\le Uk \left(\alpha_0^t\right)^2 \left(1+\log k\right)
\le UK \left(\alpha_0^t\right)^2 \left(1+\log K\right).
\end{split}
\]
We thus obtain
\begin{equation*} 
	\mathbb{E}\left[\left\|\bm{v}_{i,k}^t - \rrr{\bm{z}_t}\right\|^2\right] 
	\le \left(2+ \frac{2\beta^2}{\frac{1}{2\sqrt{2}}-\beta^2}\right) UK\left(1+\log K\right) \left(\alpha_0^t\right)^2
	\le \frac{2}{1-2\sqrt{2}\beta^2} UK\left(1+\log K\right) \left(\alpha_0^t\right)^2.
\end{equation*}
\end{proof}

\begin{lemma} \label{lemma:bound-A}
\rrr{Consider \Cref{alg:UDES}. 
Let \Cref{assumption:smoothness,assumption:variance-boundedness,assumption:iid-data} hold for some generic vector norm $\|\cdot\|$}. Denote $\mathbb{E}_{\mathcal{D}_i}$ as the expectation taken over the minibatch $\mathcal{D}_i$. We have 
\begin{equation}
\begin{split}
	\mathbb{E}_{\mathcal{D}_i} & \left[\left|\nabla f\left(\bm{z}_t\right)^T\bm{u}_{i,k}^t\right|\mathbb{I}\left\{\sign \left( f_i\left(\bm{v}_{i,k}^t\right) - f_i\left(\bm{v}_{i,k}^t + \alpha_k^t \bm{u}_{i,k}^t\right)\right) = \sign \left(\nabla f\left(\bm{z}_t\right)^T\bm{u}_{i,k}^t\right)\right\}\right] \\
	 &\;\;\;\;\;\;\; \le \frac{\alpha_k^tL + \omega_1 + \omega_2}{2} \left\| \bm{u}_{i,k}^t\right\|^2
	+ \frac{L^2}{2\omega_1} \left\|\bm{v}_{i,k}^t - \bm{z}_t\right\|^2   
	+\frac{\sigma^2}{2\omega_2b}
\end{split}
\end{equation}
\end{lemma}

\begin{proof}
Define 
\[
	\mathfrak{A} = \left|\nabla f\left(\bm{z}_t\right)^T\bm{u}_{i,k}^t\right|\mathbb{I}\left\{\sign \left( f_i\left(\bm{v}_{i,k}^t\right) - f_i\left(\bm{v}_{i,k}^t + \alpha_k^t \bm{u}_{i,k}^t\right)\right) = \sign \left(\nabla f\left(\bm{z}_t\right)^T\bm{u}_{i,k}^t\right)\right\}.
\]
By \cref{eq:indicator-neq}, we have
\begin{equation*}
\mathfrak{A} \overset{\cref{eq:indicator-neq}}{=} \left|\nabla f\left(\bm{z}_t\right)^T\bm{u}_{i,k}^t\right|\mathbb{I}\left\{\underbrace{\left| f_i\left(\bm{v}_{i,k}^t + \alpha_k^t \bm{u}_{i,k}^t\right) - f_i\left(\bm{v}_{i,k}^t\right)- \alpha_k^t \nabla f\left(\bm{z}_t\right)^T\bm{u}_{i,k}^t\right|}_{\defeq \mathfrak{B}} \ge \alpha_k^t \left|\nabla f\left(\bm{z}_t\right)^T\bm{u}_{i,k}^t\right| \right\},
\end{equation*}
where
\begin{equation*}
\begin{split}
\mathfrak{B} &\le \underbrace{\left| f_i\left(\bm{v}_{i,k}^t + \alpha_k^t \bm{u}_{i,k}^t\right) - f_i\left(\bm{v}_{i,k}^t\right)- \alpha_k^t \nabla f_i\left(\bm{v}_{i,k}^t\right)^T\bm{u}_{i,k}^t\right|}_{\mathfrak{C}_1} \\
& + \alpha_k^t \underbrace{\left|\nabla f_i\left(\bm{v}_{i,k}^t\right)^T\bm{u}_{i,k}^t - \nabla f_i\left(\bm{z}_t\right)^T\bm{u}_{i,k}^t\right|}_{\mathfrak{C}_2}
+ \alpha_k^t \underbrace{\left|\nabla f_i\left(\bm{z}_t\right)^T\bm{u}_{i,k}^t - \nabla f\left(\bm{z}_t\right)^T\bm{u}_{i,k}^t \right|}_{\mathfrak{C}_3}.
\end{split}
\end{equation*}

By \Cref{assumption:smoothness} we have
\[
	\mathfrak{C}_1 \le \frac{L}{2}\left\|\alpha_k^t \bm{u}_{i,k}^t\right\|^2.
\]
Noting that we have $|\bm{a}^T\bm{b}| \le \frac{1}{2c}\|\bm{a}\|^2_* + \frac{c}{2}\|\bm{b}\|^2$ for any $\bm{a},\bm{b} \in \mathbb{R}^n$ and $c \in \mathbb{R}_+$, so
\[
\mathfrak{C}_2 \le \frac{1}{2\omega_1} \left\|\nabla f_i\left(\bm{v}_{i,k}^t\right) - \nabla f_i\left(\bm{z}_t\right)\right\|^2_* + \frac{\omega_1}{2} \left\|\bm{u}_{i,k}^t\right\|^2
\le \frac{L^2}{2\omega_1} \left\| \bm{v}_{i,k}^t - \bm{z}_t\right\|^2 + \frac{\omega_1}{2} \left\|\bm{u}_{i,k}^t\right\|^2
\]
and
\[
	\begin{split}
	\mathfrak{C}_3 & \le \frac{1}{2\omega_2} \left\|\nabla f_i\left(\bm{z}_t\right) - \nabla f\left(\bm{z}_t\right)\right\|^2_* + \frac{\omega_2}{2} \left\|\bm{u}_{i,k}^t\right\|^2,
	\end{split}
\]
for some $\omega_1, \omega_2 \in \mathbb{R}_+$.

Putting all these together, we reach
\begin{equation*}
\begin{split}
\mathfrak{A} & = \left|\nabla f\left(\bm{z}_t\right)^T\bm{u}_{i,k}^t\right|\mathbb{I}\left\{\mathfrak{B} \ge \alpha_k^t \left|\nabla f\left(\bm{z}_t\right)^T\bm{u}_{i,k}^t\right| \right\} \\
& \le \left|\nabla f\left(\bm{z}_t\right)^T\bm{u}_{i,k}^t\right|\mathbb{I}\left\{\frac{\alpha_k^tL + \omega_1 + \omega_2}{2}\left\| \bm{u}_{i,k}^t\right\|^2
+ \frac{L^2}{2\omega_1} \left\|\bm{v}_{i,k}^t - \bm{z}_t\right\|^2
+\frac{\left\|\nabla f_i\left(\bm{z}_t\right) - \nabla f\left(\bm{z}_t\right)\right\|^2_*}{2\omega_2}  \ge \left|\nabla f\left(\bm{z}_t\right)^T\bm{u}_{i,k}^t\right| \right\}
\end{split}
\end{equation*}

Now take expectation over $\mathcal{D}_i$. Noting that \Cref{assumption:variance-boundedness,assumption:iid-data} indicate that the gradient variance can be scaled down by a factor of $b = \left|\mathcal{D}_i\right|$, so we have, based on the Markov inequality, 
\begin{equation*} 
\begin{split}
\mathbb{E}_{\mathcal{D}_i}\left[\mathfrak{A} \right]
& \rrr{\le} \left|\nabla f\left(\bm{z}_t\right)^T\bm{u}_{i,k}^t\right|
	\mathbb{P}_{\mathcal{D}_i}\left\{\frac{\alpha_k^tL + \omega_1 + \omega_2}{2}\left\| \bm{u}_{i,k}^t\right\|^2 
+ \frac{L^2}{2\omega_1} \left\|\bm{v}_{i,k}^t - \bm{z}_t\right\|^2   
+\frac{\left\|\nabla f_i\left(\bm{z}_t\right) - \nabla f\left(\bm{z}_t\right)\right\|^2_*}{2\omega_2}  \ge \left|\nabla f\left(\bm{z}_t\right)^T\bm{u}_{i,k}^t\right| \right\} \\
& \le \frac{\alpha_k^tL + \omega_1 + \omega_2}{2} \left\| \bm{u}_{i,k}^t\right\|^2
	+ \frac{L^2}{2\omega_1} \left\|\bm{v}_{i,k}^t - \bm{z}_t\right\|^2   
	+\frac{\mathbb{E}_{\mathcal{D}_i}\left[\left\|\nabla f_i\left(\bm{z}_t\right) - \nabla f\left(\bm{z}_t\right)\right\|^2_*\right]}{\rrr{2\omega_2}} \\
& \le \frac{\alpha_k^tL + \omega_1 + \omega_2}{2} \left\| \bm{u}_{i,k}^t\right\|^2
	+ \frac{L^2}{2\omega_1} \left\|\bm{v}_{i,k}^t - \bm{z}_t\right\|^2   
	+\frac{\sigma^2}{2\omega_2b}.
\end{split}
\end{equation*}
\end{proof}

\section{Proof of \Cref{theorem:convergence-DES-l2,theorem:convergence-DES-l1}}
\rrr{In this section we proof the convergence results for \Cref{alg:DES}. Since \Cref{alg:DES} is a special case of \Cref{alg:UDES} with Gaussian mutation, we can proceed in two steps. In the first step, we start from \Cref{lemma:bound-on-descent-step,lemma:bound-on-deviation-from-virual-sequence,lemma:bound-A} (which are obtained for \Cref{alg:UDES}) with the specification $\bm{u}_{i,k}^t \sim \mathcal{N}(\bm{0},\bm{I})$. This admits bounding the gradient norm averaged over the virtual sequence $\{\bm{z}_t\}$ with some constant $U$. The result is given in \Cref{lemma:key-lemma}. Then, in the second step, we further specify the vector norm used in the assumptions, from which we can get the detailed values for $U$. In particular, based on \Cref{lemma:key-lemma,lemma:bound-on-x-change}, we can prove \Cref{theorem:convergence-DES-l2} with the specification $\|\cdot\| = \|\cdot\|_2$ and prove \Cref{theorem:convergence-DES-l1} with $\|\cdot\| = \|\cdot\|_\infty$.}

\begin{lemma} \label{lemma:key-lemma}
Let \Cref{assumption:smoothness,assumption:variance-boundedness,assumption:iid-data} hold \rrr{for some generic vector norm $\|\cdot\|$}. 
The virtual sequence $\bm{z}_t$ produced by \Cref{alg:DES} satisfies, \rrr{for some $U \ge \mathbb{E}\left[\|\bm{u}_{i,k}^t\|^2\right]$},
\begin{equation} \label{eq:z-sequence-gradient-bound}
\frac{1}{T}\sum_{t=0}^{T-1} \mathbb{E}\left[\left\|\nabla f\left(\bm{z}_t\right)\right\|_2\right]  
\le \frac{\sqrt{2\pi}}{\alpha T^{3/4}} \left(\frac{f\left(\bm{x}_0\right) - f_*}{\sqrt{K}} + LU \Psi \sum_{t=0}^{T-1} \left(\alpha_0^t\right)^2 + 2\sigma\sqrt{\frac{U}{b}} \sum_{t=0}^{T-1} \alpha_0^t\right),
\end{equation}
where $\Psi$ is given in \cref{eq:psi-definition}. 
\end{lemma}
\begin{proof}

First rewrite $\mathbb{E}\left[\nabla f\left(\bm{z}_t\right)^T \bm{d}_{t+1}\right]$ as
\begin{equation*}
\begin{split}
\mathbb{E} & \left[\nabla f\left(\bm{z}_t\right)^T \bm{d}_{t+1}\right] \\
& = \mathbb{E}\left[\nabla f\left(\bm{z}_t\right)^T \left(\frac{1}{M}\sum_{i=1}^M \bm{v}_{i,K}^t - \bm{x}_t\right)\right] \\
& = \frac{1}{M}\sum_{i=1}^M \sum_{k=0}^{K-1} \alpha_k^t \mathbb{E}\left[\signplus\left( f_i\left(\bm{v}_{i,k}^t\right) - f_i\left(\bm{v}_{i,k}^t + \alpha_k^t \bm{u}_{i,k}^t\right)\right)\nabla f\left(\bm{z}_t\right)^T\bm{u}_{i,k}^t\right] \\
& \overset{\cref{eq:definition-sign-signplus}}{=} \frac{1}{2M}\sum_{i=1}^M \sum_{k=0}^{K-1} \alpha_k^t \mathbb{E}\left[\left(1+\sign\left( f_i\left(\bm{v}_{i,k}^t\right) - f_i\left(\bm{v}_{i,k}^t + \alpha_k^t \bm{u}_{i,k}^t\right)\right)\right)\nabla f\left(\bm{z}_t\right)^T\bm{u}_{i,k}^t\right] \\
& \overset{(*)}{=} \frac{1}{2M}\sum_{i=1}^M \sum_{k=0}^{K-1} \alpha_k^t \mathbb{E}\left[\sign\left( f_i\left(\bm{v}_{i,k}^t\right) - f_i\left(\bm{v}_{i,k}^t + \alpha_k^t \bm{u}_{i,k}^t\right)\right)\nabla f\left(\bm{z}_t\right)^T\bm{u}_{i,k}^t\right] \\
& \overset{\cref{eq:sign_identity}}{=} \frac{1}{2M}\sum_{i=1}^M \sum_{k=0}^{K-1} \alpha_k^t \mathbb{E}\left[\left|\nabla f\left(\bm{z}_t\right)^T\bm{u}_{i,k}^t\right|\left(-1 + 2 \mathbb{I}\left\{\sign\left( f_i\left(\bm{v}_{i,k}^t\right) - f_i\left(\bm{v}_{i,k}^t + \alpha_k^t \bm{u}_{i,k}^t\right)\right) = \sign\left(\nabla f\left(\bm{z}_t\right)^T\bm{u}_{i,k}^t\right)\right\}\right)\right],
\end{split}
\end{equation*}
where $(*)$ is due to $\mathbb{E}\left[\bm{u}_{i,k}^t\right] = \bm{0}$.

\rrr{Now specify $\bm{u}_{i,k}^t \sim \mathcal{N}(\bm{0},\bm{I})$. Using the identity in \cref{eq:expectation-half-gaussian}, we have}
\begin{equation*}
\begin{split}
\mathbb{E} &\left[\nabla f\left(\bm{z}_t\right)^T \bm{d}_{t+1}\right] \\
& \overset{\cref{eq:expectation-half-gaussian}} \le -\frac{\mathbb{E}\left[\left\|\nabla f\left(\bm{z}_t\right)\right\|_2\right]}{\sqrt{2\pi}} \sum_{k=0}^{K-1} \alpha_k^t \\
& \;\;\;\;\;\;\;\;\;+ \frac{1}{M}\sum_{i=1}^M \sum_{k=0}^{K-1} \alpha_k^t \mathbb{E}\left[\underbrace{\left|\nabla f\left(\bm{z}_t\right)^T\bm{u}_{i,k}^t\right|\mathbb{I}\left\{\sign\left( f_i\left(\bm{v}_{i,k}^t\right) - f_i\left(\bm{v}_{i,k}^t + \alpha_k^t \bm{u}_{i,k}^t\right)\right) = \sign\left(\nabla f\left(\bm{z}_t\right)^T\bm{u}_{i,k}^t\right)\right\}}_{\defeq \mathfrak{A}}\right] \\
& \overset{\cref{eq:bound-sum-0.5-series}}\le -\frac{\mathbb{E}\left[\left\|\nabla f\left(\bm{z}_t\right)\right\|_2\right]}{\sqrt{2\pi}} \alpha_0^t \sqrt{K} + \frac{1}{M}\sum_{i=1}^M \sum_{k=0}^{K-1} \alpha_k^t \mathbb{E}\left[\mathfrak{A}\right],
\end{split}
\end{equation*}

\rrr{Now use \Cref{lemma:bound-on-variance,lemma:bound-A} to bound $\mathbb{E}\left[\mathfrak{A}\right]$:}
\begin{equation*} 
\begin{split}
	\mathbb{E} & \left[\nabla f\left(\bm{z}_t\right)^T \bm{d}_{t+1}\right] + \frac{\mathbb{E}\left[\left\|\nabla f\left(\bm{z}_t\right)\right\|_2\right]}{\sqrt{2\pi}} \alpha_0^t \sqrt{K}  \\
	& \le \frac{1}{2M}\sum_{i=1}^M \sum_{k=0}^{K-1} \alpha_k^t \left\{\left(\alpha_k^tL + \omega_1 + \omega_2\right) U
	+ \frac{L^2}{\omega_1} \mathbb{E}\left[\left\|\bm{v}_{i,k}^t - \bm{z}_t\right\|^2\right]  
	+\frac{\sigma^2}{\omega_2b}\right\} \\
	& \le 
	\frac{L^2}{2M\omega_1} \sum_{i=1}^M \sum_{k=0}^{K-1} \alpha_k^t  \mathbb{E}\left[\left\|\bm{v}_{i,k}^t - \bm{z}_t\right\|^2\right]
	+ \frac{LU}{2}\sum_{k=0}^{K-1}\left(\alpha_k^t\right)^2
	+ \left(\frac{\omega_1 + \omega_2}{2}U + \frac{\sigma^2}{2\omega_2b}\right) \sum_{k=0}^{K-1}\alpha_k^t\\
	& \overset{(\ref{eq:bound-sum-1-series},\ref{eq:bound-sum-0.5-series})}{\le} 
	\frac{L^2}{2M\omega_1} \sum_{i=1}^M \sum_{k=0}^{K-1} \alpha_k^t  \mathbb{E}\left[\left\|\bm{v}_{i,k}^t - \bm{z}_t\right\|^2\right]
	+ \frac{LU}{2} (1+\log K)\left(\alpha_0^t\right)^2
	+ \left((\omega_1 + \omega_2)U + \frac{\sigma^2}{\omega_2b}\right) \sqrt{K}\alpha_0^t
\end{split}
\end{equation*}

Using \Cref{lemma:bound-on-deviation-from-virual-sequence,eq:bound-sum-0.5-series} yields
\[
\begin{split}
	\mathbb{E} &\left[\nabla f\left(\bm{z}_t\right)^T \bm{d}_{t+1}\right] + \frac{\mathbb{E}\left[\left\|\nabla f\left(\bm{z}_t\right)\right\|_2\right]}{\sqrt{2\pi}} \alpha_0^t \sqrt{K}  \\
	& \le LU\sqrt{K}\left(\frac{\alpha_0^tL}{\omega_1}   \frac{2}{1-2\sqrt{2}\beta^2} K + \frac{1}{2\sqrt{K}} \right)(1+\log K)\left(\alpha_0^t\right)^2
	+ \left((\omega_1 + \omega_2)U + \frac{\sigma^2}{\omega_2b}\right) \sqrt{K}\alpha_0^t \\
\end{split}
\]
Consider now the setting $\omega_1 = \frac{L\alpha_0^t}{\sqrt{K}}, \omega_2 = \frac{\sigma}{\sqrt{Ub}}$, and we can reach 
\begin{equation} \label{eq:expected-z-decent}
	\begin{split}
	\mathbb{E} & \left[\nabla f\left(\bm{z}_t\right)^T \bm{d}_{t+1}\right] \\
	& \le -\frac{\mathbb{E}\left[\left\|\nabla f\left(\bm{z}_t\right)\right\|_2\right]}{\sqrt{2\pi}} \alpha_0^t \sqrt{K}  
	+ LU\sqrt{K}\left(\left(  \frac{2}{1-2\sqrt{2}\beta^2} \sqrt{K}
					+ \frac{1}{2\sqrt{K}}\right) (1+\log K)+\sqrt{K}\right) \left(\alpha_0^t\right)^2 + 2\sigma\sqrt{K}\sqrt{\frac{U}{b}} \alpha_0^t.
\end{split}
\end{equation}

Using \Cref{assumption:smoothness}, we have
\begin{equation} \label{eq:per-iteration-progress-tmp}
\begin{split}
f\left(\bm{z}_{t+1}\right) & \le f\left(\bm{z}_t\right) + \nabla f\left(\bm{z}_t\right)^T (\bm{z}_{t+1} - \bm{z}_t) + \frac{L}{2}\left\|\bm{z}_{t+1}-\bm{z}_t\right\|^2 \\
& \overset{\cref{eq:properties-virtual-sequence}}{=} f\left(\bm{z}_t\right) + \nabla f\left(\bm{z}_t\right)^T \bm{d}_{t+1} + \frac{L}{2}\left\|\bm{d}_{t+1}\right\|^2
\end{split}
\end{equation}

Taking total expectation, using \cref{eq:descent-step-bound-square,eq:expected-z-decent}, and rearranging yield 
\begin{equation*}
\begin{split}
\frac{\mathbb{E}\left[\left\|\nabla f\left(\bm{z}_t\right)\right\|_2\right]}{\sqrt{2\pi}} \alpha_0^t  
& \le \frac{\mathbb{E}\left[f\left(\bm{z}_t\right)-f\left(\bm{z}_{t+1}\right)\right]}{\sqrt{K}} \\
& + LU\left(\underbrace{\left(\left(\frac{2}{1-2\sqrt{2}\beta^2} + \frac{1}{2}\right)\sqrt{K} + \frac{1}{2\sqrt{K}} \right) (1+\log K)+\sqrt{K}}_{\defeq \Psi} \right) \left(\alpha_0^t\right)^2
+ 2\sigma\sqrt{\frac{U}{b}} \alpha_0^t.
\end{split}
\end{equation*}
Summing over $t=0,\cdots,T-1$ gives
\[
\sum_{t=0}^{T-1} \frac{\mathbb{E}\left[\left\|\nabla f\left(\bm{z}_t\right)\right\|_2\right]}{\sqrt{2\pi}} \alpha_0^t  
\le \frac{f\left(\bm{z}_0\right) - f_*}{\sqrt{K}} + LU \Psi \sum_{t=0}^{T-1} \left(\alpha_0^t\right)^2 + 2\sigma\sqrt{\frac{U}{b}} \sum_{t=0}^{T-1} \alpha_0^t.
\]
By \cref{eq:bound-sum-0.25-series-2}, the left-hand side is \rrr{no smaller than} $\frac{\alpha T^{3/4}}{\sqrt{2\pi}} \frac{1}{T} \sum_{t=0}^{T-1} \mathbb{E}\left[\left\|\nabla f\left(\bm{z}_t\right)\right\|_2\right]$. And noting that, by definition, $\bm{z}_0 = \bm{x}_0$, we then obtain \cref{eq:z-sequence-gradient-bound}.
\end{proof}

\begin{proof}[\textbf{Proof of \Cref{theorem:convergence-DES-l2}}]
\rrr{Under \Cref{assumption:smoothness} and using the specification $\|\cdot\| = \|\cdot\|_* = \|\cdot\|_2$, we have }
\[
\|\nabla f(\bm{x}_t)\|_2	\le \|\nabla f(\bm{x}_t) - \nabla f(\bm{z}_t)\|_2 + \|\nabla f(\bm{z}_t)\|_2
\le L\|\bm{x}_t - \bm{z}_t\|_2 + \|\nabla f(\bm{z}_t)\|_2
\overset{\cref{eq:properties-virtual-sequence}}{\rrr{=}} \frac{L\beta}{1-\beta} \|\bm{x}_t - \bm{x}_{t-1}\|_2 + \|\nabla f(\bm{z}_t)\|_2
\]
which gives, via taking expectation, 
\[
	\mathbb{E}\left[\|\nabla f(\bm{z}_t)\|_2\right] \ge \mathbb{E}\left[\|\nabla f(\bm{x}_t)\|_2\right] - \frac{L\beta}{1-\beta} \mathbb{E}\left[\|\bm{x}_t - \bm{x}_{t-1}\|_2\right].
\]
Substituting this into \cref{eq:z-sequence-gradient-bound} and using $b \ge \sqrt{T}$ yield
\begin{equation*} 
\begin{split}
\frac{1}{T}\sum_{t=0}^{T-1} \mathbb{E}\left[\|\nabla f(\bm{x}_t)\|_2\right]  
& \le 
\frac{\sqrt{2\pi}}{\alpha T^{3/4}}\left(\frac{f\left(\bm{x}_0\right) - f_*}{\sqrt{K}} 
+ LU \Psi \sum_{t=0}^{T-1} \left(\alpha_0^t\right)^2 
+ 2\sigma\sqrt{\frac{U}{b}} \sum_{t=0}^{T-1} \alpha_0^t \right)
+\frac{L\beta}{1-\beta} \frac{1}{T}\sum_{t=0}^{T-1}  \mathbb{E}\left[\|\bm{x}_t - \bm{x}_{t-1}\|_2\right] \\
& \overset{(\ref{eq:x-change-bound-1}),(\ref{eq:bound-sum-0.5-series}),(\ref{eq:bound-sum-0.25-series})}{\le} 
\frac{\sqrt{2\pi}}{\alpha T^{3/4}}\left(\frac{f\left(\bm{x}_0\right) - f_*}{\sqrt{K}} 
+ LU \Psi \alpha^2 2\sqrt{T}
+ 2\sigma\sqrt{\frac{U}{b}} \alpha \frac{4}{3} T^{3/4} \right)
+ L\beta\frac{160\alpha\sqrt{KU}}{3T^{1/4}} \\
& \le \frac{\sqrt{2\pi}}{T^{3/4}}\frac{f\left(\bm{x}_0\right) - f_*}{\alpha\sqrt{K}}
+ \frac{\sqrt{U}}{T^{1/4}}\left(
	2\alpha L\left(\sqrt{2\pi U} \Psi  
			+ \frac{80\beta\sqrt{K}}{3}\right)
	+ \frac{8\sqrt{2\pi}\sigma}{3}    
	\right).
\end{split}
\end{equation*}
\rrr{where when using \cref{eq:x-change-bound-1} we have specified $\|\cdot\|=\|\cdot\|_2$. 
Finally, according to \Cref{lemma:bound-on-variance}, we have $\mathbb{E}\left[\|\bm{u}_{i,k}^t\|_2^2\right] = n$ when $\bm{u}_{i,k}^t \sim \mathcal{N}(\bm{0},\bm{I})$. We can therefore choose $U=n$ and then reach the target bound. }
\end{proof}

\begin{proof}[\textbf{Proof of \Cref{theorem:convergence-DES-l1}}]

\rrr{Firstly, the assumption $\|\nabla f(\bm{x})\|_0 \le s$ implies}
\begin{equation*} 
	\|\nabla f(\bm{x})\|_\infty \le \|\nabla f(\bm{x})\|_2 \le \|\nabla f(\bm{x})\|_1 \le \sqrt{s} \|\nabla f(\bm{x})\|_\infty,
\end{equation*}
and hence we have
\[
\begin{split}
\|\nabla f(\bm{x}_t)\|_1 
& \le \|\nabla f(\bm{x}_t) - \nabla f(\bm{z}_t)\|_1 + \|\nabla f(\bm{z}_t)\|_1 \\
& \le \|\nabla f(\bm{x}_t) - \nabla f(\bm{z}_t)\|_1 + \sqrt{s}\|\nabla f(\bm{z}_t)\|_2 \\
& \overset{(*)}{\le} L \|\bm{x}_t - \bm{z}_t\|_\infty + \sqrt{s}\|\nabla f(\bm{z}_t)\|_2 \\
& \rrr{\overset{\cref{eq:properties-virtual-sequence}}{=}}  \frac{L\beta}{1-\beta} \|\bm{x}_t - \bm{x}_{t-1}\|_\infty + \sqrt{s}\|\nabla f(\bm{z}_t)\|_2
\end{split}
\]
where $(*)$ uses \Cref{assumption:smoothness} \rrr{with the specification $\|\cdot\| = \|\cdot\|_{\infty}$ and $\|\cdot\|_* = \|\cdot\|_1$}.
Taking expectation and rearranging give
\[
\mathbb{E}[\|\nabla f(\bm{z}_t)\|_2] 
\ge \frac{1}{\sqrt{s}}\left(\mathbb{E}[\|\nabla f(\bm{x}_t)\|_1] - \frac{L\beta}{1-\beta}\mathbb{E}[\|\bm{x}_t - \bm{x}_{t-1}\|_\infty]\right).
\]
Substituting this into the left-hand side of \cref{eq:z-sequence-gradient-bound} \rrr{in \Cref{lemma:key-lemma}} yields
\begin{equation*} 
\begin{split}
\frac{1}{T}\sum_{t=0}^{T-1} \mathbb{E}[\|\nabla f(\bm{x}_t)\|_1]   
& \le \frac{L\beta}{1-\beta}\frac{1}{T}\sum_{t=0}^{T-1} \mathbb{E}[\|\bm{x}_t - \bm{x}_{t-1}\|_\infty] 
+ \frac{\sqrt{2\pi s}}{\alpha T^{3/4}}\left(\frac{f\left(\bm{x}_0\right) - f_*}{\sqrt{K}} + LU \Psi \sum_{t=0}^{T-1} \left(\alpha_0^t\right)^2 + 2\sigma\sqrt{\frac{U}{b}} \sum_{t=0}^{T-1} \alpha_0^t\right) \\
& \overset{\cref{eq:x-change-bound-1}}\le \frac{160L\beta\alpha\sqrt{KU}}{3T^{1/4}} 
+ \frac{\sqrt{2\pi s}}{\alpha T^{3/4}}\left(\frac{f\left(\bm{x}_0\right) - f_*}{\sqrt{K}} + LU \Psi \sum_{t=0}^{T-1} \left(\alpha_0^t\right)^2 + 2\sigma\sqrt{\frac{U}{b}} \sum_{t=0}^{T-1} \alpha_0^t\right) \\
& \overset{(\ref{eq:bound-sum-0.5-series},\ref{eq:bound-sum-0.25-series})}\le 
	\frac{160L\beta\alpha\sqrt{KU}}{3T^{1/4}} 
	+ \frac{\sqrt{2\pi s}}{\alpha T^{3/4}}\left(\frac{f\left(\bm{x}_0\right) - f_*}{\sqrt{K}} 
	+ LU \Psi \alpha^2 2\sqrt{T}
	+ 2\sigma\sqrt{\frac{U}{b}} \alpha \frac{4}{3}T^{3/4}\right) \\
& \le \frac{\sqrt{2\pi s}}{T^{3/4}} \frac{f\left(\bm{x}_0\right) - f_*}{\alpha\sqrt{K}}
	+ \frac{\sqrt{U}}{T^{1/4}} \left(
		2\alpha L \left(\sqrt{2\pi U s} \Psi + \frac{80\beta\sqrt{K}}{3}  \right)
		+  \frac{8\sqrt{2\pi s} \sigma}{3} \right)
\end{split}
\end{equation*}
where the last step uses the assumption $b \ge \sqrt{T}$. 

\rrr{Since we have used \Cref{lemma:key-lemma}, we need $U \ge \mathbb{E}\left[\|\bm{u}_{i,k}^t\|_\infty^2\right]$. According to \Cref{lemma:bound-on-variance}, we know $U=4 \log(\sqrt{2}n)$ is valid choice. We then obtain the final bound as}
\[
\frac{1}{T}\sum_{t=0}^{T-1} \mathbb{E}[\|\nabla f(\bm{x}_t)\|_1]   
\le \frac{\sqrt{2\pi s}}{T^{3/4}} \frac{f\left(\bm{x}_0\right) - f_*}{\alpha\sqrt{K}}
	+ \frac{8\sqrt{\log(\sqrt{2}n)}}{T^{1/4}} \left(
		\alpha L \left(\sqrt{2\pi s\log(\sqrt{2}n)} \Psi + \frac{10\beta\sqrt{K}}{3}  \right)
		+  \frac{2\sqrt{2\pi s} \sigma}{3} \right)
\]
\end{proof}

\section{Proof of \Cref{proposition:properties-of-mixture-Gaussian-sampling,proposition:properties-of-mixture-Rademacher-sampling}}
\rrr{
In the above proofs for DES with Gaussian mutation, we have repeatedly used the lower bound of $\mathbb{E}[|\bm{u}^T \bm{y}|]$ where $\bm{y} \in \mathbb{R}^n$ and $\bm{u}$ is random. This bound is trivial when $\bm{u} \sim \mathcal{N}(\bm{0},\bm{I})$, as has been given in \Cref{eq:expectation-half-gaussian}. To prove \Cref{theorem:convergence-DES-mixture-Gaussian-l2,theorem:convergence-DES-mixture-Rademacher-l2} we need a similar bound when $\bm{u}$ is sampled from the mixture Gaussian distribution $\mathcal{M}_l^G$ or the mixture Rademacher distribution $\mathcal{M}_l^R$. This can be achieved by analyzing the second-order and the fourth-order momentums of the corresponding probability distribution; this is the reason why \Cref{proposition:properties-of-mixture-Gaussian-sampling,proposition:properties-of-mixture-Rademacher-sampling} are required.
}

\begin{proof}[\textbf{Proof of \Cref{proposition:properties-of-mixture-Gaussian-sampling}}]
We prove this proposition using moment-generating function.

Denote the moment-generating function of $\mathcal{M}_l^G$ by $M(\bm{t})$. By definition, $M(\bm{t})$ can be written as
\[
\begin{split}
M(\bm{t}) & = \mathbb{E}\left[\exp (\bm{t}^T \bm{u})\right]
= \mathbb{E}\left[\exp \sqrt{\frac{n}{l}}\left(\bm{t}^T \sum_{j=1}^l \bm{e}_{r_j} z_j\right)\right]
= \mathbb{E}\left[\exp \sqrt{\frac{n}{l}} \left(\sum_{j=1}^l t_{r_j} z_j\right)\right]
\overset{(*)}= \prod_{j=1}^l \mathbb{E}\left[\exp \left( \sqrt{\frac{n}{l}} t_{r_j} z_j\right)\right] \\
& = \prod_{j=1}^l \mathbb{E}\left[\sum_{k=1}^n \mathbb{I}\{r_j = k\}\exp \left( \sqrt{\frac{n}{l}} t_{r_j} z_j\right)\right]
= \prod_{j=1}^l \sum_{k=1}^n\mathbb{P}\{r_j = k\}\mathbb{E}_k\left[ \exp \left(\sqrt{\frac{n}{l}} t_k z_j\right)\right]
\end{split}
\]
where $t_{r_j}$ denotes the $r_j$-th element of $\bm{t}$ and $\mathbb{E}_k$ denotes the expectation conditioned on the event $r_j = k$. Equation ($*$) is due to the independence of $\{z_j\}$ and $\{r_j\}$.

Since the coordinate index $r_j$ is sampled uniformly with replacement, we have $\mathbb{P}\{r_j = k\} = \frac{1}{n}$. Note that $\mathbb{E}_k [\exp (t_k z_j)]$ is in fact the (conditioned) moment-generating function of the univariate Gaussian variable $\sqrt{\frac{n}{l}} z_j$, which is given by $\exp \left(\frac{n}{2l} t_k^2\right)$. So we reach
\begin{equation*} 
M(\bm{t}) = \prod_{j=1}^l \sum_{k=1}^n \frac{1}{n}\exp \left(\frac{n}{2l} t_k^2\right)
= \left(\frac{1}{n}\sum_{k=1}^n \exp \left(\frac{n}{2l} t_k^2\right)\right)^l.
\end{equation*}

By construction, the covariance matrix must be diagonal, so we focus on its diagonal elements. Firstly, take the partial derivative with respect to $t_j$ and this yields
\[
\frac{\partial M(\bm{t})}{\partial t_j} 
= \frac{l}{n}\left(\frac{1}{n}\sum_{k=1}^n \exp \left(\frac{n}{2l} t_k^2\right)\right)^{l-1} \frac{\partial }{\partial t_j} \exp \rrr{\left(\frac{n}{2l} t_j^2\right)} 
= \left( \frac{1}{n}\sum_{k=1}^n\exp \left(\frac{n}{2l} t_k^2\right)\right)^{l-1} \exp \left(\frac{n}{\rrr{2l}} t_j^2\right) t_j.
\]
The second-order partial derivative is then
\[
\frac{\partial^2 M(\bm{t})}{\partial t_j^2} 
= \underbrace{\left\{\frac{\partial }{\partial t_j} \left(\frac{1}{n}\sum_{k=1}^n \exp \left(\frac{n}{2l} t_k^2\right)\right)^{l-1}\right\} \exp \left(\frac{n}{\rrr{2l}} t_j^2\right) t_j}_{T_1}
+ \underbrace{\left(\frac{1}{n}\sum_{k=1}^n \exp \left(\frac{n}{2l} t_k^2\right)\right)^{l-1}}_{T_2} 
\underbrace{\frac{\partial }{\partial t_j}\left(\exp \left(\frac{n}{2l} t_j^2\right) t_j\right)}_{T_3}.
\]
When setting $\bm{t} = \bm{0}$, $T_1$ vanishes and $T_2$ becomes 1. 
We also have
\[
	T_3 = \exp \left(\frac{n}{\rrr{2l}} t_j^2\right) \left(\frac{\partial }{\partial t_j} \left(\frac{n}{2l} t_j^2\right)\right) t_j + \exp \left(\frac{n}{2l} t_j^2\right) \overset{\bm{t}=\bm{0}} \Rightarrow 1.
\]
So the $j$-th diagonal element is 1. We therefore conclude that $\bm{u}$ has an identity covariance matrix.

In a similar manner, the moment-generating function of $\bm{y}^T \bm{u}$ is
\[
\tilde{M}(t) = \left(\frac{1}{n}\sum_{k=1}^n \exp \left(\frac{n}{2l} y_k^2 t^2\right)\right)^l.
\]
Now expand the exponential term as Taylor series  
\[
\tilde{M}(t) 
= \left(\frac{1}{n}\sum_{k=1}^n \left(1 + \frac{n}{2l} y_k^2 t^2 + \frac{1}{2}\left(\frac{n}{2l} y_k^2 t^2\right)^2 + \mathcal{O}(t^6) \right)\right)^l 
= \left(1 + \frac{1}{2l} \|\bm{y}\|_2^2 t^2 + \frac{n}{8l^2} \|\bm{y}\|_4^4 t^4 + \mathcal{O}(t^6)\right)^l.
\]

Using the multinomial theorem, we get
\[
\begin{split}
\tilde{M}(t) 
& = \sum_{j=0}^l \left(\substack{l \\ j}\right) \left(\frac{1}{2l} \|\bm{y}\|_2^2 t^2 + \frac{n}{8l^2} \|\bm{y}\|_4^4 t^4 + \mathcal{O}(t^6)\right)^j \\
& = \left(\substack{l \\ 1}\right) \left(\frac{n}{8l^2} \|\bm{y}\|_4^4 t^4 \right) 
+ \left(\substack{l \\ 2}\right) \left(\frac{1}{2l} \|\bm{y}\|_2^2 t^2\right)^2 + 1 + A t^2 + \mathcal{O}(t^6)\\
& = \frac{1}{8}\left(\frac{n}{l} \|\bm{y}\|_4^4 + \frac{l-1}{l} \|\bm{y}\|_2^4\right) t^4 
+ 1 + A t^2 + \mathcal{O}(t^6)
\end{split} 
\]
where $A$ is some constant not depending on $t$. We can then reach the desired result by taking the fourth-order derivative and setting $t=0$, i.e.,
\[
	\mathbb{E}\left[|\bm{y}^T \bm{u}|^4\right] 
	= \frac{\partial^4}{\partial t^4}\tilde{M}(t) \Bigg|_{t=0}
	= 3\left(\frac{n}{l} \|\bm{y}\|_4^4 + \frac{l-1}{l} \|\bm{y}\|_2^4\right) + \mathcal{O}(t^2) \Bigg|_{t=0}
	= 3\left(\frac{n}{l} \|\bm{y}\|_4^4 + \frac{l-1}{l} \|\bm{y}\|_2^4\right).
\]
\end{proof}

\begin{proof}[\textbf{Proof of \Cref{proposition:properties-of-mixture-Rademacher-sampling}}]
The proof is very similar to that of \Cref{proposition:properties-of-mixture-Gaussian-sampling}.
First, we obtain the moment-generating function for $\mathcal{M}_l^R$ as
\[
\begin{split}
M(\bm{t}) & = \mathbb{E}\left[\exp (\bm{t}^T \bm{u})\right]
= \mathbb{E}\left[\exp \sqrt{\frac{n}{l}}\left(\bm{t}^T \sum_{j=1}^l \bm{e}_{r_j} z_j\right)\right]
= \mathbb{E}\left[\exp \sqrt{\frac{n}{l}} \left(\sum_{j=1}^l t_{r_j} z_j\right)\right]
\overset{(*)}{=} \prod_{j=1}^l \mathbb{E}\left[\exp \left( \sqrt{\frac{n}{l}} t_{r_j} z_j\right)\right] \\
& = \prod_{j=1}^l \mathbb{E}\left[\sum_{k=1}^n \mathbb{I}\{r_j = k\}\exp \left( \sqrt{\frac{n}{l}} t_{r_j} z_j\right)\right]
= \prod_{j=1}^l \sum_{k=1}^n\mathbb{P}\{r_j = k\}\mathbb{E}_k\left[ \exp \left(\sqrt{\frac{n}{l}} t_k z_j\right)\right] \\
& = \prod_{j=1}^l \frac{1}{2n}\sum_{k=1}^n \left(\exp \left(\sqrt{\frac{n}{l}} t_k\right) + \exp \left(-\sqrt{\frac{n}{l}} t_k\right)\right)
= \left(\frac{1}{2n}\sum_{k=1}^n \left(\exp \left(\sqrt{\frac{n}{l}} t_k\right) + \exp \left(-\sqrt{\frac{n}{l}} t_k\right)\right)\right)^l.
\end{split}
\]
\rrr{Equation ($*$) in the above is due to the independence of $\{z_j\}$ and $\{r_j\}$.}
The partial derivative with respect to $t_j$ is then
\[
\frac{\partial M(\bm{t})}{\partial t_j} 
= \frac{1}{2} \sqrt{\frac{l}{n}} \left(\frac{1}{2n}\sum_{k=1}^n \left(\exp \left(\sqrt{\frac{n}{l}} t_k\right) + \exp \left(-\sqrt{\frac{n}{l}} t_k\right)\right)\right)^{l-1} \left(\exp \left(\sqrt{\frac{n}{l}} t_j\right) - \exp \left(-\sqrt{\frac{n}{l}} t_j\right)\right)
\]
and
\[
\begin{split}
& \frac{\partial^2 M(\bm{t})}{\partial t_j^2}
= \frac{1}{2} \sqrt{\frac{l}{n}} \frac{\partial}{\partial t_j} \left(\frac{1}{2n}\sum_{k=1}^n \left(\exp \left(\sqrt{\frac{n}{l}} t_k\right) + \exp \left(-\sqrt{\frac{n}{l}} t_k\right)\right)\right)^{l-1} 
\underbrace{\left(\exp \left(\sqrt{\frac{n}{l}} t_j\right) - \exp \left(-\sqrt{\frac{n}{l}} t_j\right)\right)}_{= 0 \text{ when } \bm{t}=\bm{0}} \\
& + \frac{1}{2} \sqrt{\frac{l}{n}} 
\underbrace{\left(\frac{1}{2n}\sum_{k=1}^n \left(\exp \left(\sqrt{\frac{n}{l}} t_k\right) + \exp \left(-\sqrt{\frac{n}{l}} t_k\right)\right)\right)^{l-1}}_{= 1 \text{ when } \bm{t}=\bm{0}}
\frac{\partial}{\partial t_j} \left(\exp \left(\sqrt{\frac{n}{l}} t_j\right) - \exp \left(-\sqrt{\frac{n}{l}} t_j\right)\right).
\end{split}
\]
We therefore obtain
\[
\begin{split}
\frac{\partial^2 M(\bm{t})}{\partial t_j^2} \Bigg|_{\bm{t}=\bm{0}} 
= \frac{1}{2} \sqrt{\frac{l}{n}} 
\frac{\partial}{\partial t_j} \left(\exp \left(\sqrt{\frac{n}{l}} t_j\right) - \exp \left(-\sqrt{\frac{n}{l}} t_j\right)\right) \Bigg|_{\bm{t}=\bm{0}} 
= 1
\end{split}
\]
As the covariance matrix is diagonal, we conclude from the above that the covariance matrix is an identity matrix.

The moment-generating function of the random variable $\bm{y}^T \bm{u}$, denoted by $\tilde{M}(t)$, can be obtained by substituting $\bm{t} = \bm{y} t$ into $M(\bm{t})$:
\[
\tilde{M}(t) = \left(\frac{1}{2n}\sum_{k=1}^n \left(\exp \left(\sqrt{\frac{n}{l}} y_k t\right) + \exp \left(-\sqrt{\frac{n}{l}} y_k t\right)\right)\right)^l 
= \left(\frac{1}{n} \sum_{k=1}^n \cosh \left(\sqrt{\frac{n}{l}} y_k t\right)\right)^l.
\] 
Now expanding the $\cosh$ function using Taylor series, we obtain
\[
\begin{split}
\tilde{M}(t) 
& = \left(\frac{1}{n} \sum_{k=1}^n \left(1+\frac{1}{2} \left(\sqrt{\frac{n}{l}} y_k t\right)^2 + \frac{1}{4!} \left(\sqrt{\frac{n}{l}} y_k t\right)^4 + \mathcal{O} (t^6)\right)\right)^l \\
& = \left(1 + \frac{1}{2l} \|\bm{y}\|_2^2 t^2 + \frac{n}{24 l^2} \|\bm{y}\|_4^4 t^4 + \mathcal{O} (t^6)\right)^l \\
& = l \left(\frac{n}{24 l^2} \|\bm{y}\|_4^4 t^4 \right)
+ \frac{l(l-1)}{2}\left(\frac{1}{2l} \|\bm{y}\|_2^2 t^2\right)^2 + 1 + A t^2 + \mathcal{O}(t^6).
\end{split}
\]
The fourth-order moment of $\bm{y}^T \bm{u}$ can be obtained as
\[
\mathbb{E}[|\bm{y}^T \bm{u}|] = \frac{\partial^4}{\partial t^4}\tilde{M}(t)  \Bigg|_{t=0} 
= \frac{n}{l} \|\bm{y}\|_4^4 + 3 \frac{l-1}{l} \|\bm{y}\|_2^4.
\]
\end{proof}

\section{Proof of \Cref{theorem:convergence-DES-mixture-Gaussian-l2,theorem:convergence-DES-mixture-Rademacher-l2}}
\rrr{We will require the following lemma, which can be derived from \Cref{proposition:properties-of-mixture-Gaussian-sampling,proposition:properties-of-mixture-Rademacher-sampling}. }

\rrr{
\begin{lemma} \label{lemma:bound-on-inner-product}
Let $\bm{y} \in \mathbb{R}^n$ be a vector satisfying
$\|\bm{y}\|_2^4 / \|\bm{y}\|_4^4 \ge \tilde{s}$
for some constant $s \in [1,n]$. We have
\[
	\mathbb{E}[|\bm{y}^T \bm{u}|] \ge \frac{\|\bm{y}\|_2}{\sqrt{3n/(\tilde{s}l) + 3}} \;\;\text{ for } \bm{u} \sim \mathcal{M}_l^G
\]
and 
\[
	\mathbb{E}[|\bm{y}^T \bm{u}|] \ge \frac{\|\bm{y}\|_2}{\sqrt{n/(\tilde{s}l) + 3}} \;\;\text{ for } \bm{u} \sim \mathcal{M}_l^R.
\]
\end{lemma}
}
\begin{proof}
\rrr{First, by H\"older's inequality, we have}
\begin{equation} \label{eq:holder-inequality}
	\mathbb{E}[|\bm{y}^T \bm{u}|] 
	\ge \frac{\left( \mathbb{E}[|\bm{y}^T \bm{u}|^2] \right)^{3/2}}{\left( \mathbb{E}[|\bm{y}^T \bm{u}|^4] \right)^{1/2}}
	= \frac{\|\bm{y}\|_2^3}{\left( \mathbb{E}[|\bm{y}^T \bm{u}|^4] \right)^{1/2}}.
\end{equation}
where the equality uses the fact $\mathbb{V}[\bm{u}] = \bm{I}$, according to \Cref{proposition:properties-of-mixture-Gaussian-sampling,proposition:properties-of-mixture-Rademacher-sampling}. 

Now consider the case of mixture Gaussian sampling. In this case, we have, from \cref{eq:fourth-order-moment-mixture-Gaussian}, that
\begin{equation} \label{eq:inner-product-bound-tmp1}
	\mathbb{E}[|\bm{y}^T \bm{u}|] 
	\ge \frac{\|\bm{y}\|_2^3}{\sqrt{3\left(\frac{n}{l}\|\bm{y}\|_4^4 + \frac{l-1}{l} \|\bm{y}\|_2^4\right)}}.
\end{equation}
Using \rrr{the assumption $\|\bm{y}\|_2^4 / \|\bm{y}\|_4^4 \ge \tilde{s}$} then yields
\[
	\mathbb{E}[|\bm{y}^T \bm{u}|] 
	\ge \frac{\|\bm{y}\|_2^3}{\sqrt{3\left(\frac{n}{\rrr{\tilde{s}l}}\|\bm{y}\|_2^4 + \frac{l-1}{l} \|\bm{y}\|_2^4\right)}}
	= \frac{\|\bm{y}\|_2}{\sqrt{3\left(\frac{n}{\rrr{\tilde{s}l}} + \frac{l-1}{l} \right)}}
	\ge \frac{\|\bm{y}\|_2}{\sqrt{3\left(n/\rrr{(\tilde{s}l)} + 1 \right)}}.
\]


Consider then the case of mixture Rademacher sampling. \rrr{From \cref{eq:holder-inequality}, \cref{eq:fourth-order-moment-mixture-Rademacher}, and the assumption $\|\bm{y}\|_2^4 / \|\bm{y}\|_4^4 \ge \tilde{s}$}, we have
\begin{equation} \label{eq:inner-product-bound-tmp2}
	\mathbb{E}[|\bm{y}^T \bm{u}|] 
	\ge \frac{\|\bm{y}\|_2^3}{\sqrt{\frac{n}{l}\|\bm{y}\|_4^4 + 3\frac{l-1}{l} \|\bm{y}\|_2^4}} 
	\ge \frac{\|\bm{y}\|_2^3}{\sqrt{\frac{n}{\rrr{\tilde{s}l}}\|\bm{y}\|_2^4 + 3\frac{l-1}{l} \|\bm{y}\|_2^4}}
	= \frac{\|\bm{y}\|_2}{\sqrt{\rrr{n/(\tilde{s}l)} + 3\frac{l-1}{l} }}
	\ge \frac{\|\bm{y}\|_2}{\sqrt{\rrr{n/(\tilde{s}l)} + 3}}.
\end{equation}

\end{proof}

\begin{proof}[\textbf{Proof of \Cref{theorem:convergence-DES-mixture-Gaussian-l2}}]

\rrr{Recall that the DES with mixture Gaussian sampling is a special case of \Cref{alg:UDES}, so we can reuse \Cref{lemma:bound-on-descent-step,lemma:bound-on-x-change,lemma:bound-on-deviation-from-virual-sequence,lemma:bound-A} which are derived for \Cref{alg:UDES}.

The first step in this proof is to obtain a similar bound as in \Cref{lemma:key-lemma}.}
We begin with rewriting $\mathbb{E}\left[\nabla f\left(\bm{x}_t\right)^T \bm{d}_{t+1}\right]$. For $\beta=0$, we have $\bm{z}_t = \bm{x}_t$ and
\begin{equation}\label{eq:sufficient-descent-tmp1} 
\begin{split}
\mathbb{E} & \left[\nabla f\left(\bm{x}_t\right)^T \bm{d}_{t+1}\right] \\
& = \mathbb{E}\left[\nabla f\left(\bm{x}_t\right)^T \left(\frac{1}{M}\sum_{i=1}^M \bm{v}_{i,K}^t - \bm{x}_t\right)\right] \\
& = \frac{1}{M}\sum_{i=1}^M \sum_{k=0}^{K-1} \alpha_k^t \mathbb{E}\left[\signplus\left( f_i\left(\bm{v}_{i,k}^t\right) - f_i\left(\bm{v}_{i,k}^t + \alpha_k^t \bm{u}_{i,k}^t\right)\right)\nabla f\left(\bm{x}_t\right)^T\bm{u}_{i,k}^t\right] \\
& \overset{\cref{eq:definition-sign-signplus}}= \frac{1}{2M}\sum_{i=1}^M \sum_{k=0}^{K-1} \alpha_k^t \mathbb{E}\left[\left(1+\sign\left( f_i\left(\bm{v}_{i,k}^t\right) - f_i\left(\bm{v}_{i,k}^t + \alpha_k^t \bm{u}_{i,k}^t\right)\right)\right)\nabla f\left(\bm{x}_t\right)^T\bm{u}_{i,k}^t\right] \\
& = \frac{1}{2M}\sum_{i=1}^M \sum_{k=0}^{K-1} \alpha_k^t \mathbb{E}\left[\sign\left( f_i\left(\bm{v}_{i,k}^t\right) - f_i\left(\bm{v}_{i,k}^t + \alpha_k^t \bm{u}_{i,k}^t\right)\right)\nabla f\left(\bm{x}_t\right)^T\bm{u}_{i,k}^t\right] \\
& \overset{\cref{eq:sign_identity}}= \frac{1}{2M}\sum_{i=1}^M \sum_{k=0}^{K-1} \alpha_k^t \mathbb{E}\left[\left|\nabla f\left(\bm{x}_t\right)^T\bm{u}_{i,k}^t\right|\left(-1 + 2 \mathbb{I}\left\{\sign\left( f_i\left(\bm{v}_{i,k}^t\right) - f_i\left(\bm{v}_{i,k}^t + \alpha_k^t \bm{u}_{i,k}^t\right)\right) = \sign\left(\nabla f\left(\bm{x}_t\right)^T\bm{u}_{i,k}^t\right)\right\}\right)\right] \\
& \overset{\cref{eq:bound-sum-0.5-series-2}}\le -\frac{\alpha_0^t}{2M\sqrt{K}}\sum_{i=1}^M \sum_{k=0}^{K-1}\mathbb{E}\left[\left|\nabla f\left(\bm{x}_t\right)^T\bm{u}_{i,k}^t\right|\right]  \\
& \;\;\;\;\;\;\;\;\;+ \frac{1}{M}\sum_{i=1}^M \sum_{k=0}^{K-1} \alpha_k^t \mathbb{E}\left[\underbrace{\left|\nabla f\left(\bm{x}_t\right)^T\bm{u}_{i,k}^t\right|\mathbb{I}\left\{\sign\left( f_i\left(\bm{v}_{i,k}^t\right) - f_i\left(\bm{v}_{i,k}^t + \alpha_k^t \bm{u}_{i,k}^t\right)\right) = \sign\left(\nabla f\left(\bm{x}_t\right)^T\bm{u}_{i,k}^t\right)\right\}}_{\defeq \mathfrak{A}}\right] \\
& = -\frac{\alpha_0^t}{2M\sqrt{K}}\sum_{i=1}^M \sum_{k=0}^{K-1}\mathbb{E}\left[\rrr{\left|\nabla f\left(\bm{x}_t\right)^T\bm{u}_{i,k}^t\right|}\right] + \frac{1}{M}\sum_{i=1}^M \sum_{k=0}^{K-1} \alpha_k^t \mathbb{E}\left[\mathfrak{A}\right]
\end{split}
\end{equation}
\rrr{Using the assumption $\|\nabla f(\bm{x})\|_2^4 / \|\nabla f(\bm{x})\|_4^4 \ge \tilde{s}$ and \Cref{lemma:bound-on-inner-product}}, we have
\begin{equation*}
	\mathbb{E} \left[\nabla f\left(\bm{x}_t\right)^T \bm{d}_{t+1}\right] \\
 \le -\frac{\alpha_0^t\sqrt{K}}{2\rrr{V}} \mathbb{E}\left[\left\|\nabla f\left(\bm{x}_t\right)\right\|_2\right] + \frac{1}{M}\sum_{i=1}^M \sum_{k=0}^{K-1} \alpha_k^t \mathbb{E}\left[\mathfrak{A}\right]
\end{equation*}
\rrr{where $V$ is a constant that can be set to 
\begin{equation} \label{eq:choice-of-V-in-mixture-Gaussian-distribution}
	V = \sqrt{3 + 3 n / (\tilde{s}l)}.
\end{equation}
}

\rrr{Note that \Cref{lemma:bound-A} gives an upper bound for the term $\mathbb{E}[\mathfrak{A}]$. We therefore have}
\[
\begin{split}
\mathbb{E}  \left[\nabla f\left(\bm{x}_t\right)^T \bm{d}_{t+1}\right] & + \frac{\alpha_0^t\sqrt{K}}{2\rrr{V}}\mathbb{E}\left[\left\|\nabla f\left(\bm{x}_t\right)\right\|_2\right]\\
& \le \frac{1}{M}\sum_{i=1}^M \sum_{k=0}^{K-1} \alpha_k^t \left(\frac{\alpha_k^tL + \omega_1 + \omega_2}{2} \mathbb{E}\left[\| \bm{u}_{i,k}^t\|^2\right]
	+ \frac{L^2}{2\omega_1} \mathbb{E}\left[\|\bm{v}_{i,k}^t - \bm{z}_t\|^2\right]
	+\frac{\sigma^2}{2\omega_2b}\right)\\ 
& \le \frac{1}{M}\sum_{i=1}^M \sum_{k=0}^{K-1} \alpha_k^t \left(\frac{\alpha_k^tL + \omega_1 + \omega_2}{2} U
	+ \frac{L^2}{2\omega_1} \mathbb{E}\left[\|\bm{v}_{i,k}^t - \bm{z}_t\|^2\right]
	+\frac{\sigma^2}{2\omega_2b}\right).
\end{split}	
\]

Now use \Cref{lemma:bound-on-deviation-from-virual-sequence} to bound $\mathbb{E}\left[\|\bm{v}_{i,k}^t - \bm{z}_t\|^2\right]$ and use the setting $\beta=0$:
\[
\begin{split}
	\mathbb{E} \left[\nabla f\left(\bm{x}_t\right)^T \bm{d}_{t+1}\right] & + \frac{\alpha_0^t\sqrt{K}}{2\rrr{V}}\mathbb{E}\left[\left\|\nabla f\left(\bm{x}_t\right)\right\|_2\right]\\
	& \overset{\cref{eq:client-drift-bound},\beta=0}\le \sum_{k=0}^{K-1} \alpha_k^t \left(\frac{\alpha_k^tL + \omega_1 + \omega_2}{2} U
	+ \frac{L^2}{\omega_1} UK(1+\log K) (\alpha_0^t)^2
	+\frac{\sigma^2}{2\omega_2b}\right) \\
	& = \frac{LU}{2}\sum_{k=0}^{K-1} (\alpha_k^t)^2 
	+\left(\frac{L^2}{\omega_1} UK(1+\log K) (\alpha_0^t)^2 + \frac{\omega_1+\omega_2}{2}U + \frac{\sigma^2}{2\omega_2b}\right)\sum_{k=0}^{K-1} \alpha_k^t.
\end{split}
\]
Letting $\omega_1 = \frac{L\alpha_0^t}{\sqrt{K}}, \omega_2 = \frac{\sigma}{\sqrt{Ub}}$ yields
\begin{equation} \label{eq:inner-product-bound-mixture-tmp}
\begin{split}
	\mathbb{E} \left[\nabla f\left(\bm{x}_t\right)^T \bm{d}_{t+1}\right] & + \frac{\alpha_0^t\sqrt{K}}{2\rrr{V}}\mathbb{E}\left[\left\|\nabla f\left(\bm{x}_t\right)\right\|_2\right]\\
	& = \frac{LU}{2}\sum_{k=0}^{K-1} (\alpha_k^t)^2 
	+\left(LU \left(\sqrt{K}(1+\log K)  + \frac{1}{2\sqrt{K}}\right)\alpha_0^t + \frac{\sqrt{U}\sigma}{\sqrt{b}}\right)\sum_{k=0}^{K-1} \alpha_k^t \\
	& \overset{(\ref{eq:bound-sum-1-series},\ref{eq:bound-sum-0.5-series})}\le \frac{LU}{2} (1+\log K) (\alpha_0^t)^2
	+\left(LU \left(\sqrt{K}(1+\log K)  + \frac{1}{2\sqrt{K}}\right)\alpha_0^t + \frac{\sqrt{U}\sigma}{\sqrt{b}}\right) 2\sqrt{K} \alpha_0^t \\
	& = \sqrt{K}LU\left(\left(\frac{1}{2\sqrt{K}} + 2\sqrt{K}\right)(1+\log K) + \frac{1}{\sqrt{K}}\right) (\alpha_0^t)^2
	+ \frac{\sqrt{U}\sigma}{\sqrt{b}}2\sqrt{K} \alpha_0^t.
\end{split}
\end{equation}
On the other hand, by the smoothness assumption, we have 
\[
\begin{split}
\mathbb{E} & \left[ f\left(\bm{x}_{t+1}\right) - f\left(\bm{x}_t\right)\right]
\le \mathbb{E}\left[\nabla f(\bm{x}_t)^T \bm{d}_{t+1}\right] + \frac{L}{2}\mathbb{E}\left[\|\bm{d}_{t+1}\|^2\right] \\
& \overset{\cref{eq:inner-product-bound-mixture-tmp}}\le -\frac{\alpha_0^t\sqrt{K}}{2\rrr{V}}\mathbb{E}\left[\left\|\nabla f\left(\bm{x}_t\right)\right\|_2\right] 
	+ \sqrt{K}LU \left(\left(\frac{1}{2\sqrt{K}} + 2\sqrt{K}\right)(1+\log K) + \frac{1}{\sqrt{K}}\right) (\alpha_0^t)^2
	+ \frac{\sqrt{U}\sigma}{\sqrt{b}}2\sqrt{K} \alpha_0^t + \frac{L}{2}\mathbb{E}\left[\|\bm{d}_{t+1}\|^2\right] \\
& \overset{\cref{eq:descent-step-bound-square}}\le -\frac{\alpha_0^t\sqrt{K}}{2\rrr{V}}\mathbb{E}\left[\left\|\nabla f\left(\bm{x}_t\right)\right\|_2\right] 
	+ \sqrt{K}LU \underbrace{\left(\left(\frac{1}{2\sqrt{K}} + \frac{5}{2}\sqrt{K}\right)(1+\log K) + \frac{1}{\sqrt{K}} \right)}_{\defeq \hat{\Psi}} (\alpha_0^t)^2
	+ \frac{\sqrt{U}\sigma}{\sqrt{b}}2\sqrt{K} \alpha_0^t, 
\end{split}
\]
\rrr{where in the last step we have reused the bound in \Cref{lemma:bound-on-descent-step}.
Summing the above up for $t = 0,\cdots,T-1$ gives}
\[
\begin{split}
\sum_{t=0}^{T-1} \alpha_0^t \frac{\mathbb{E}\left[\left\|\nabla f\left(\bm{x}_t\right)\right\|_2\right]}{\rrr{V}} 
& \le 2\frac{f(\bm{x}_0) - f_*}{\sqrt{K}}
	+ 4\frac{\sqrt{U}\sigma}{\sqrt{b}} \sum_{t=0}^{T-1}\alpha_0^t 
	+ 2LU \hat{\Psi} \sum_{t=0}^{T-1}(\alpha_0^t)^2 \\
& \overset{(\ref{eq:bound-sum-0.5-series},\ref{eq:bound-sum-0.25-series})}\le 2\frac{f(\bm{x}_0) - f_*}{\sqrt{K}}
	+ \frac{16}{3}\frac{\sqrt{U}\sigma}{\sqrt{b}}\alpha T^{\frac{3}{4}}
	+ 4LU \hat{\Psi} \alpha^2 \sqrt{T} \\
& \overset{b \ge \sqrt{T}}\le 2\frac{f(\bm{x}_0) - f_*}{\sqrt{K}}
	+ \frac{16}{3}\sqrt{U}\sigma\alpha \sqrt{T}
	+ 4LU \hat{\Psi} \alpha^2 \sqrt{T}.
\end{split}
\]
The left-hand side is bounded from below as
\[
	\sum_{t=0}^{T-1} \alpha_0^t \frac{\mathbb{E}\left[\left\|\nabla f\left(\bm{x}_t\right)\right\|_2\right]}{\rrr{V}} 
	\rrr{\overset{(\ref{eq:bound-sum-0.25-series-2})}\ge} \alpha T^{\frac{3}{4}} \frac{1}{T}\sum_{t=0}^{T-1} \frac{\mathbb{E}\left[\left\|\nabla f\left(\bm{x}_t\right)\right\|_2\right]}{\rrr{V}}.
\]
We therefore have
\[
	\frac{1}{T}\sum_{t=0}^{T-1} \frac{\mathbb{E}\left[\left\|\nabla f\left(\bm{x}_t\right)\right\|_2\right]}{\rrr{V}}
	\le \frac{2}{T^{\frac{3}{4}}}\frac{f(\bm{x}_0) - f_*}{\alpha\sqrt{K}}
	+ \left(\frac{16}{3}\sigma + 4L\sqrt{U} \hat{\Psi} \alpha\right) \frac{\sqrt{U}}{T^{\frac{1}{4}}}.
\]

\rrr{The final step is to specify the value of $U$ which is an upper bound of $\mathbb{E}[\|\bm{u}_{i,k}^t\|^2]$. 
For any $\bm{u}_{i,k}^t$ drawn from $\mathcal{M}_l^G$, we know from \Cref{proposition:properties-of-mixture-Gaussian-sampling} that it has an identical covariance matrix and all its coordinates are independently distributed. It means \Cref{lemma:bound-on-variance} can be used here. In particular, since we are considering the setting $\|\cdot\| = \|\cdot\|_2$, we can choose $U=n$. Substituting the value of $V$ in \cref{eq:choice-of-V-in-mixture-Gaussian-distribution} into the above inequality completes the proof.}

\end{proof}

\begin{proof}[Proof of \Cref{theorem:convergence-DES-mixture-Rademacher-l2}]
The proof is almost identical to that of \Cref{theorem:convergence-DES-mixture-Gaussian-l2}, since by \Cref{proposition:properties-of-mixture-Gaussian-sampling,proposition:properties-of-mixture-Rademacher-sampling} the two mixture sampling schemes only differ in the fourth-order moment, which is used in bounding 
\[
	\rrr{\mathbb{E}\left[\left|\nabla f(\bm{x}_t)^T \bm{u}_{i,k}^t\right|\right]}
\]
in \cref{eq:sufficient-descent-tmp1}. 
Note that by \Cref{proposition:properties-of-mixture-Rademacher-sampling} the above can be lower bounded by
$\frac{\|\nabla f(\bm{x}_t)\|_2}{\sqrt{n/(\tilde{s}l)+3}}$. Therefore, we can simply replace \rrr{the value of $V$ in \cref{eq:choice-of-V-in-mixture-Gaussian-distribution} by
\[ V_t = \sqrt{3+n/(\tilde{s}l)} \]}
and we will get the final bound.

\end{proof}

\section{Auxiliary Lemmas}
\begin{lemma} \label{lemma:bound-on-variance}
Let $\|\cdot\|$ be a vector norm in $\mathbb{R}^n$.
\rrr{Let $\bm{u} \in \mathbb{R}^n$ be any random vector satisfying $\mathbb{E}[\bm{u}] = \bm{0}$ and $\mathbb{V}[\bm{u}] = \bm{I}$. Assume all coordinates of $\bm{u}$ are distributed independently.
Then, there exists a constant $U >0$ such that $\mathbb{E}\left[\|\bm{u}\|^2\right] \le U$.}
In particular, we can choose $U=n$ for $\|\cdot\| = \|\cdot\|_2$ and $U = 4 \log \left(\sqrt{2}n\right)$ for $\|\cdot\| = \|\cdot\|_\infty$.
\end{lemma}
\begin{proof}
\rrr{First, by the identity covariance matrix assumption, we have $\mathbb{E}[\|\bm{u}\|_2^2] = \text{Tr}[\mathbb{E}[\bm{u}\bm{u}^T]]  = \text{Tr}\left[\mathbb{V}[\bm{u}]\right] = \text{Tr}[\bm{I}]= n$, where $\text{Tr}[\cdot]$ denotes the matrix trace. So this suggests the $\ell_2$ norm of $\bm{u}$ can be bounded by $U\defeq n$.}
Then, due to the equivalence of vector norm, we know such a bound exists for all norms. 
In the next, we study the case of $\ell_\infty$ norm.


Let $t \in (0,1/2)$ be a constant. By the convexity of $\|\cdot\|_\infty$, we have
\[
	\exp\left(t \mathbb{E}\left[\|\bm{u}\|_\infty^2\right]\right) 
	\le \mathbb{E}\left[\exp\left(t\left\|\bm{u}\right\|_\infty^2\right)\right]
	= \mathbb{E}\left[\exp\left(t \rrr{\max_{1 \le i \le n}} u_i^2 \right)\right]
	\le \sum_{i=1}^n \mathbb{E}\left[\exp\left(t u_i^2 \right)\right]
	= n \mathbb{E}\left[\exp\left(t u_1^2 \right)\right],
\]
where the last equation is due to that all elements in $\bm{u}$ are independently distributed.

The rightmost expectation can be calculated explicitly as
\[
	\mathbb{E}\left[\exp\left(t u_1^2\right)\right] 
	= \frac{1}{\sqrt{2\pi}}\int \exp\left(t u_1^2\right) \exp\left(-\frac{1}{2}u_1^2\right) \text{d} u_1
	= \frac{1}{\sqrt{2\pi}}\int \exp\left(-\frac{1-2t}{2}u_1^2\right) \text{d} u_1 = \frac{1}{\sqrt{1-2t}}.
\]
It follows that
\[
	\exp\left(t \mathbb{E}\left[\|\bm{u}\|_\infty^2\right]\right) 
	\le \frac{n}{\sqrt{1-2t}}
\]
and therefore
\[
	\mathbb{E}\left[\|\bm{u}\|_\infty^2\right] \le \frac{1}{t}\log\frac{n}{\sqrt{1-2t}}.
\]
Note that this inequality holds for any $t \in (0,1/2)$. So we can choose $t=1/4$ and then the desired bound $\mathbb{E}\left[\|\bm{u}\|^2\right] \le 4 \log \left(\sqrt{2}n\right)$ follows.
\end{proof}

\begin{lemma} \label{lemma:bound-partial-sum-p-series}
For $J \in \mathbb{Z}_+$ we have the following properties for the partial sum of $p$-series with $p=1,0.5,$ or $0.25$:
\begin{align}
	& \sum_{j=1}^J \frac{1}{j} \le 1 + \log J \label{eq:bound-sum-1-series}\\
	& \rrr{\sqrt{J} \le} \sum_{j=1}^J \frac{1}{j^{0.5}} \le 2\sqrt{J} \label{eq:bound-sum-0.5-series}\\
	& \sum_{j=1}^J \frac{a_j}{ j^{0.5}} \ge  \sqrt{J} \left(\frac{1}{J} \sum_{j=1}^J a_j\right), \forall a_j \ge 0 \label{eq:bound-sum-0.5-series-2} \\
	& \sum_{j=1}^J \frac{1}{j^{0.25}} \le  \frac{4}{3}J^{3/4} \label{eq:bound-sum-0.25-series} \\
	& \sum_{j=1}^J \frac{a_j}{ j^{0.25}} \ge  J^{3/4} \left(\frac{1}{J} \sum_{j=1}^J a_j\right), \forall a_j \ge 0 \label{eq:bound-sum-0.25-series-2}
\end{align}
\end{lemma}

\begin{proof}
\Cref{eq:bound-sum-0.25-series-2,eq:bound-sum-0.5-series-2} are trivial.
For \cref{eq:bound-sum-1-series,eq:bound-sum-0.5-series,eq:bound-sum-0.25-series} see \cite[Section 4.1]{chlebus_approximate_2009}.
\end{proof}

\begin{lemma}
For $\beta < \sqrt{\frac{1}{2\sqrt{2}}}$  , we have the following bounds
\begin{align}
	\label{eq:beta-series-bound-1}\sum_{j=1}^t \frac{\beta^{t-j}}{j^{0.25}} 
		& = \frac{\beta^{t-1}}{1^{0.25}} + \cdots + \frac{\beta^0}{t^{0.25}} 
		\le \frac{20}{t^{0.25}}, \\
	\label{eq:beta-series-bound-2}\sum_{j=1}^t \frac{\left(2\beta^2\right)^{t-j}}{\sqrt{j}} 
		& = \frac{\left(2\beta^2\right)^{t-1}}{\sqrt{1}} + \cdots + \frac{\left(2\beta^2\right)^0}{\sqrt{t}} 
		\le \frac{1}{\sqrt{t}\left(1-2\sqrt{2}\beta^2\right)}
\end{align}
\end{lemma}
\begin{proof}
Both bounds hold trivially for $t=1$, so we prove them with induction.

First let $R_t = \sum_{j=1}^t \frac{\beta^{t-j}}{j^{0.25}}$.  Then $R_{t+1}$ can be expressed as
\begin{equation*}
	\begin{split}
	R_{t+1} & = \frac{\beta^t}{1^{0.25}} + \cdots + \frac{\beta^1}{t^{0.25}} + \frac{\beta^0}{(t+1)^{0.25}} \\
	& = \beta R_t + \frac{1}{(t+1)^{0.25}} \\
	& \le \frac{20\beta}{t^{0.25}} + \frac{1}{(t+1)^{0.25}} \\
	& = \frac{20\beta \left(1+\frac{1}{t}\right)^{0.25} + 1}{(t+1)^{0.25}}
	\end{split}
\end{equation*}
For $t \ge 2$ and $\beta < \sqrt{\frac{1}{2\sqrt{2}}}$, we get
\[
	R_{t+1} \le \frac{20\beta \times 1.5^{0.25} + 1}{(t+1)^{0.25}} \approx \frac{22.13 \times \beta + 1}{(t+1)^{0.25}}
	\lessapprox \frac{14.16}{(t+1)^{0.25}} \le \frac{20}{(t+1)^{0.25}}.
\]

To prove the second bound, we define $P_t = \sum_{j=1}^t \frac{\left(2\beta^2\right)^{t-j}}{\sqrt{j}}$ and $\delta = \frac{1}{2\sqrt{2}} - \beta^2$. The right-hand side of \cref{eq:beta-series-bound-2} then becomes $\frac{1}{2\sqrt{2t}\delta}$.
\[
\begin{split}
	P_{t+1} & = \frac{\left(2\beta^2\right)^t}{\sqrt{1}} + \cdots + \frac{\left(2\beta^2\right)^0}{\sqrt{t+1}} \\
	& = 2\beta^2 P_t + \frac{1}{\sqrt{t+1}} \\
	& \le \frac{\beta^2}{\sqrt{2t}\delta} + \frac{1}{\sqrt{t+1}} \\
	& = \frac{\frac{1}{2\sqrt{2}}-\delta}{\sqrt{2t}\delta} + \frac{1}{\sqrt{t+1}} \\
	& = \frac{1}{\sqrt{t+1}}\left(\sqrt{1+\frac{1}{t}}\frac{\frac{1}{2\sqrt{2}}-\delta}{\sqrt{2}\delta} +1\right)
\end{split}
\]
For $t \ge 1$, we have 
\[
	P_{t+1} \le \frac{1}{\sqrt{t+1}}\left(\frac{\frac{1}{2\sqrt{2}}-\delta}{\delta} +1\right) 
	= \frac{1}{\sqrt{t+1}} \frac{1}{2\sqrt{2}\delta}.
\]
Substituting $\delta$ then gives the bound \cref{eq:beta-series-bound-2}.
\end{proof}

\section{Additional Experimental Results}
\Cref{fig:testing-curve-part-one,fig:testing-curve-part-two} report correspondingly the generalization performance of all considered algorithms. In most cases, the test error of the DES methods decreases monotonically, suggesting that overfitting is not occurring. The generalization performance of the DES methods are clearly very good, being consistent with their training performance.

\begin{figure*}[htb]
\centering
\subfloat{\includegraphics[width=0.6\textwidth]{figs/legend_curve}} \\[-2ex]
\addtocounter{subfigure}{-1}
\subfloat[LR, rcv1]{\includegraphics[width=0.33\textwidth]{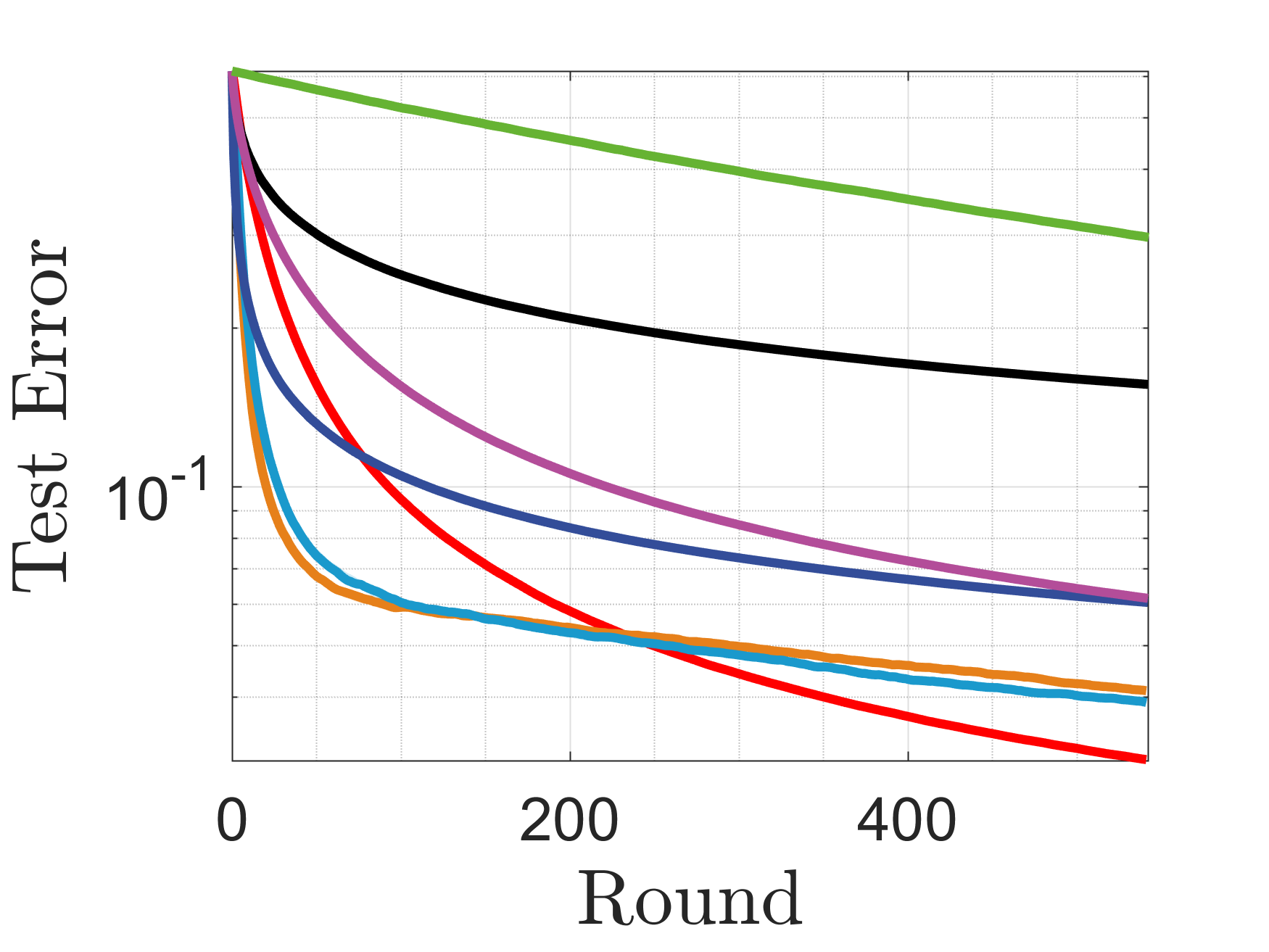}}
\subfloat[NSVM, rcv1]{\includegraphics[width=0.33\textwidth]{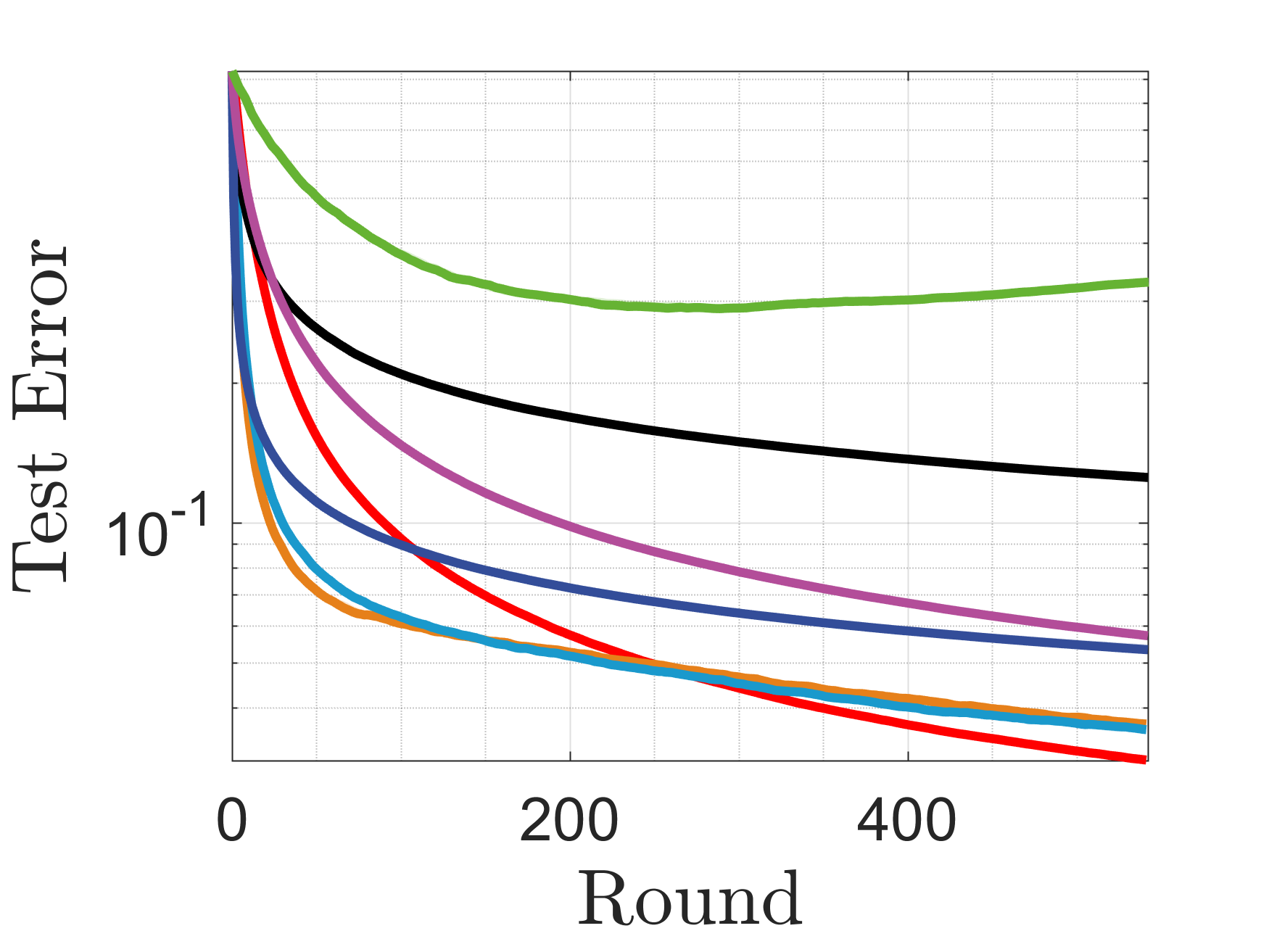}}
\subfloat[LSVM, rcv1]{\includegraphics[width=0.33\textwidth]{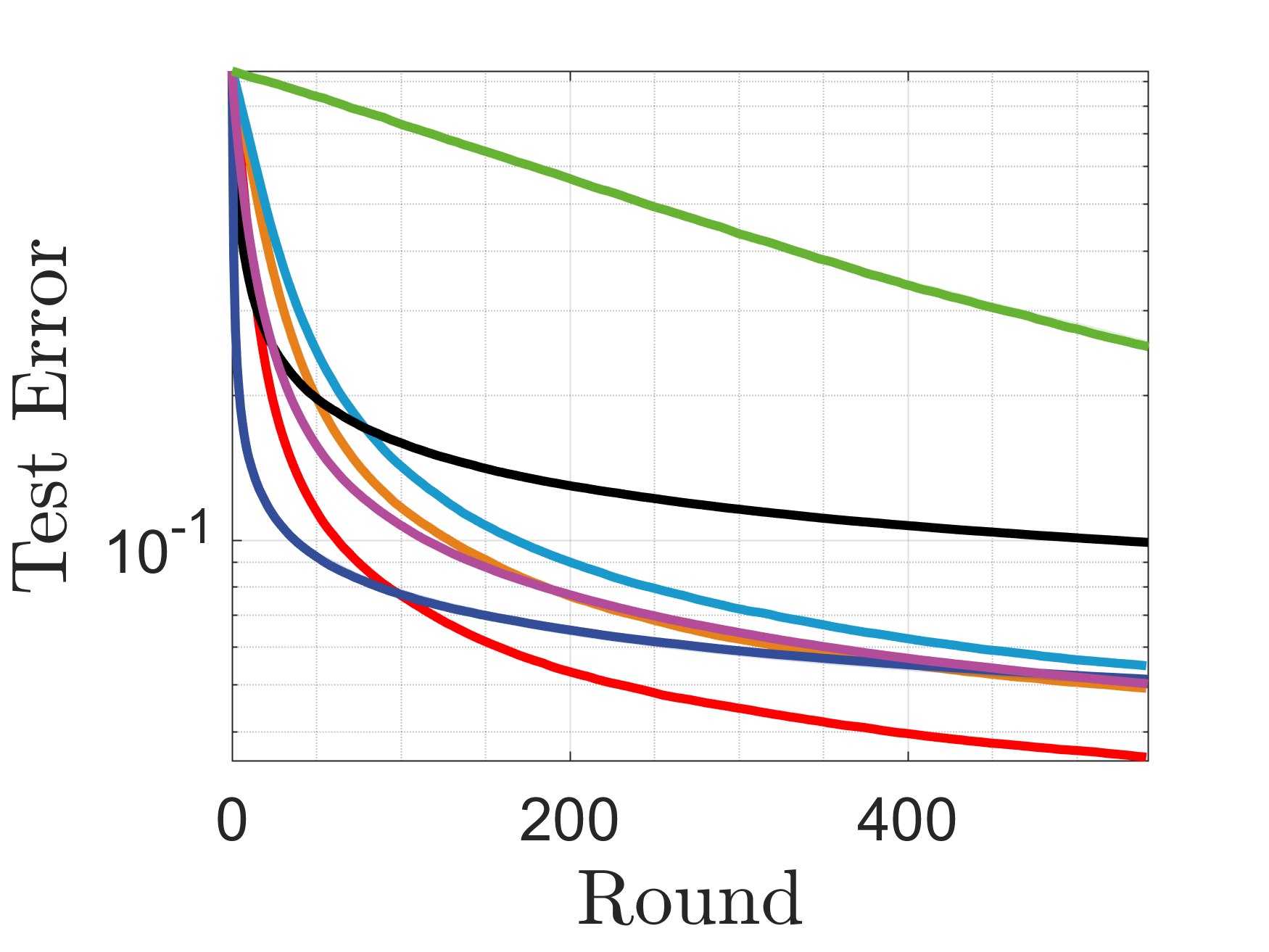}}
\hfil
\subfloat[LR, SUSY]{\includegraphics[width=0.33\textwidth]{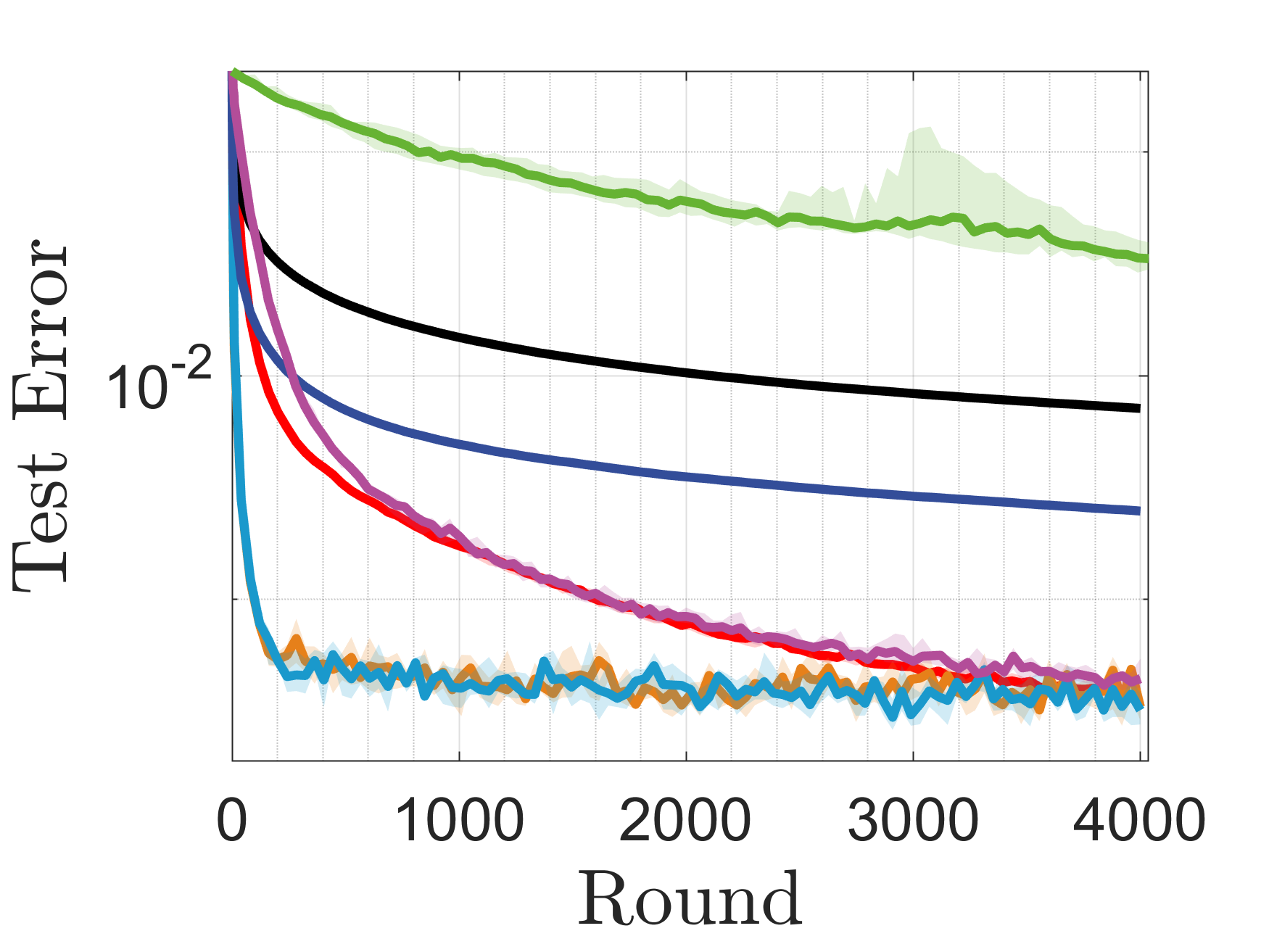}}
\subfloat[NSVM, SUSY]{\includegraphics[width=0.33\textwidth]{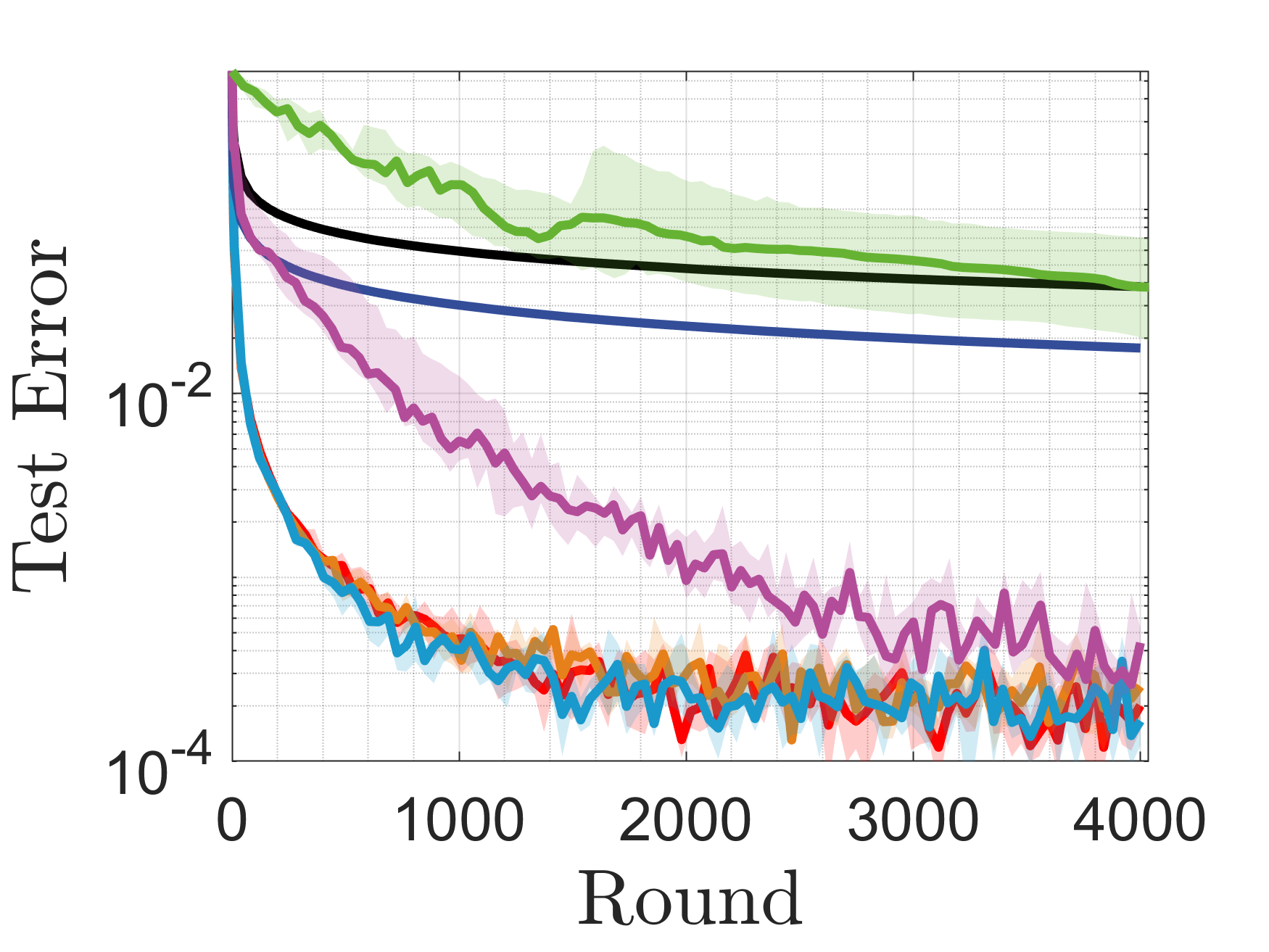}}
\subfloat[LSVM, SUSY]{\includegraphics[width=0.33\textwidth]{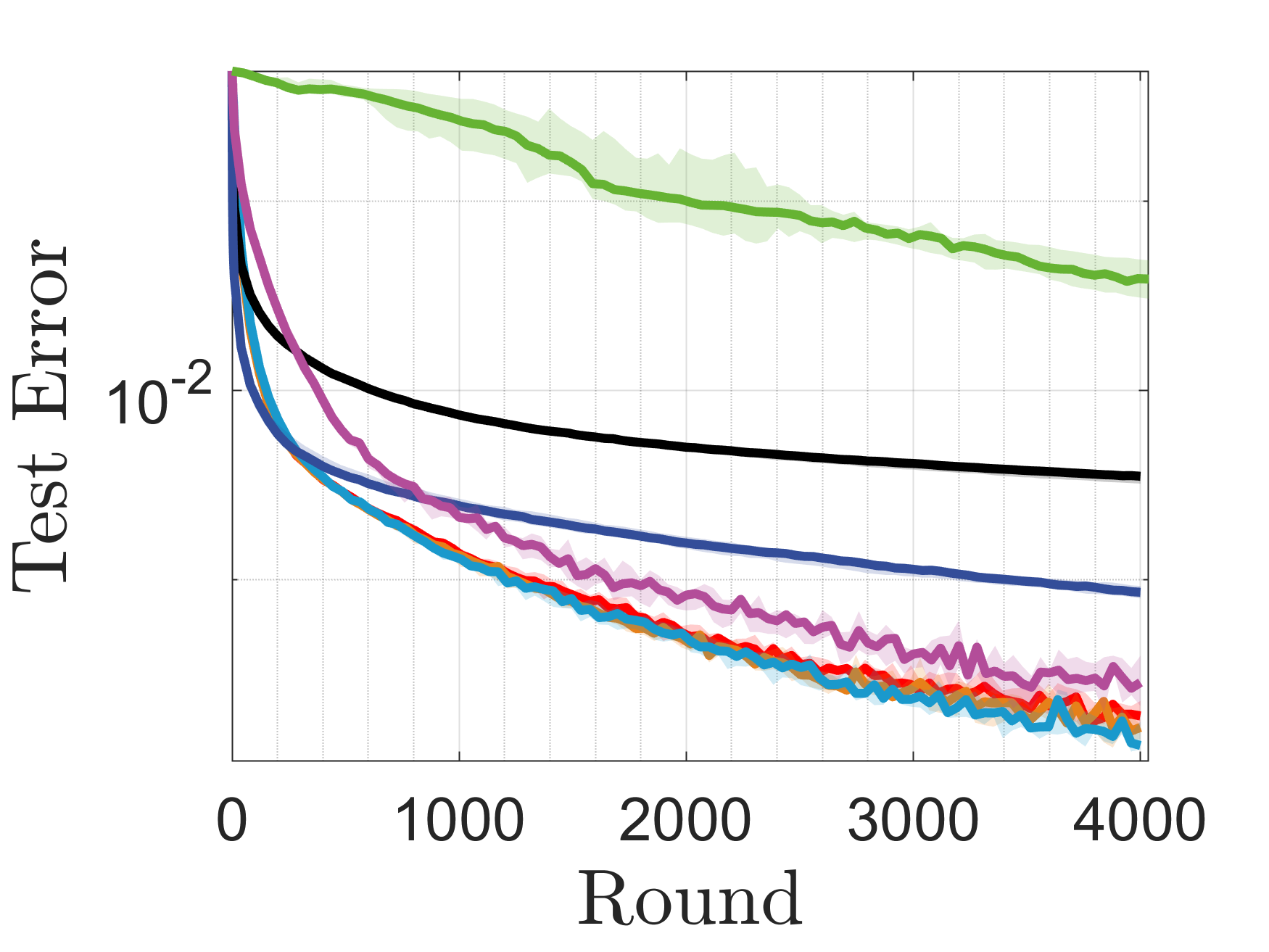}}
\hfil
\subfloat[LR, mnist]{\includegraphics[width=0.33\textwidth]{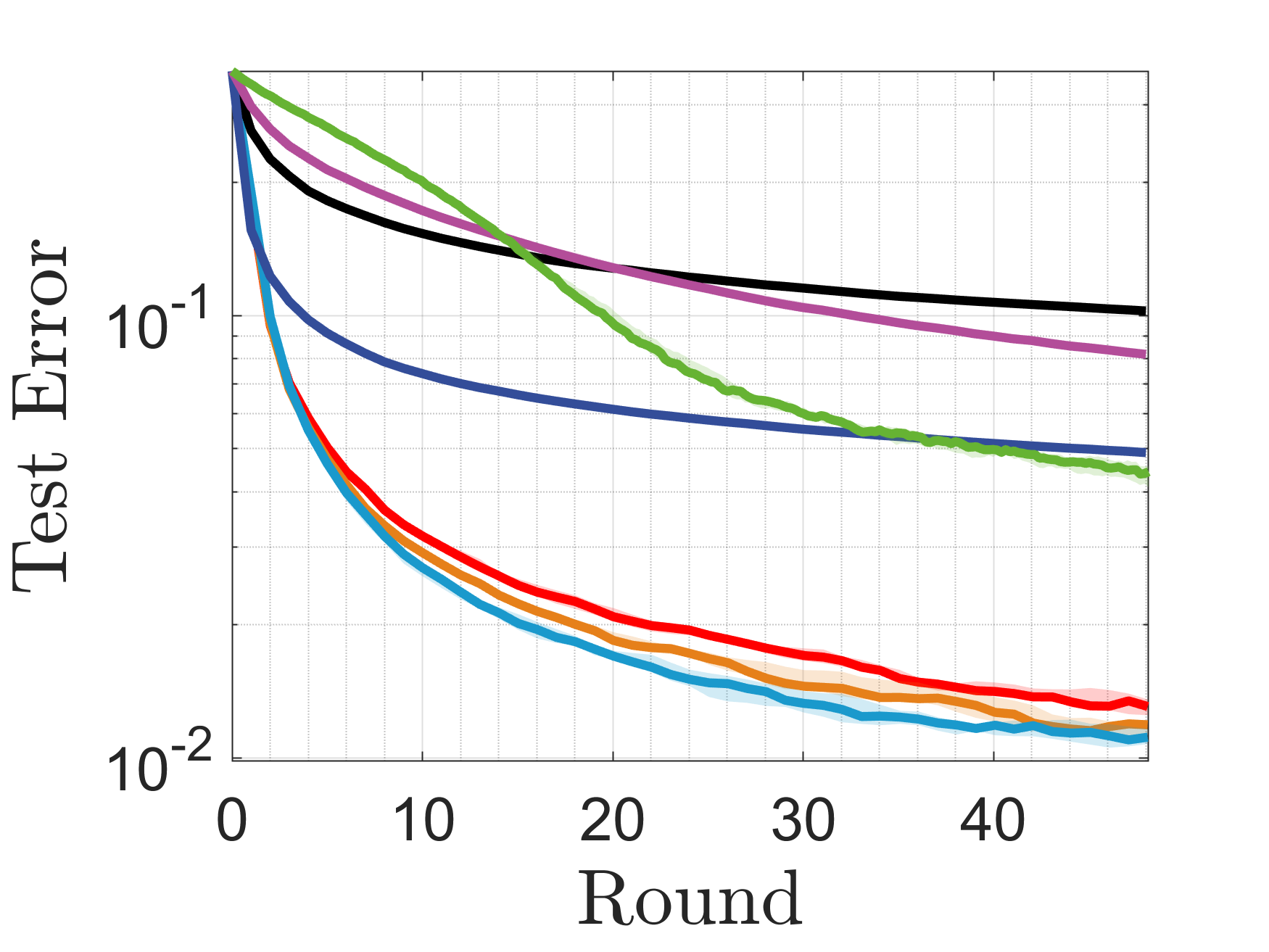}}
\subfloat[NSVM, mnist]{\includegraphics[width=0.33\textwidth]{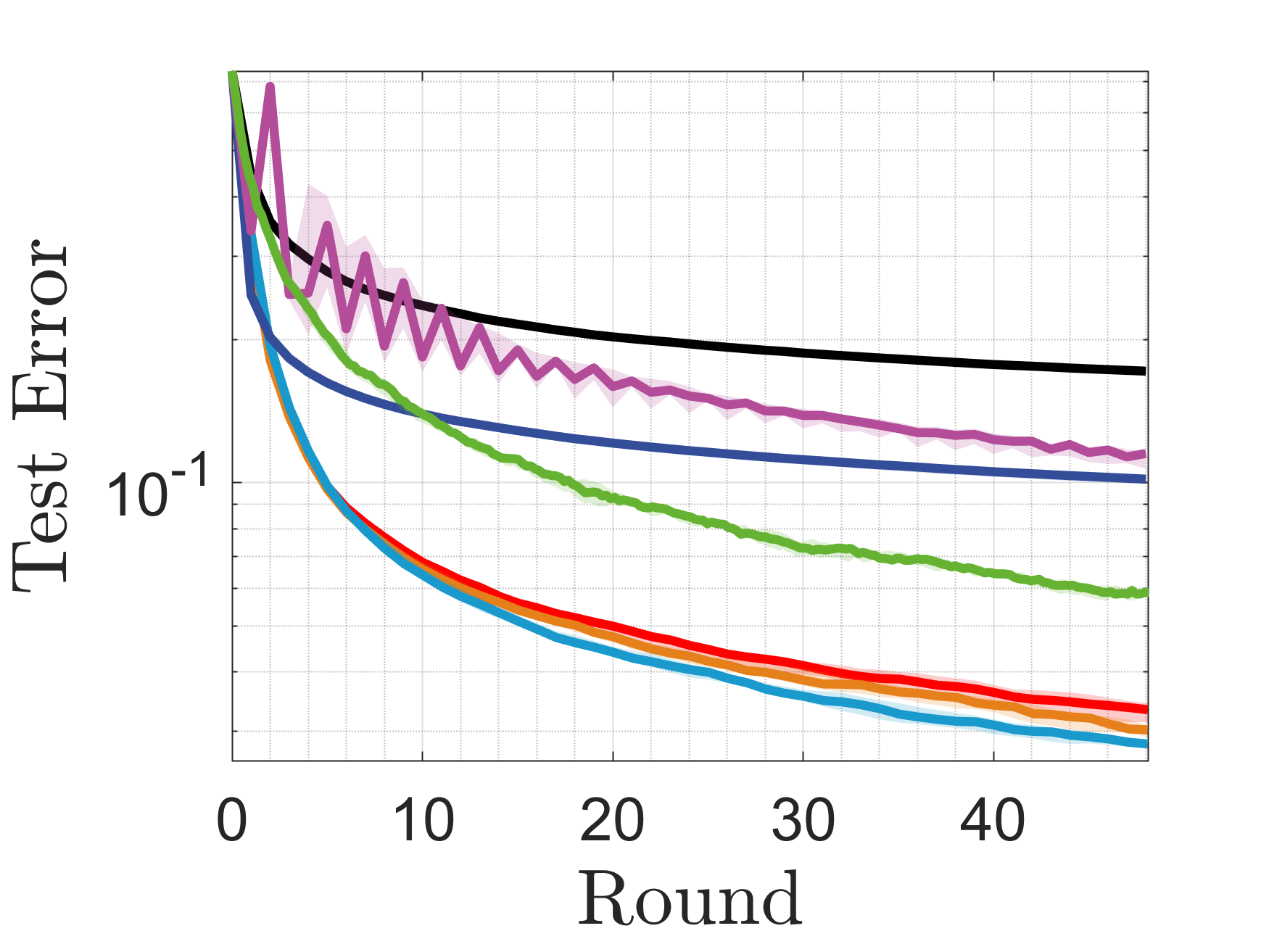}}
\subfloat[LSVM, mnist]{\includegraphics[width=0.33\textwidth]{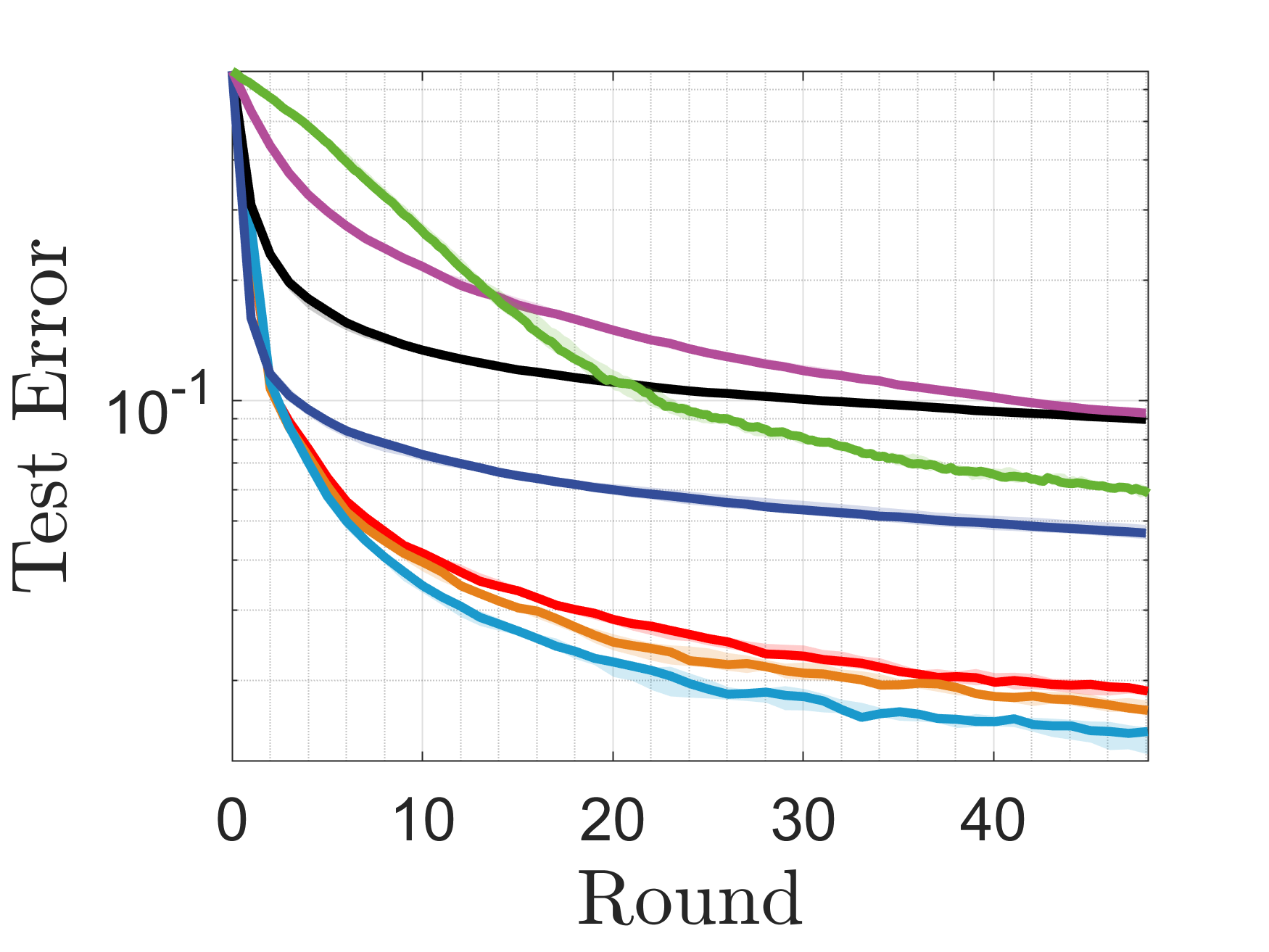}}
\caption{Generalization performance on rcv1, SUSY, and mnist datasets. The curve displays the test error versus the number of rounds and the corresponding shaded area extends from the 25th to 75th percentiles over the results obtained from all independent runs.}
\label{fig:testing-curve-part-one}
\end{figure*}

\begin{figure*}[htb]
\centering
\subfloat{\includegraphics[width=0.6\textwidth]{figs/legend_curve}} \\[-2ex]
\addtocounter{subfigure}{-1}
\subfloat[LR, real-sim]{\includegraphics[width=0.33\textwidth]{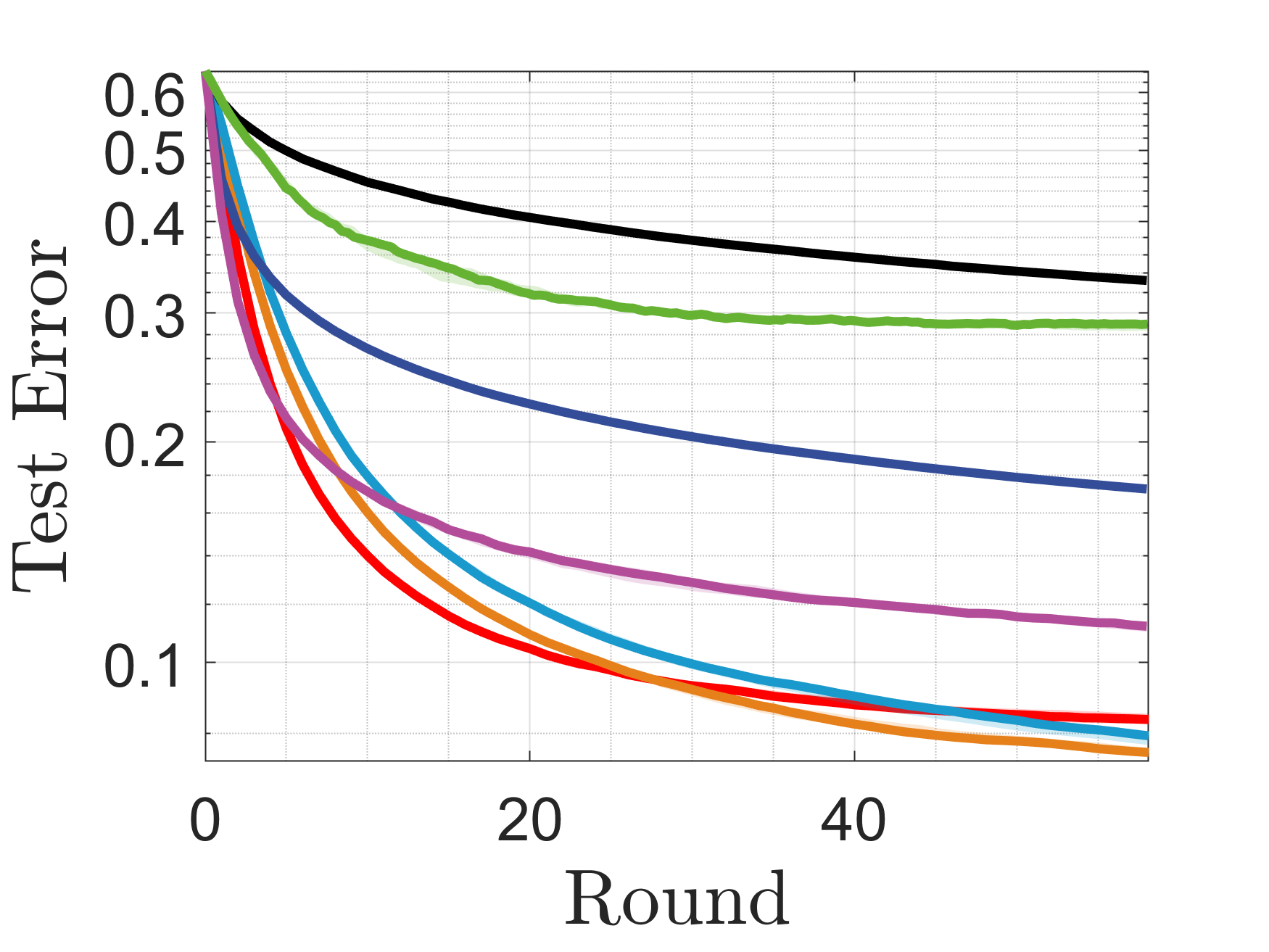}}
\subfloat[NSVM, real-sim]{\includegraphics[width=0.33\textwidth]{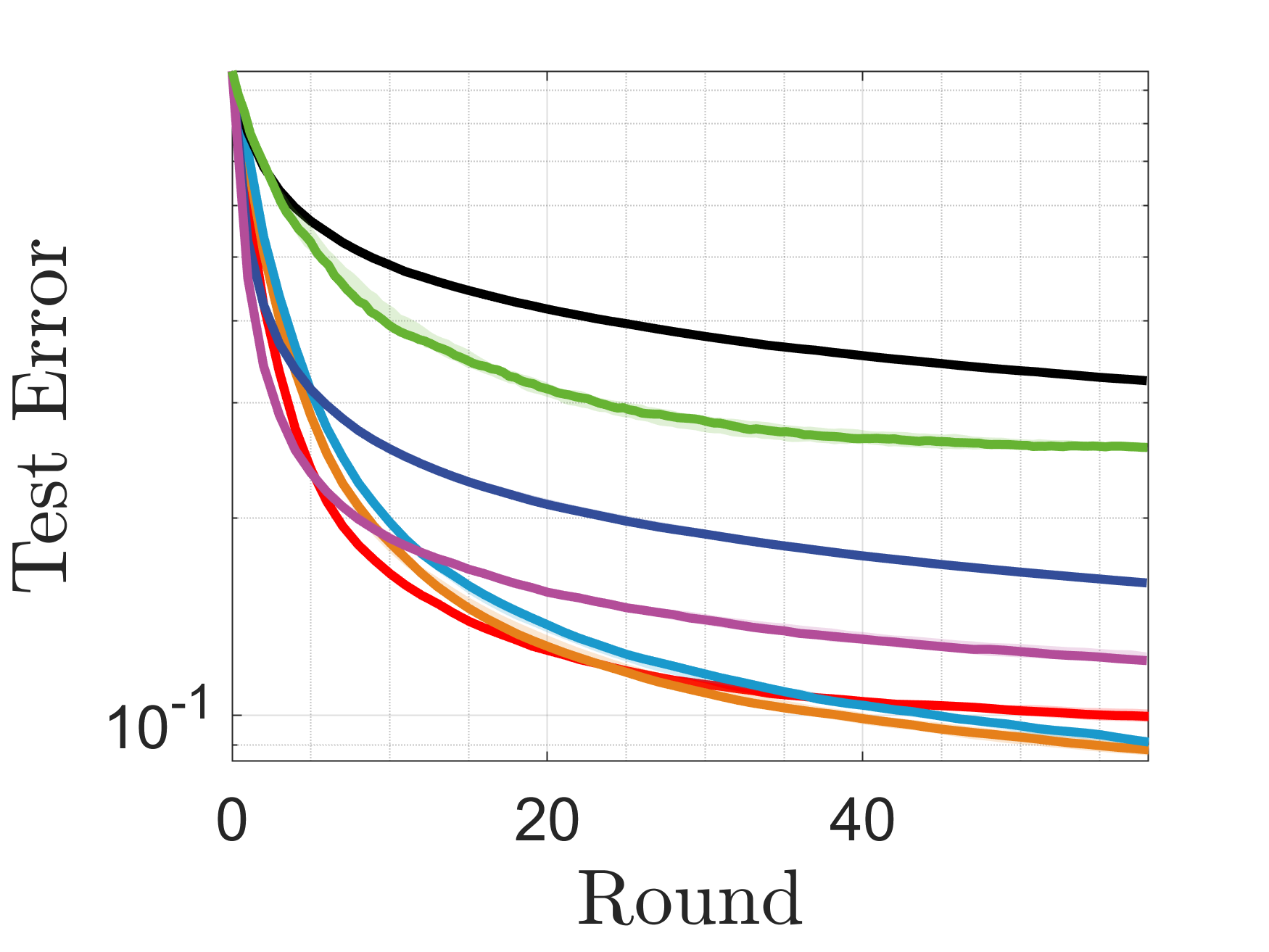}}
\subfloat[LSVM, real-sim]{\includegraphics[width=0.33\textwidth]{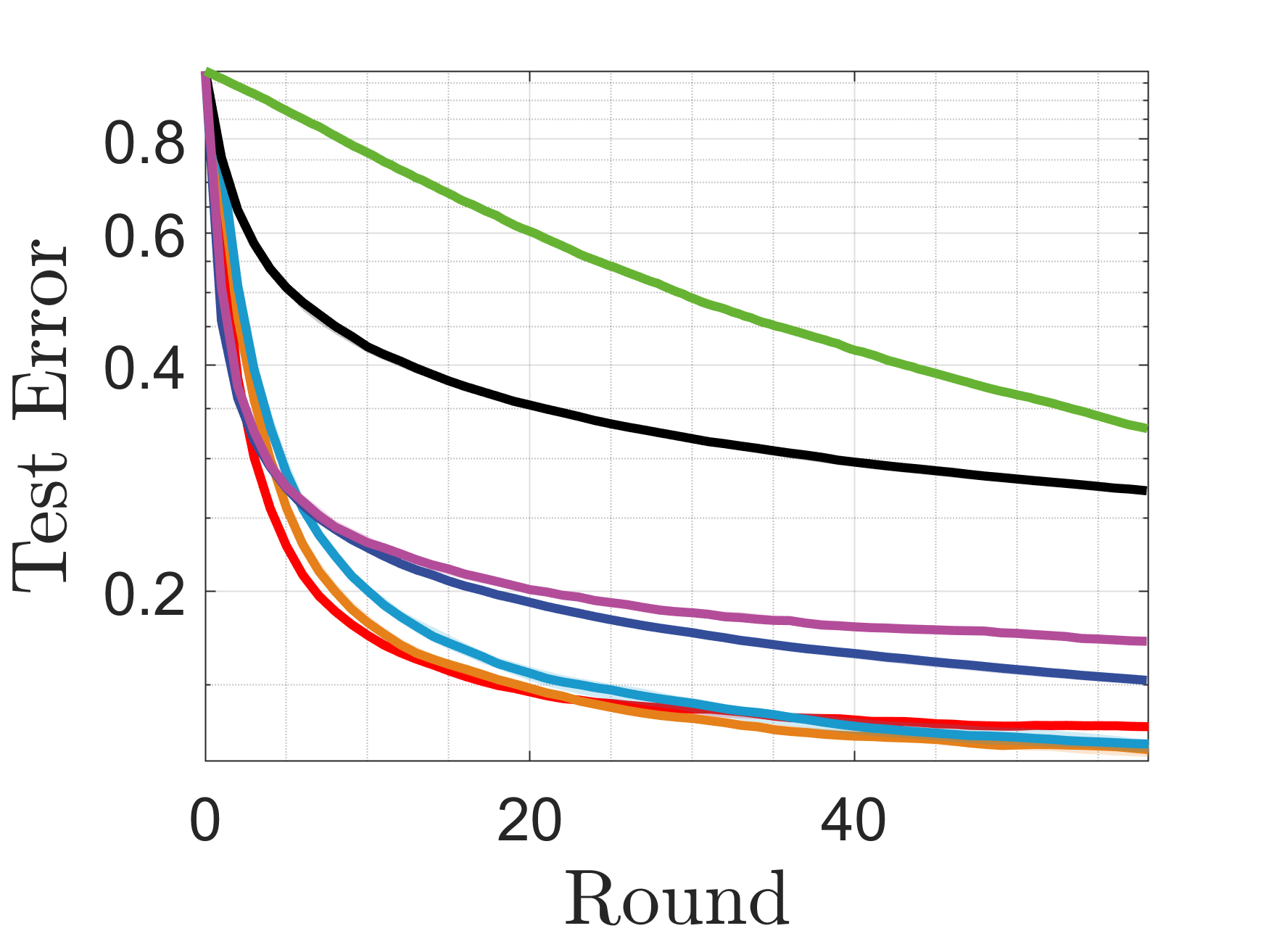}}
\hfil
\subfloat[LR, ijcnn1]{\includegraphics[width=0.33\textwidth]{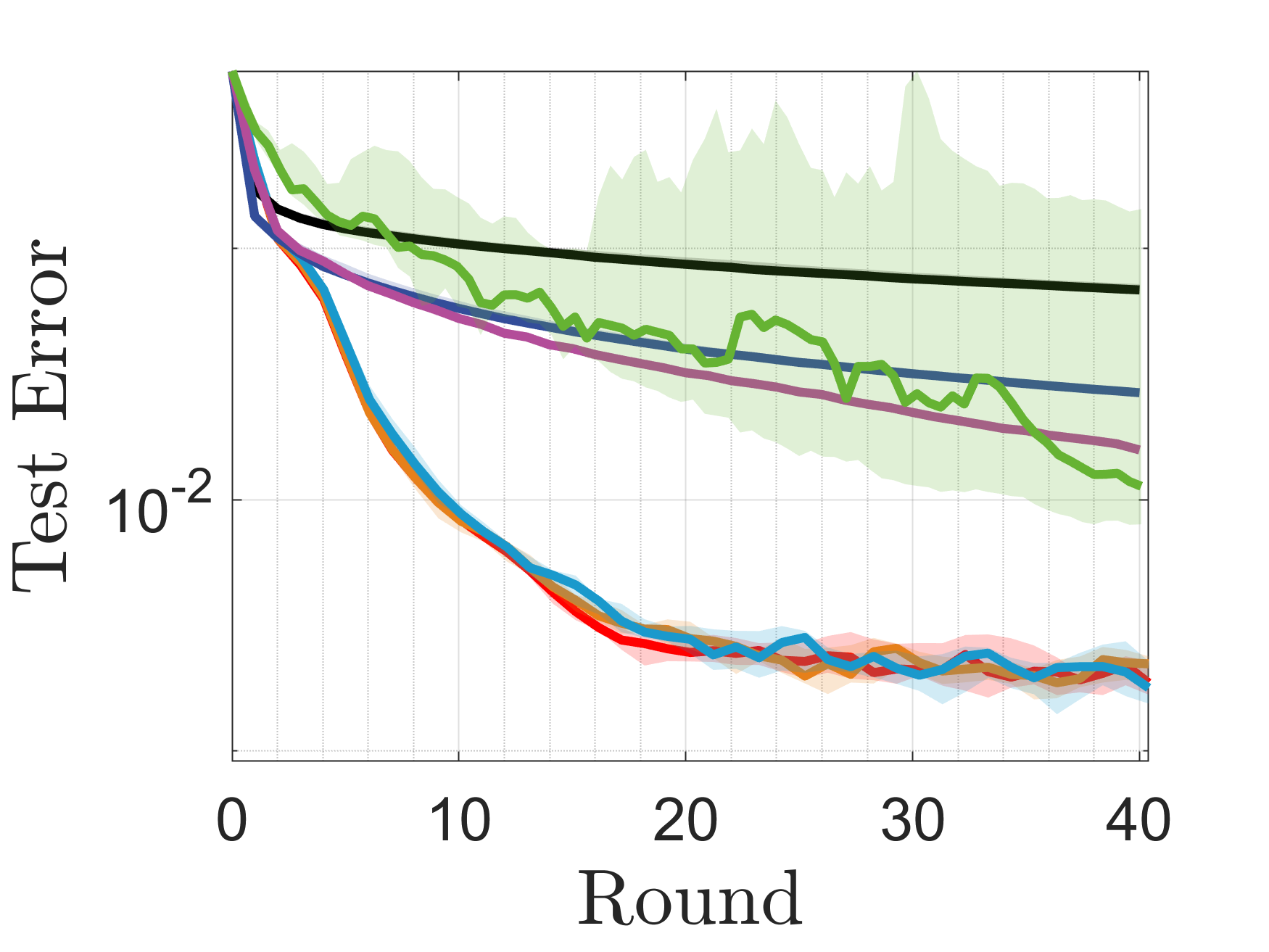}}
\subfloat[NSVM, ijcnn1]{\includegraphics[width=0.33\textwidth]{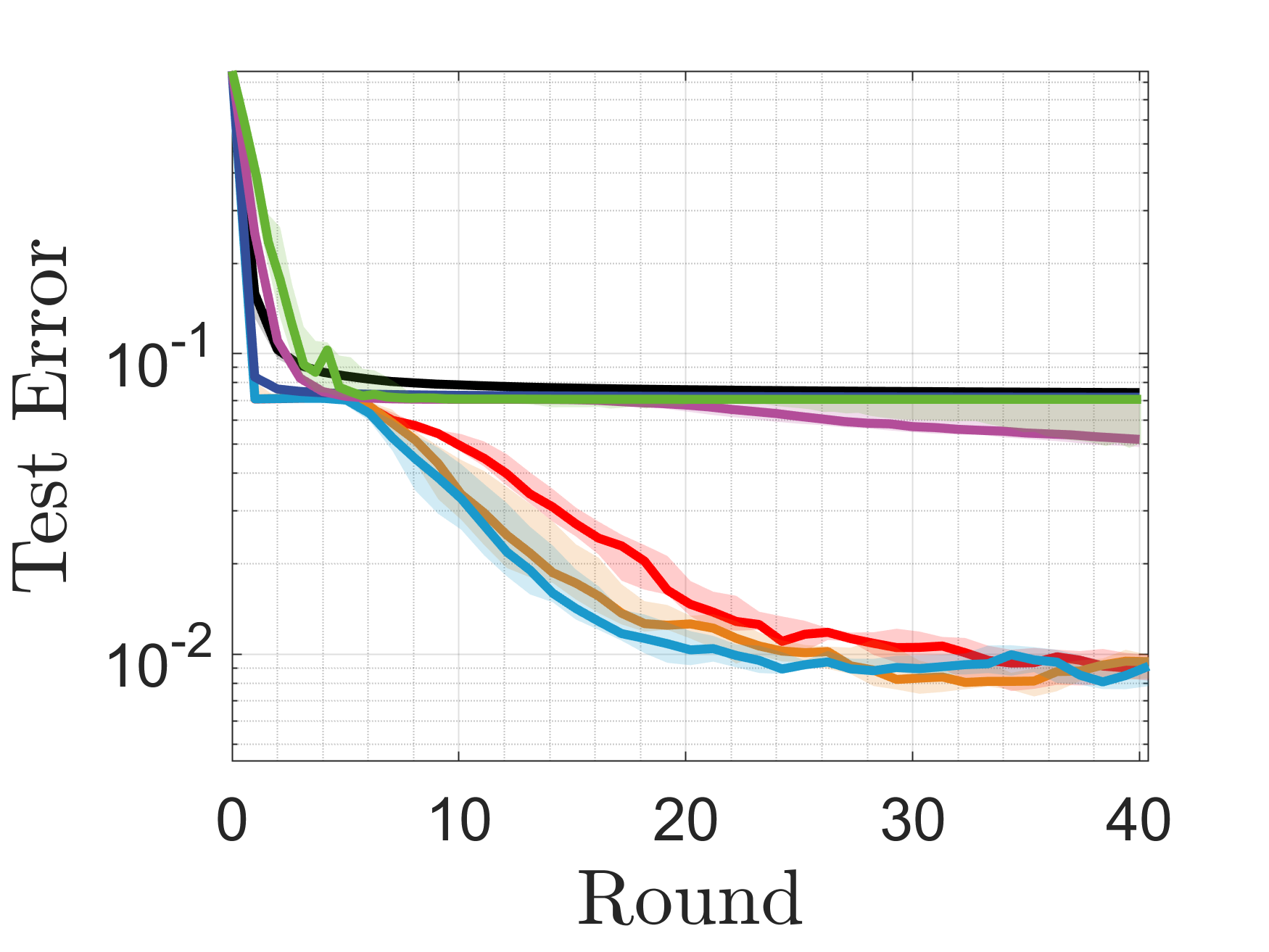}}
\subfloat[LSVM, ijcnn1]{\includegraphics[width=0.33\textwidth]{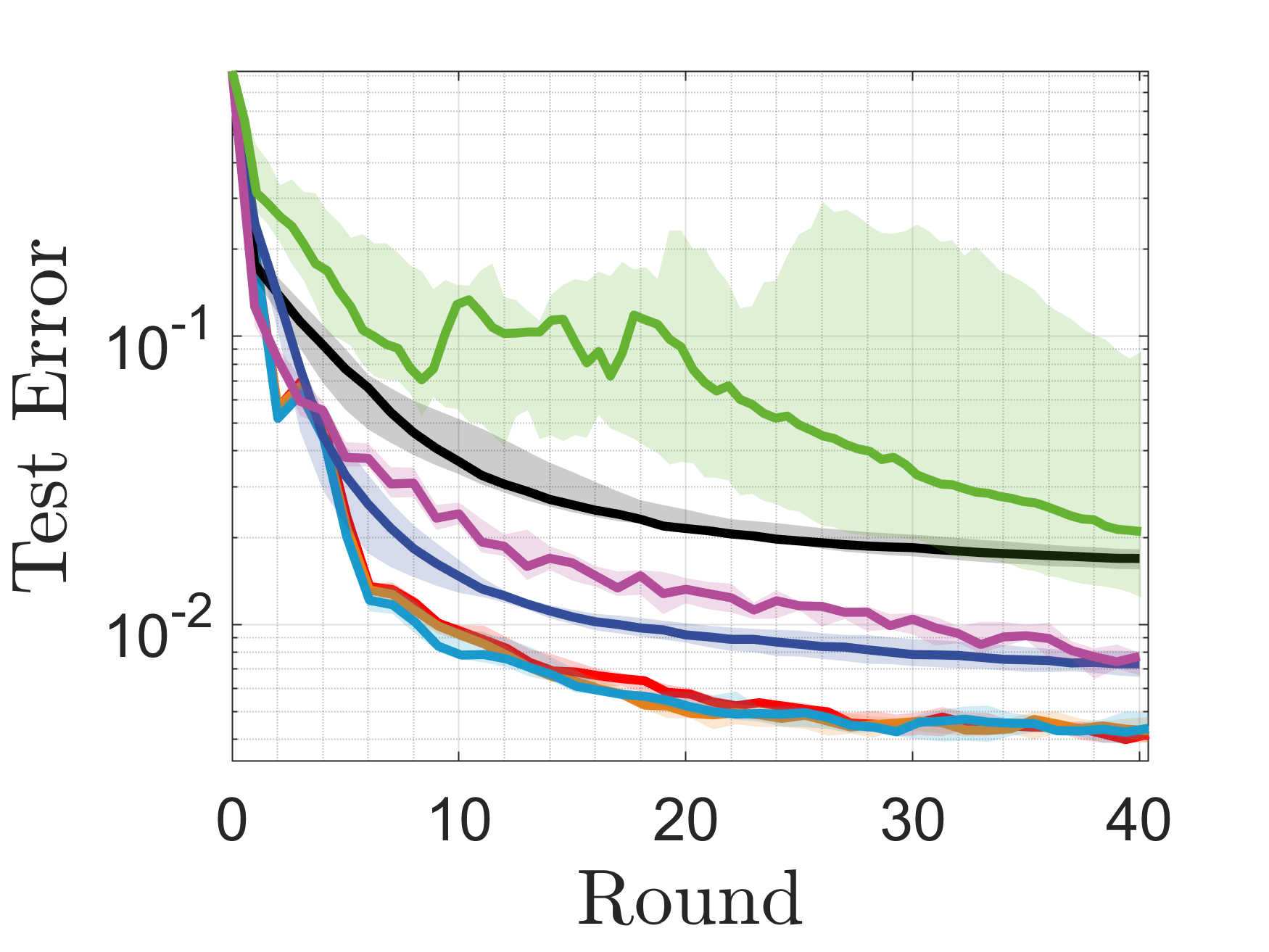}}
\hfil
\subfloat[LR, covtype]{\includegraphics[width=0.33\textwidth]{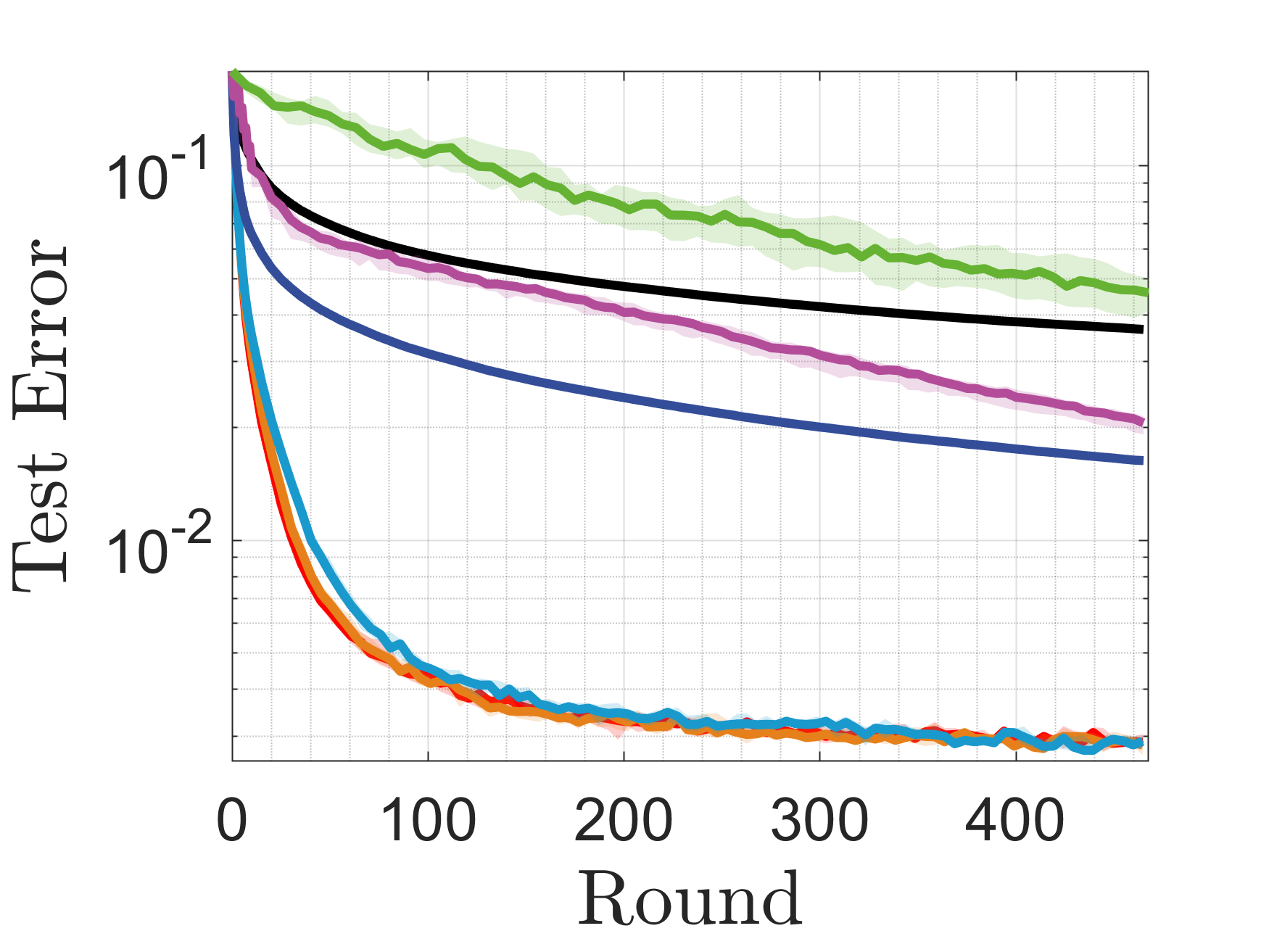}}
\subfloat[NSVM, covtype]{\includegraphics[width=0.33\textwidth]{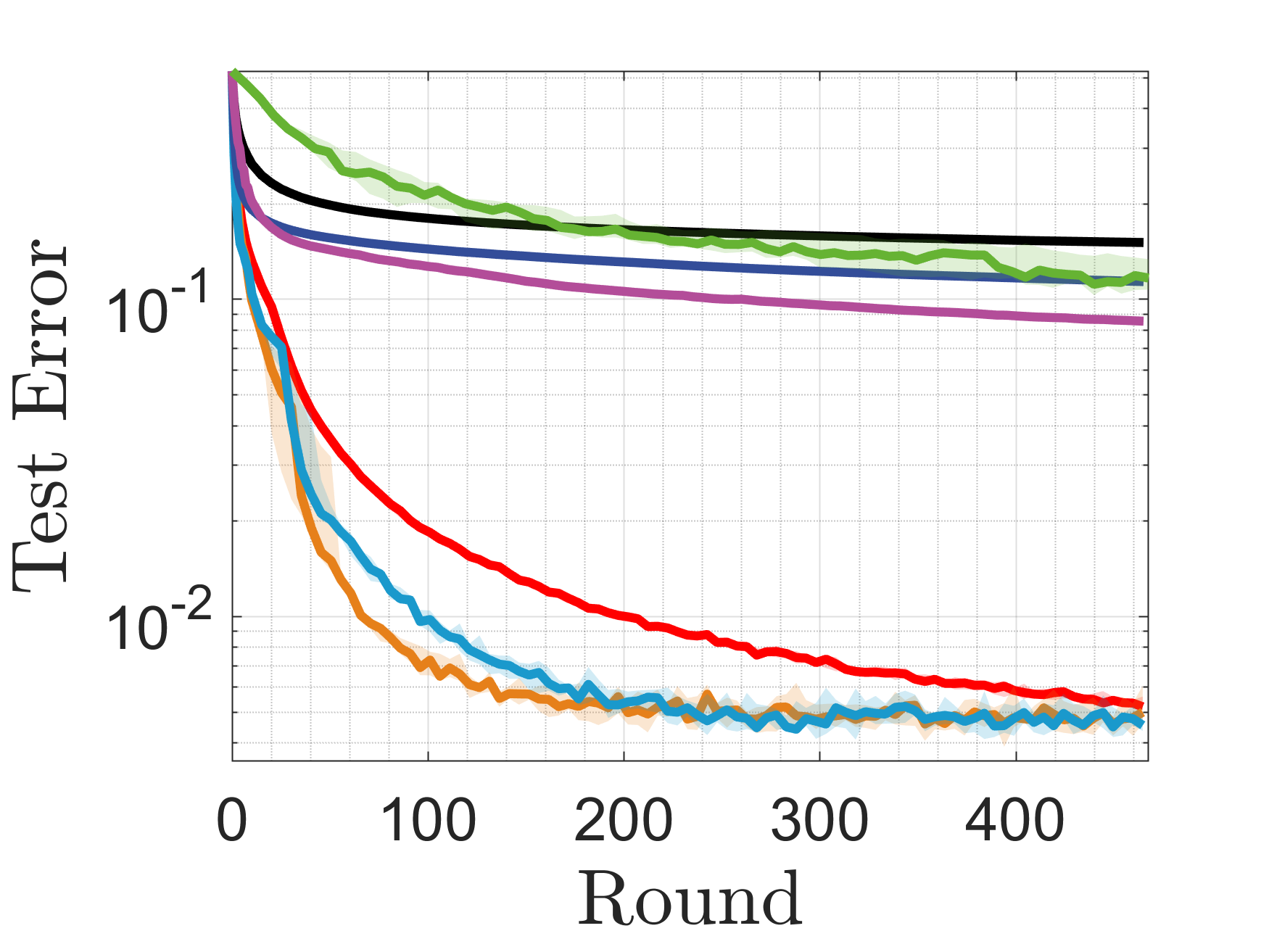}}
\subfloat[LSVM, covtype]{\includegraphics[width=0.33\textwidth]{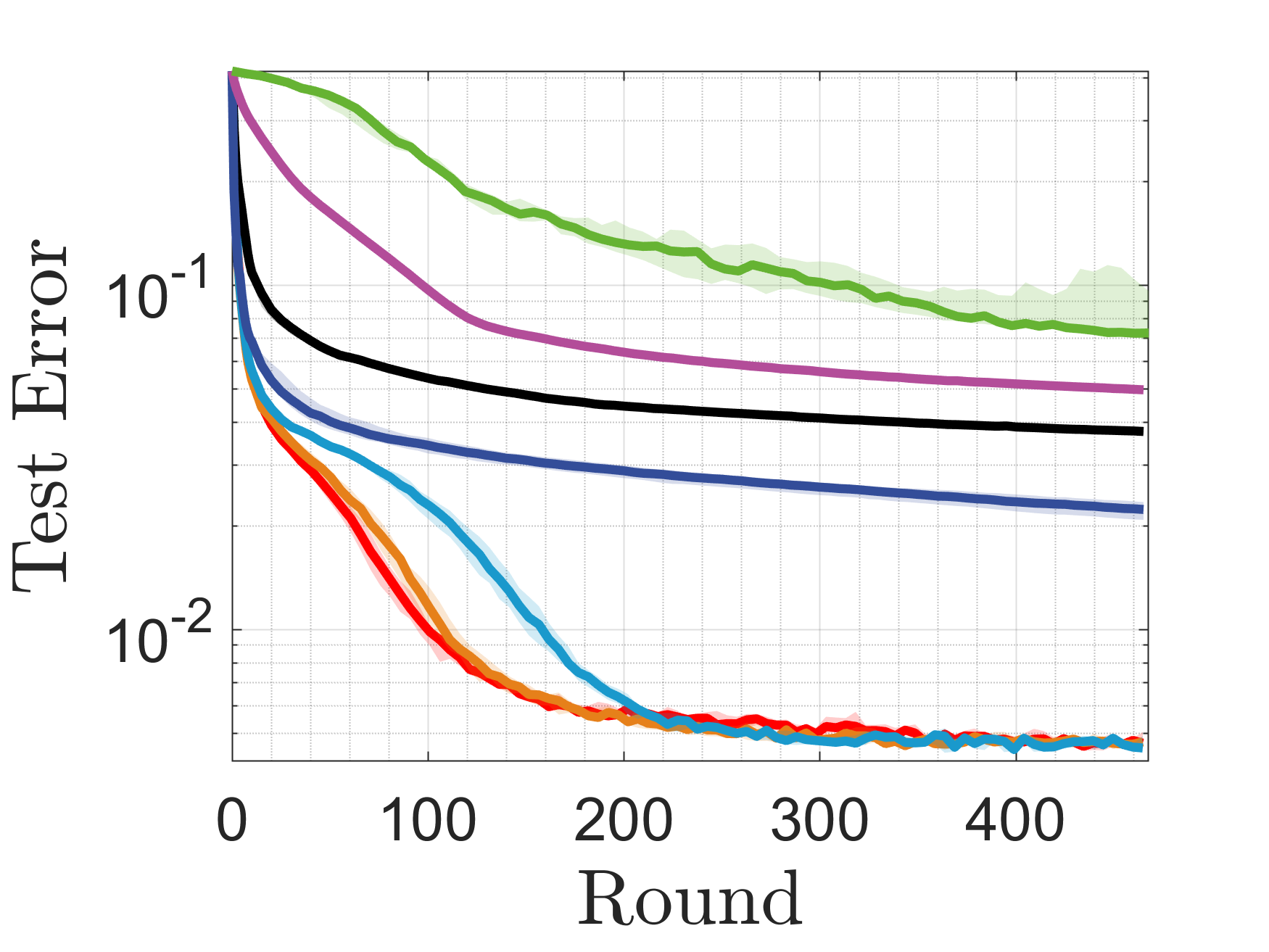}}
\caption{Generalization performance on real-sim, ijcnn1, and covtype datasets. The curve displays the test error versus the number of rounds and the corresponding shaded area extends from the 25th to 75th percentiles over the results obtained from all independent runs.}
\label{fig:testing-curve-part-two}
\end{figure*}

\end{document}